\documentclass{article}

\usepackage[preprint]{neurips_2026}

\usepackage{amsmath}
\usepackage{multirow}
\usepackage{mathrsfs}

\usepackage{subcaption}

\usepackage{booktabs}
\usepackage{adjustbox}
\usepackage{pdflscape}
\newcommand{\prof}[2]{\ensuremath{#1 \pm #2}}

\usepackage[utf8]{inputenc} % allow utf-8 input
\usepackage[T1]{fontenc}    % use 8-bit T1 fonts
\usepackage[
  colorlinks=true,
  linkcolor=cyan,
  citecolor=cyan
]{hyperref}       % hyperlinks
\usepackage{url}            % simple URL typesetting
\usepackage{booktabs}       % professional-quality tables
\usepackage{amsfonts}       % blackboard math symbols
\usepackage{nicefrac}       % compact symbols for 1/2, etc.
\usepackage{microtype}      % microtypography
\usepackage{xcolor}         % colors

\usepackage{booktabs}

\usepackage{microtype}
\usepackage{graphicx}
\usepackage{amssymb}
\usepackage{multicol}
\usepackage{wasysym}
\usepackage{makecell}
\usepackage{pifont}
\usepackage{booktabs} % for professional tables
\usepackage{float}
\usepackage{hyperref}

\usepackage[ruled,linesnumbered,noend]{algorithm2e}
\usepackage{amsmath}
\usepackage{amssymb}

\SetKwInput{KwInput}{Input}
\SetKwInput{KwInit}{Initialize}
\SetKw{KwFor}{for}
\SetKw{KwEnd}{end for}
\DontPrintSemicolon
\SetNlSty{}{}{:}
\usepackage[bb=dsserif]{mathalpha}
\usepackage{bm}

\usepackage{amsmath}
\usepackage{amssymb}
\usepackage{mathtools}
\usepackage{amsthm}
\usepackage{hyperref} % 可选，用于生成可点击的目录
\usepackage{wrapfig}
\usepackage{colortbl}

\usepackage[capitalize,noabbrev]{cleveref}

\theoremstyle{plain}
\newtheorem{theorem}{Theorem}[section]

\newtheorem{lemma}[theorem]{Lemma}

\theoremstyle{definition}
\newtheorem{definition}[theorem]{Definition}

\theoremstyle{remark}

\usepackage{mdframed}

\newmdtheoremenv[
  backgroundcolor=white!95!blue!5,  % 冷灰蓝，干净
  linecolor=blue!40!cyan,          % 柔和蓝黑
  linewidth=1pt,
  roundcorner=6pt,
  innertopmargin=8pt,
  innerbottommargin=8pt,
  innerrightmargin=10pt,
  innerleftmargin=10pt,
  frametitlebackgroundcolor=cyan!10, % (可选) 标题条
  frametitleaboveskip=2pt,
  frametitlebelowskip=2pt,
  frametitlefont=\bfseries,
  shadow=true,
  shadowsize=3pt,
  shadowcolor=cyan!15,
]{theoremwithframe}{Theorem}

\newmdtheoremenv[
  backgroundcolor=white!95!red!5,  % 冷灰蓝，干净
  linecolor=red!40!red,          % 柔和蓝黑
  linewidth=1pt,
  roundcorner=6pt,
  innertopmargin=8pt,
  innerbottommargin=8pt,
  innerrightmargin=10pt,
  innerleftmargin=10pt,
  frametitlebackgroundcolor=red!10, % (可选) 标题条
  frametitleaboveskip=2pt,
  frametitlebelowskip=2pt,
  frametitlefont=\bfseries,
  shadow=true,
  shadowsize=3pt,
  shadowcolor=red!15,
]{insightmwithframe}{Problem}

\title{Scaling World-Model Reinforcement Learning Through Diffusion Policy Optimization}

\author{
Xiaoyuan Cheng$^{1,}$\thanks{
Core contributors: Xiaoyuan Cheng: \url{ucesxc4@ucl.ac.uk} and Wenxuan Yuan: \url{YUAN0186@e.ntu.edu.sg}. $^\dagger$ Corresonding authors: Zhuo Sun: \url{zhuosunreid@outlook.com}, Che Liu: \url{cl522@ic.ac.uk}.
}
\quad
Wenxuan Yuan$^{2,*}$ \quad
Zhancun Mu$^{3}$ \quad
Yuanzhao Zhang$^{4}$ \\
\textbf{Yiming Yang}$^{1}$ \quad
\textbf{Hai Wang}$^{1}$ \quad
\textbf{Zhuo Sun}$^{5,\dagger}$ \quad
\textbf{Che Liu}$^{6,\dagger}$ \\[0.3cm]
$^{1}$Dynamic Systems Lab, University College London \\
$^{2}$College of Computing and Data Science, Nanyang Technological University \\
$^{3}$School of Intelligence Science and Technology, Peking University \quad 
$^{4}$Santa Fe Institute \\
$^{5}$School of Statistics and Data Science, Shanghai University of Finance and Economics \\
$^{6}$Department of Computing, Imperial College London \\
\textit{Code:} \url{https://github.com/Edmond1Cheng/MBDPO}\\
\textit{Hugging Face:} \url{https://huggingface.co/BruceYuan/MBDPO}
}

\begin{document}

\maketitle

\begin{figure}[H]
    \centering
    \includegraphics[width=0.9\linewidth]{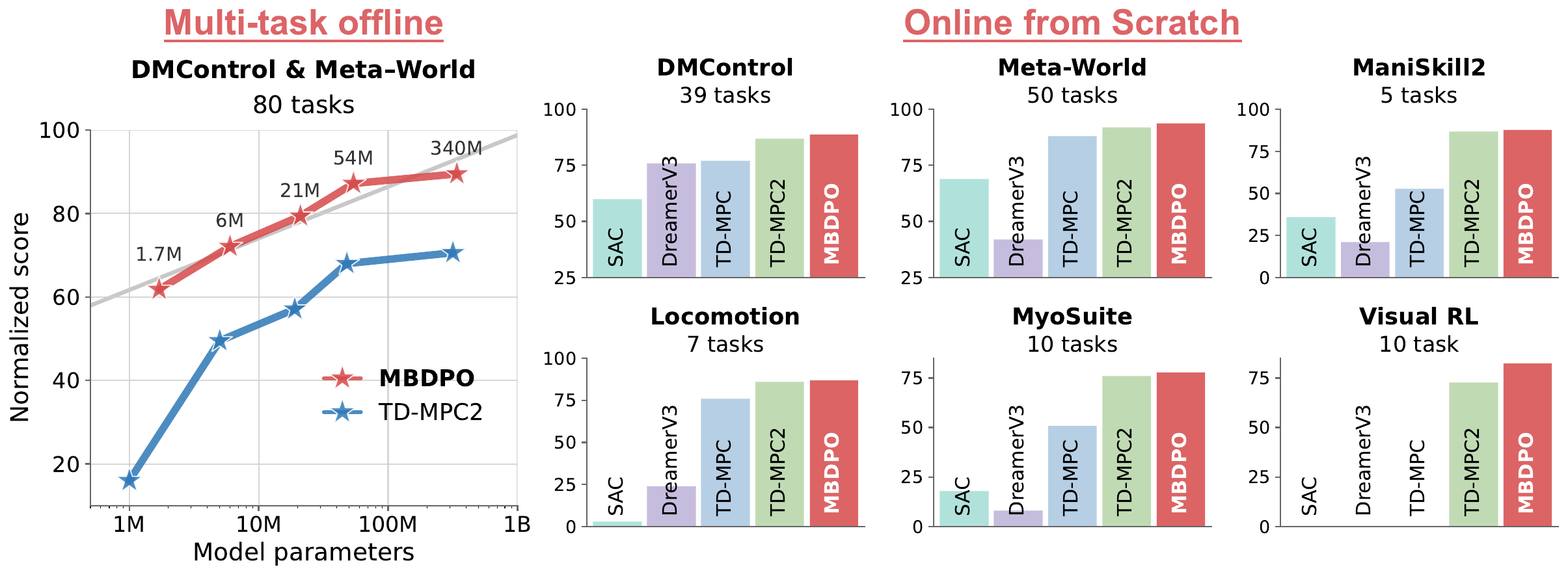}
    \caption{\textbf{Overview of offline and online performance.} (Left) Our method (MBDPO) significantly outperforms TD-MPC2 \citep{hansen2023td} in multi-task offline pretraining, exhibiting a clear monotonic scaling behavior as model parameters increase from $1.7$M to $340$M. (Right) In the online-from-scratch setting, MBDPO consistently achieves superior or competitive results across $4$ benchmarks with $121$ tasks.}
    \label{fig:overview}
\end{figure}

\begin{abstract}
Model-based reinforcement learning (RL) can be effectively supported at scale through the use of world models. However, in practice, scaling such approaches remains fundamentally limited. A commonly recognized challenge is model bias and error compounding, which degrade long-horizon predictions. Beyond these issues, we identify a more critical yet underexplored bottleneck: a structural misalignment between search and value learning in existing world model approaches. In particular, policy improvement often relies on value functions induced by a separate, non-search policy, resulting in training inconsistency and ultimately suboptimal learning. To address this limitation, we propose Model-Based Diffusion Policy Optimization (MBDPO) in world models, a framework that unifies search and policy optimization through diffusion policy representations, thereby unlocking the potential of world models for scalable policy learning. Instead of constructing an explicit planner over a learned world model, we reformulate policy optimization as a diffusion process over searched trajectories in latent world models.  In this view, we extract an implicit energy function from the collected dataset that anchors the policy, enabling MBDPO to refine the score field for policy optimization while mitigating misalignment. We evaluate MBDPO across a wide range of settings, including multi-task offline pretraining, online learning, and offline-to-online fine-tuning. In the offline regime, we further investigate its scaling behavior by pretraining on large-scale datasets, observing consistent and monotonic performance gains with increasing model capacity.
\end{abstract}

\section{Introduction}

Model-based reinforcement learning (RL) has long been regarded as a promising paradigm for sample-efficient planning in general cases. By learning an explicit dynamical system, agents can dream and plan without multiple interactions with real environments \citep{sutton1996model, sutton1998reinforcement}. Recent advances in world models have further expanded this paradigm \citep{garrido2024learning, zhou2024dino, hafner2025mastering, psenka2026parallel}, enabling large models to capture rich environment dynamics across diverse tasks. Instead of serving solely as local frame predictors \citep{sv2023gradient, thrun1990planning, guan2023leveraging}, modern world models aim to learn generalizable representations of environment dynamics that can be reused across downstream control problems. At the same time, the emergence of large-scale datasets has created new opportunities for pretraining world models at scale \citep{wu2023daydreamer, hafner2025training}. Analogous to foundation models in vision and language, pretrained world models offer the potential to unlock scalable model-based RL among multi-tasks \citep{hansen2023td}.  

Despite recent advances, model-based RL methods built upon world models still struggle to meaningfully outperform their model-free counterparts \citep{van2019use, fujimoto2025towards, chang2026surprising}. Two explanations have been proposed. The most common explanation attributes this shortfall to limited model accuracy and compounding prediction errors \citep{m2023model, lambert2022investigating}. When a learned model is used to simulate long-horizon trajectories, even small one-step prediction errors accumulate over time, resulting in increasingly unreliable state estimates. Motivated by this observation, prior work has focused on mitigating policy degradation induced by imperfect model dynamics. Representative approaches include improving model accuracy \citep{nagabandi2018neural, hafner2019learning}, incorporating uncertainty-aware modeling \citep{deisenroth2011pilco, curi2020efficient, fu2022model}, employing truncated or branched rollouts \citep{janner2019trust, park2025horizon}, and enhancing long-horizon prediction \citep{ma2024transformer, farebrother2025temporal}. Underlying these efforts is an implicit assumption: that sufficiently accurate world models will naturally translate into better planning and, ultimately, better control performance.

However, recent work challenges this point \citep{chang2026surprising}, arguing that model accuracy alone does not dictate planning effectiveness. Even highly accurate models can fail to yield strong policies, causing RL within world models to struggle in both performance and scalability.
The core obstacle lies in the misalignment between policy search and value learning. Due to the substantial computational cost of search, the value function is typically trained using rollouts generated by a learned non-search policy network \citep{hansen2022temporal, hansen2023td, hansen2025learning, wang2025bootstrapped, zhan2025bootstrap, shimizu2024bisimulation}. In contrast, policy improvement relies on trajectories produced by a search procedure (i.e., model predictive path integral \citep{williams2015model}) over a learned world model. This mismatch introduces a distributional discrepancy: the value function is trained on state-action pairs generated by a non-search policy, whereas policy execution and improvement are governed by trajectories produced via search over the world model. As a result, the value function is often queried on out-of-distribution state-action pairs where its estimates are unreliable and typically overoptimistic; the search procedure then exploits these errors, leading to systematic overestimation bias and ultimately degrading both performance and scalability \citep{fujimoto2019off}.

% \begin{wrapfigure}{r}{0.7\linewidth}
% \centering
% % \vspace{-10pt}
% \includegraphics[width=\linewidth]{abstract.pdf}
% \vspace{-10pt}
% \caption{}
% \label{fig:mbdpo}
% \vspace{-10pt}
% \end{wrapfigure} 

\begin{figure}[ht]
    \centering
    \includegraphics[width=0.9\linewidth]{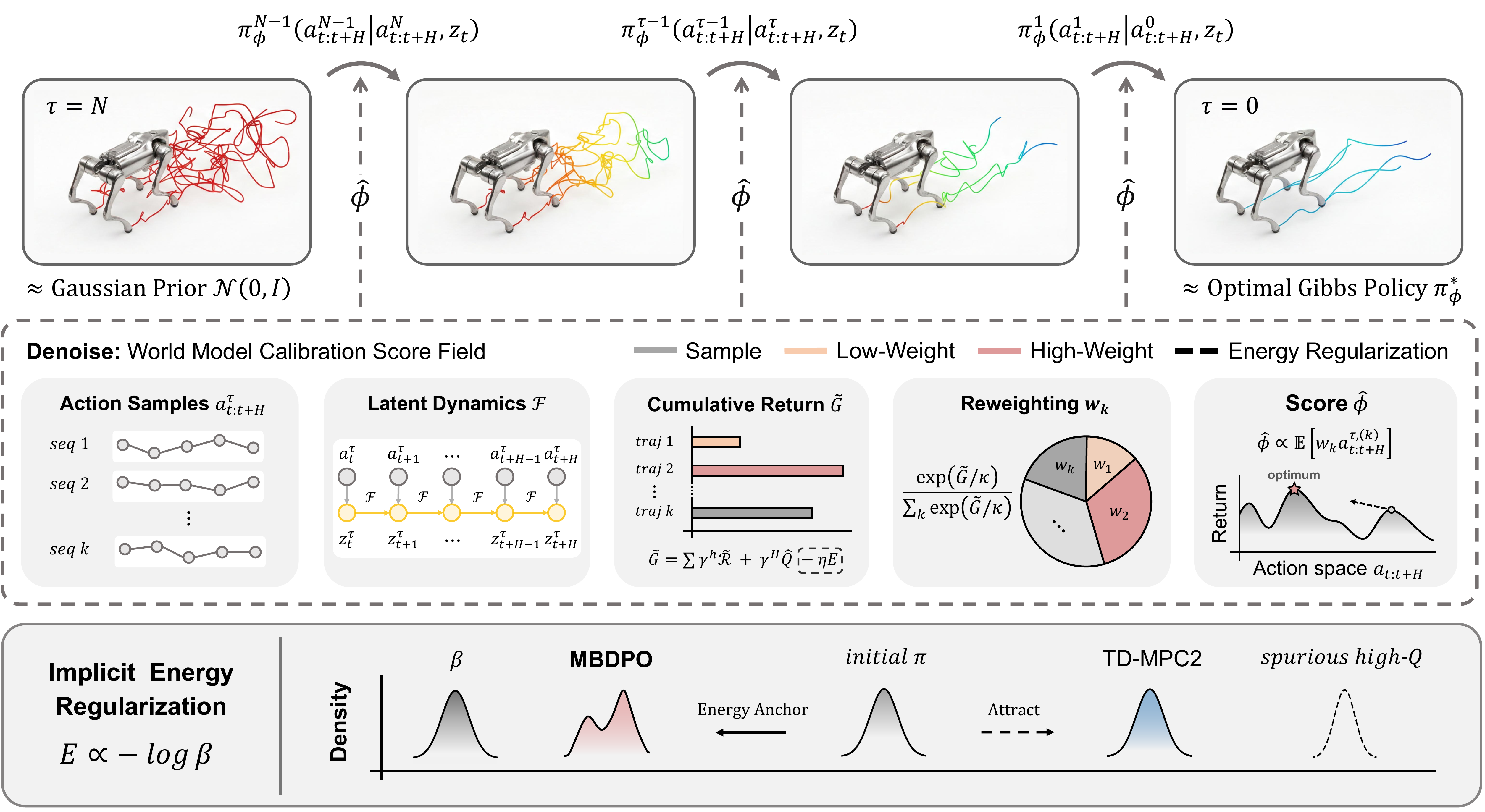}
%     \caption{\textbf{Core Framework of MBDPO.} 
% The target policy distribution is progressively shaped by a sequence of stepwise transition kernels (from $\pi_\phi^N$ to $\pi_\phi^0$), which transform a Gaussian prior $\mathcal{N}(0, I)$ into the optimal Gibbs policy $\pi_\phi^*$ via multi-step refinement. 
% Crucially, each transition kernel $\pi^{\tau-1}_{\phi}(a_{t:t+H}^{\tau-1}| a_{t:t+H}^{\tau}, z_t)$ is governed by the score function $\hat{\phi}$, which is estimated entirely within the learned world model: candidate sequences are rolled out through the latent dynamics $\mathcal{F}$, evaluated by their energy-regularized cumulative return $\tilde{G}$, and reweighted to recover the score that points toward higher-return actions. 
% In this way, the world model redefines and corrects the score field, mathematically transforming the standard generative denoising process into model-based policy optimization. 
% Furthermore, the implicit energy $E \propto -\log\beta$ anchors the policy to the behavior distribution within a KL trust region (bottom), thereby unifying search and policy optimization.}
\caption{\textbf{Core framework of MBDPO.} 
The target policy distribution is progressively shaped by a sequence of stepwise transition kernels, from $\pi_\phi^N$ to $\pi_\phi^0$, which transforms a Gaussian prior $\mathcal{N}(0, I)$ into the optimal Gibbs policy $\pi_\phi^*$ through multi-step refinement.  Crucially, each transition kernel $\pi^{\tau-1}_{\phi}(a_{t:t+H}^{\tau-1}|a_{t:t+H}^{\tau}, z_t)$ is governed by the score function $\hat{\phi}$, estimated entirely within the learned world model: Monte Carlo action sequence samples are rolled out through the latent dynamics $\mathcal{F}$, evaluated by their energy-regularized cumulative return $\tilde{G}$, and reweighted to recover a score that points toward higher-return actions.  In this way, the world model redefines and corrects the score field, transforming the standard generative denoising process into model-based policy optimization.  Furthermore, the implicit energy $E \propto -\log\beta$ anchors the policy to the non-search policy distribution within a KL trust region, thereby unifying search and policy optimization.}
    \label{fig:mbdpo}
\end{figure}

Motivated by the bottlenecks of prior approaches, we investigate how to overcome these two limitations and unlock the scalability potential of model-based RL within world models. In this work, we introduce a novel perspective grounded in diffusion models and rethink how to make model-based RL in world models scalable. From this perspective, we demonstrate that world models enable direct policy optimization, where score matching is achieved through imagined trajectories. Our contributions are threefold:

(1) We revisit the foundation of policy optimization and show, from the perspectives of both value iteration and policy improvement, that the misalignment between search and value learning can lead to suboptimality in existing world model approaches.

(2) We propose model-based diffusion policy optimization (MBDPO), a unified framework that formulates search and policy optimization as a diffusion process over imagined trajectories of world models, thereby unlocking the potential of world models for scalable policy learning (see Figure~\ref{fig:mbdpo}). In this formulation, the buffer dataset induces an implicit energy function that anchors policy updates to the behavior distribution. This mechanism enables score correction beyond the suboptimality in previous methods. By doing so, MBDPO unifies search and policy optimization, eliminating their structural misalignment.

(3) We develop a scalable diffusion RL algorithm jointly trained with the world model and demonstrate consistent, monotonic performance improvements as model capacity increases across diverse settings, including multi-task offline pretraining, online learning, and offline-to-online fine-tuning. Beyond these quantitative results, we visualize our controlled latent trajectories to evaluate the underlying representation, finding that the learned policy consistently exhibits structured behaviors that align with physical intuition.

% \begin{itemize}
%     \item We revisit the foundation of policy optimization and show, from the perspectives of both value iteration and policy improvement, that the misalignment between search and value learning can lead to suboptimality in existing world model approaches.
%     \item We propose model-based diffusion policy optimization (MBDPO), a unified framework that formulates search and policy optimization as a diffusion process over imagined trajectories of world models, thereby unlocking the potential of world models for scalable policy learning (see Figure~\ref{fig:mbdpo}). In this formulation, the buffer dataset induces an implicit energy function that anchors policy updates to the behavior distribution. This mechanism enables score correction, beyond suboptimality in previous methods. By doing so, MBDPO unifies search and policy optimization, eliminating their structural misalignment.
%     \item We develop a scalable diffusion RL algorithm jointly trained with the world model and demonstrate consistent, monotonic performance improvements as model capacity increases across diverse settings, including multi-task offline pretraining, online learning, and offline-to-online fine-tuning. Beyond these quantitative results, we visualize our controlled latent trajectories to evaluate the underlying representation, finding that the learned policy consistently exhibits structured behaviors that align with physical intuition.
% \end{itemize}
% (1)  (3) 
\textbf{Related Work.} We review prior work in two areas: world models for RL, and diffusion policy. 

\textbf{World Models.} The concept of world models originates from Jay W. Forrester, who introduced it for constructing mental representations of complex systems \citep{forrester1971counterintuitive}. It was later adapted to machine learning, enabling agents to learn and act within learned internal environments \citep{ha2018world}. Later, the Dreamer series \citep{hafner2019learning, hafner2020mastering, hafner2025mastering, hafner2025training} transformed latent world modeling from a proof-of-concept into a scalable and robust RL paradigm. However, due to model bias, training the policy purely within the learned world model can lead to distributional shift \citep{janner2019trust}. Another line of research, TD-MPC \citep{hansen2022temporal, hansen2023td, hansen2025learning}, 
combines model-based and model-free RL. However, the existence of misalignment between search and value hinders scalability and performance. Recent approaches, including Vision-Language-Action models with world models \citep{zhen20243d,cen2025worldvla, team2026gigabrain, intelligence2026pi} and World Action Models (WAMs) \citep{han2025percept, ye2026worldactionmodelszeroshot, li2026causal}, have been proposed to generate action chunks and state representations, driven by the rapid development of large pre-trained vision models. Compared to RL, the WAM family does not perform explicit policy optimization, which limits its ability to achieve long-horizon control and robust generalization.

\textbf{Diffusion Policy.} Diffusion policies can be broadly categorized into model-free and model-based approaches.  Model-free diffusion policies do not rely on learned dynamics; instead, they directly learn the policy distribution from a fixed dataset \citep{janner2022planning, chi2025diffusion} or a learned $Q$-function \citep{psenka2023learning, dong2025maximum, cheng2026does, li2026q}.  In contrast, model-based diffusion techniques leverage environment dynamics to estimate the score function without relying on offline data \citep{pan2024model, xue2025full, cheng2025safe}. Crucially, computing this score function requires access to the transition dynamics, making world models an ideal substrate; they provide the simulated environment necessary. However, in many practical scenarios, ground-truth dynamics are unavailable, limiting their direct applicability. This bottleneck, in turn, highlights the critical necessity of learned world models. By approximating a diverse spectrum of dynamics, world models hold the key to unlocking the full potential of model-based diffusion for scalable policy learning. Despite this immense promise, effective model-based diffusion policies within learned world models remain underexplored.

% \textbf{Diffusion Policy.}
% Instead of generating a one-step action, we locally optimize an action sequence 
% $a_{t:t+H} = \{ a_t, \dots, a_{t+H} \}$ 
% and execute only the first action in a receding-horizon manner.
% We now describe how reinforcement learning leverages diffusion models \citep{ho2020denoising}.

% Let the cumulative scale factor be
% $\bar{\alpha}^{\tau} = \prod_{k=1}^{\tau} \alpha^k$.
% The forward and reverse diffusion processes are defined as
% \begin{align}
% \tag{forward}
% \pi^{\tau}(a_{t:t+H}^{\tau} \mid a_{t:t+H}^{0}, s_t)
% \sim
% \mathcal{N}\big(
% \sqrt{\bar{\alpha}^{\tau}}\, a_{t:t+H}^{0},
% (1-\bar{\alpha}^{\tau}) I
% \big), \\
% \tag{backward}
% a_{t:t+H}^{\tau-1}
% =
% \frac{1}{\sqrt{\alpha^{\tau}}}
% \left(
% a_{t:t+H}^{\tau}
% +
% (1 - \alpha^{\tau})
% \colorbox{red!15}{$\nabla_{a_{t:t+H}^{\tau}}
% \log \pi(a_{t:t+H}^{\tau} \mid s_t)$}
% \right).
% \end{align}

% The key difference between model-free and model-based diffusion lies in how the score function is estimated. Model-free diffusion approximates the score using the $Q$-function, whereas model-based diffusion estimates it through imagined rollouts and trajectory dreaming.
\vspace{-0.5em}
\section{Preliminaries}
\textbf{Model-Based RL.} We consider a Markov Decision Process (MDP) defined by the tuple $\mathcal{M} = (\mathcal{S}, \mathcal{A}, \mathcal{T}, r, \rho, \gamma)$ \citep{puterman1990markov}, where $\mathcal{S}$ and $\mathcal{A}$ denote the state and action spaces, respectively. $\rho$ is the initial state distribution, $\mathcal{T}: \mathcal{S} \times \mathcal{A} \to \mathcal{S}$ represents the transition dynamics, $r: \mathcal{S} \times \mathcal{A} \to \mathbb{R}$ is the reward function in the state space, and $\gamma \in (0,1)$ is the discount factor. For a policy $\pi(a|s)$, the action-value function (or $Q$-function) $Q^\pi(s, a)$ is defined as the expected discounted return starting from state $s$ and action $a$: $Q^\pi(s, a) \coloneqq \mathbb{E}_{\pi} \left[ \sum_{t=0}^\infty \gamma^t r(s_t, a_t) \;\middle|\; s_0 = s, a_0 = a \right]$. Our objective is to find a policy $\pi$ that maximizes the expected discounted return: $\mathbb{E}_{s \sim d^{\pi}, a \sim \pi} [Q^\pi(s, a)].$ In this paper, we jointly optimize the policy $\pi$ and a world model that approximates the latent MDP dynamics.

\textbf{Latent World Model.} Drawing on recent advances \citep{hansen2023td, hansen2025learning,  hafner2025mastering}, we employ a latent world model to capture compact representations essential for complex decision-making tasks. Our world model comprises the following components:
\begin{equation}
\label{eq:components_mapping}
\begin{alignedat}{3}
    &\text{Encoder}         && \mathcal{E} \colon \mathcal{S} \times \textcolor{gray}{\mathcal{E}_{env}} \to \mathcal{Z}, \quad && z_t = \mathcal{E}(s_t, \textcolor{gray}{e}); \\
    &\text{Latent Dynamics} \quad && \mathcal{F} \colon \mathcal{Z} \times \mathcal{A} \times \textcolor{gray}{\mathcal{E}_{env}} \to \mathcal{Z}, \quad && z_{t+1} = \mathcal{F}(z_t, a_t, \textcolor{gray}{e}); \\
    &\text{Latent Reward}   && \mathcal{R} \colon \mathcal{Z} \times \mathcal{A} \times \textcolor{gray}{\mathcal{E}_{env}} \to \mathbb{R}, \quad && \hat{r}_t = \mathcal{R}(z_t, a_t, \textcolor{gray}{e});
\end{alignedat}
\end{equation}
% \vspace{-8pt}
where $z_t \in \mathcal{Z}$ denotes the latent state and \textcolor{gray}{$e \in \mathcal{E}_{env}$} is a learnable task embedding designed for multi-task generalization. For brevity, we \textit{omit} \textcolor{gray}{$e$} in subsequent discussions unless its inclusion is explicitly required. Rather than operating in the raw observation space, our RL framework is built entirely upon the learned latent world model.

\textbf{Diffusion Policy.} Instead of generating a single-step action, we optimize an action sequence $a_{t:t+H} \coloneqq \{ a_t, \dots, a_{t+H} \}$ and execute it in a receding-horizon fashion, similar to MPC. We adopt the Denoising Diffusion Probabilistic Model (DDPM) framework \citep{ho2020denoising} to represent our generative policy. Let $\tau \in \{1, \dots, N\}$ denote the diffusion timestep. Since each step of the reverse-time transition is determined by the score function $\phi(z_t, a_{t:t+H}^\tau, \tau) \coloneqq \nabla_{a_{t:t+H}^{\tau}}
\log \pi^\tau_\phi(a_{t:t+H}^{\tau} | z_t)$, the evolution of the policy is directly governed by $\phi$. Accordingly, the reverse diffusion process is characterized by the reverse-time transition kernel $\pi_{\phi}^{\tau-1} (a_{t:t+H}^{\tau-1}|a_{t:t+H}^{\tau}, z_t)$, where the marginal distribution of the clean action sequence is 
\begin{equation} \label{eq:policy_refinement}
    \pi_{\phi}^{0} (a_{t:t+H}^{0}|z_t)= \int_{a_{t:t+H}^{1:N}} \pi^{N}_{\phi}(a_{t:t+H}^{N}|z_t) \prod_{\tau = N}^1 \pi_{\phi}^{\tau-1} (a_{t:t+H}^{\tau-1}|a_{t:t+H}^{\tau}, z_t) da_{t:t+H}^{1:N},
\end{equation}
which describes a multi-step refinement process that transforms from a Gaussian prior $\pi^{N}(a_{t:t+H}^{N}|z_t) \sim \mathcal{N}(0, I)$. Specifically, the stepwise transition kernel is guided by the score function:
\begin{equation} \label{eq:transition_kernel} \small
    \pi_{\phi}^{\tau-1} (a_{t:t+H}^{\tau-1} |a_{t:t+H}^{\tau}, z_t) = \mathcal{N} \left( a_{t:t+H}^{\tau-1} ; \frac{1}{\sqrt{\alpha^\tau}} \left( a_{t:t+H}^\tau + (1-\alpha^\tau) \phi(z_t, a_{t:t+H}^\tau ,\tau) \right), \sigma^2(\tau) I \right),
\end{equation}
where $\alpha^\tau$ is the scalar factor of noise injection at the $\tau$ step, and $\sigma^2(\tau)$ denotes the step-dependent variance. In contrast to conventional diffusion policies in the model-free setting, our approach leverages the world model to estimate the score function $\phi$ via imagined trajectories, thereby unifying the search and policy optimization. The next section establishes the link between the diffusion refinement in \eqref{eq:policy_refinement} and policy optimization within the world model framework.
% Let $\tau \in \{1, \dots, N\}$ denote the diffusion timestep and $\bar{\alpha}^{\tau} \in (0, 1]$ be the cumulative noise schedule. The forward and reverse diffusion processes for the action sequence are defined as:
% \begin{align} 
% \text{(Forward)} \quad & p(a_{t:t+H}^{\tau} \mid a_{t:t+H}^{0}, s_t) = \mathcal{N}\left(\sqrt{\bar{\alpha}^{\tau}} a_{t:t+H}^{0}, (1-\bar{\alpha}^{\tau}) I\right), \\
% \text{(Reverse)} \quad & a_{t:t+H}^{\tau-1} = \frac{1}{\sqrt{\alpha^{\tau}}} \left( a_{t:t+H}^{\tau} + (1 - \alpha^{\tau}) \colorbox{red!15}{$\nabla_{a_{t:t+H}^{\tau}}
% \log \pi(a_{t:t+H}^{\tau} \mid s_t)$} \right) , \label{eq:reverse_sampling}
% \end{align}
% The core distinction between model-free and model-based diffusion policies lies in the estimation of the score function \colorbox{red!15}{$\nabla_{a_{t:t+H}^{\tau}}
% \log\pi$}. Model-free approaches approximate the score via gradients of a learned $Q$-function, which can suffer from sharp gradients. In contrast, by exploiting a learned world model, we can directly perform policy optimization through imagined trajectories without any gradients. 

Given the unified latent representation within world models, we focus on the following central question:
\begin{center}
\textit{What are the fundamental bottlenecks in scaling model-based RL within latent world models, and how can we architect scalable algorithms to overcome them?}
\end{center}
\vspace{-0.5em}
In the subsequent section, we analyze the scalability limitations inherent in current model-based frameworks and propose a principled approach to resolve these challenges.
 \vspace{-0.5em}
\section{Method}
In this section, we analyze why existing model-based RL approaches in world models do not correspond to a monotonic policy improvement, and thus fail to guarantee convergence to the optimal policy. To address this issue, we propose a corrected and scalable learning paradigm that reconsiders the interplay between the world model and policy optimization. 

\subsection{Misalignment Between Search and Value Learning} 
We briefly revisit RL strategies built upon world models, such as those proposed in \citep{hansen2022temporal, hansen2025learning, zhan2025bootstrap, lin2025tdmpc2improvingtemporaldifference}. In these representative approaches, a sampling-based planner operates within the learned world model to generate high-value actions for environment interaction. However, this design inherently introduces a misalignment between policy improvement and value learning \citep{chang2026surprising}. Specifically, while the value function is trained on state-action pairs generated by a non-search policy $\beta$\footnote{
Here, $\beta$ denotes the data-collection policy associated with the training samples. 
In offline settings, $\beta$ is the behavior policy that generated the fixed offline dataset. 
In online settings, $\beta$ refers to the implicit replay-buffer policy, i.e., the mixture of historical policies that collected the transitions stored in the replay buffer.
}, the executed policy $\pi$\footnote{Here, we use $\pi$ to denote the search policy used in other methods, contrasting it with our diffusion policy $\pi_{\phi}$.} relies on the sampling-based planner over the world model, resulting in a significant distributional discrepancy (see more details in Appendix~\ref{append:previous_method_stragety}). This misalignment hinders performance and limits scalability. In particular, we elucidate why such a procedure may fail to guarantee optimality, suggesting that it does not constitute true policy optimization in the conventional sense.

% In this sense, the world model is \textit{explicitly} leveraged for policy improvement, primarily to enhance sample efficiency. However, the discrepancy between the value learning (updated using search actions from planner) and the parameterized policy $\pi$ introduces a fundamental misalignment. This mismatch can hinder convergence and limit scalability. In particular, we revisit why such a learning procedure may fail to converge to a Banach fixed point.

% While search aims for policy improvement, it often leads to a vicious cycle of error propagation: (1). Since $\hat Q$ is trained on $\pi$, it lacks accuracy in regions unexplored by $\pi$. Search policy $\mu$ search from learned world models naturally gravitates toward actions with maximum $\hat Q$. These peaks are often artifacts of approximation error rather than true value, leading to: $\hat Q(z,a) \ge Q^\pi(z,a)$ for some $(z,a)$. (2). The Bellman update $T^\mu \hat Q$ propagates this overestimation bias back into the value function. Instead of correcting the error, the update reinforces it by treating the fake peaks as target values. The iterative updating can lead a suboptimal results according to the following results. 

\begin{insightmwithframe}[Value Iteration Gap] \label{prob_1}
Let $\beta$ denote the learned, non-search policy, and let $\pi$ be the policy induced by search. Suppose $\hat Q$ is the learned value function trained on data collected under $\beta$. Then, the discrepancy between the Bellman operators under $\pi$ and $\beta$ (denoted as $T^\pi$ and $T^\beta$, respectively; see Definition~\ref{def:bellman_operator}) introduces the following error when updating $\hat{Q}$:
\begin{equation} \label{eq:bellman_gap}
\|T^\pi \hat Q - T^\beta \hat Q\|_{\infty} \leq \gamma \| \hat{Q}\|_{\infty} \sqrt{2D_{\text{KL}}^{\max}(\pi\|\beta)},
\end{equation}
where $D_{\text{KL}}^{\max}(\pi\|\beta) \coloneqq \sup_{z\in \mathcal{Z}} D_{\text{KL}}(\pi(\cdot|z)\|\beta(\cdot|z))$, see the detailed proof in Appendix~\ref{append:proof_prob_1}. 
\end{insightmwithframe}

While search aims to improve the policy, it can induce a vicious cycle of error amplification. Since $\hat{Q}$ is trained under $\beta$, it is unreliable in out-of-distribution regions. The search policy $\pi$ tends to select actions that maximize $\hat{Q}$, which often correspond to spurious overestimated values, i.e., $\hat{Q}(z,a) > Q^{\beta}(z,a)$ for some $(z,a) \in \mathcal{Z\times A}$. The Bellman update $T^\pi \hat Q$ then bootstraps on these erroneous peaks, reinforcing rather than correcting them. As a result, the iterative updates can amplify the estimation error of $\hat{Q}$ as shown in \eqref{eq:bellman_gap}.

\begin{insightmwithframe}[Policy Improvement Gap] \label{prob_2}
Let $J(\pi)$ and $J(\beta)$ denote the standard discounted returns induced by
policies $\pi$ and $\beta$, respectively. Furthermore, let
$\hat{J}(\beta)
\coloneqq
\mathbb{E}_{z \sim d^{\beta}, a \sim \beta}[\hat{Q}(z,a)]$
and
$\hat{J}_{\beta}(\pi)
\coloneqq
\mathbb{E}_{z \sim d^{\beta}, a \sim \pi}[\hat{Q}(z,a)]$
represent the estimated $Q$-based surrogate objectives evaluated under the
discounted state occupancy distribution $d^\beta$ using an approximate value
function $\hat Q$. Then, the performance gap satisfies the following bound:
\begin{equation}  \label{eq:updating_gap}
    \underbrace{J(\pi) - J(\beta)}_{\text{true improvement}} \geq \frac{1}{1-\gamma}\underbrace{\big(\hat{J}_{\beta}(\pi) - \hat{J}(\beta)\big)}_{\text{estimated improvement}} - C_1 D_{\text{KL}}^{\max}(\pi\|\beta) - \underbrace{C_2 \| Q^{\beta} - \hat{Q} \|_{\infty}}_{\text{value iteration gap in}~\text{Problem}~\ref{prob_1}},
\end{equation}
where $C_1$ and $C_2$ are positive constants depending on the discount factor $\gamma$ and the boundedness of the value functions, see the analysis in Appendix~\ref{append:proof_prob_2}. 
\end{insightmwithframe}
The above result in \eqref{eq:updating_gap} highlights two key failure modes induced by the search policy. First, the policy improvement is inherently local under the metric of KL divergence: once the search policy $\pi$ deviates significantly from $\beta$, the learned value function $\hat Q$ no longer provides a reliable estimate for evaluating $\pi$, leading to biased performance estimation. In particular, the approximation gap $\|Q^\beta - \hat Q\|_\infty$ can grow as updates are performed under $\pi$ (as shown in Problem~\ref{prob_1}), since the Bellman operator $T^\pi$ propagates and reinforces errors in regions where $\hat Q$ is inaccurate. Empirically, this phenomenon is characterized by a severe inflation in cross TD error and substantial action drift relative to the base policy, as illustrated in Figure~\ref{fig:insight_misaligment}. Together, these theoretical and empirical effects culminate in a self-reinforcing cycle of distribution shift and error amplification.

\begin{figure}[ht]
    \centering
    \includegraphics[width=1.00\linewidth]{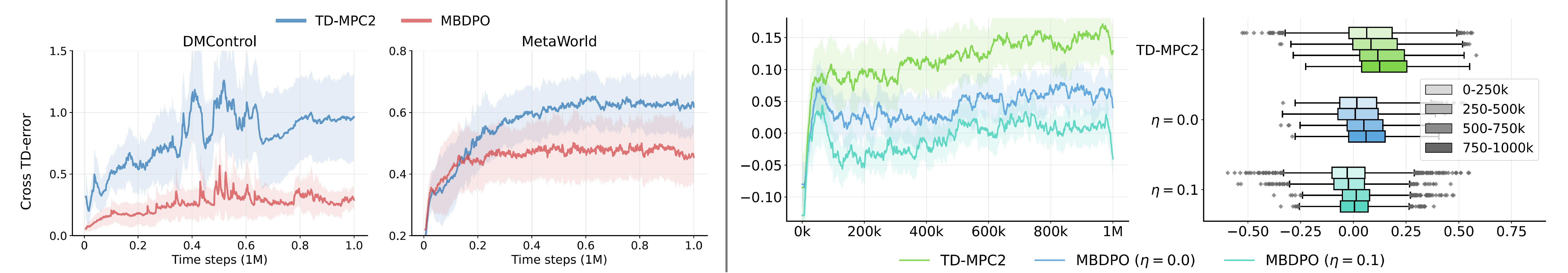}
    \caption{(Left) Cross temporal difference (TD) error comparison between MBDPO and TD-MPC2 across two benchmark suites, DMControl and MetaWorld (detailed subtask curves are provided in Figure~\ref{fig:all_TD_error}). (Right) Averaged relative action drift during training. (Right 1) Exponential moving average (EMA) of the mean drift across 8 tasks; larger values indicate greater deviation. (Right 2) Action drift distributions across different training intervals, measured by the average log-likelihood ratio 
$\frac{1}{H}\sum_{h=0}^{H}\log\frac{\pi(a_h \mid z_h)}{\beta(a_h \mid z_h)}$. Compared to TD-MPC2, the diffusion policy optimization framework significantly reduces action drift, with the contrastive variant ($\eta=0.1$) yielding the closest alignment with the policy network. The reported evaluation averages results across 8 distinct online tasks (individual subtask distributions are detailed in Figure~\ref{fig:all_method_drift}).}
\label{fig:insight_misaligment}
\end{figure}

At the core, this issue arises from the interplay between value approximation error and unconstrained policy improvement within world models: the former introduces spurious high-value estimates, while the latter allows the search policy $\pi$ to deviate significantly from $\beta$, exacerbating distribution shift. To bridge this gap, we unify search and policy optimization into a single algorithm, achieving  a monotonic policy improvement within the world model.

\subsection{Model-Based Diffusion Policy Optimization (MBDPO) in World Models}
In this subsection, we define the policy learning objective and detail the core component of our framework: MBDPO within world models. 

\textbf{Policy Objective.} In RL built upon world models, policy optimization is framed as seeking a policy $\pi_\phi$ through the following formulation: 
\begin{equation} \label{eq:policy_objective}
    \arg\max_{\pi_{\phi}} \mathbb{E}_{a_{t:t+H}\sim \pi_{\phi}(\cdot|z_t)} \big[ G(z_h, a_{t:t+H})  \big], 
\end{equation}
where
\begin{equation}
    G(z_h, a_{t:t+H}) \coloneqq \mathbb{E}_{z_{h+1}\sim \mathcal{F}(z_h, a_h), \  a_{t:t+H}\sim \pi_{\phi}(\cdot|z_t),} \big[ \sum_{h = t}^{t+H-1} \gamma^{h-t} \mathcal{R}(z_h, a_h) + \gamma^{H} \hat{Q}(z_{t+H}, a_{t+H})  \big].
\end{equation}
In prior work \citep{hansen2022temporal, hansen2023td, hansen2025learning, zhan2025bootstrap, psenka2026parallel}, the action sequence $a_{t:t+H}$ is typically obtained via model predictive path integral (MPPI). However, this approach can be problematic: MPPI is initialized with a Gaussian prior and relies on trajectories generated by the world model, which can lead to a large KL divergence from the non-search policy $\beta$. Consequently, both the Bellman update and policy improvement may enter a detrimental feedback loop, as illustrated in Problems~\ref{prob_1} and \ref{prob_2}. Essentially, this paradigm does not constitute true policy optimization; it merely performs a heuristic search that lacks the guarantees and distributional consistency required for a monotonic improvement. To address this, it is imperative to establish a principled policy optimization objective that simultaneously maximizes expected returns and constrains the search policy to the non-search base policy $\beta$:
\begin{equation} \label{eq:Lagrangian}
\begin{split}
    \max_{\pi_{\phi}} \quad & \mathbb{E}_{a_{t:t+H}\sim \pi_{\phi}(\cdot|z_t)} \big[G(z_t, a_{t:t+H}) \big] \\
    \text{s.t.} \quad & \underbrace{D_{\text{KL}}(\pi_\phi(\cdot|z_t)\|\beta(\cdot|z_t)) \leq \varepsilon}_{\text{KL constraint}},\\
    & \underbrace{z_{h+1} = \mathcal{F}(z_h, a_h), \quad \forall h \in \{t, \dots, t+H-1\}.}_{\text{latent dynamics constraint}}
\end{split}
\end{equation}
Here, the KL divergence enforces policy regularization to prevent suboptimality, while the stepwise transitions encapsulate the rollouts within the learned latent world model $\mathcal{F}$ over a planning horizon $H$. The formulation in \eqref{eq:Lagrangian} anchors the policy optimization on the distribution $\beta$, thereby mitigating the failure modes characterized in Problems~\ref{prob_1} and \ref{prob_2}. Because the exact density of $\beta$ is intractable, we introduce the following implicit energy function.

\textbf{Implicit Energy Function.} To this end, the implicit energy function $E_{\theta}$ is learned from the buffer dataset via a contrastive objective:
\begin{equation} \label{eq:minimization_of_energy}
   \min_{E_\theta} \mathcal{L}(E_\theta) =  - \mathbb{E}\Big[ \log \frac{\exp(-E_\theta(z,a))}{\exp(-E_\theta(z,a))+\sum_{j} \exp(-E_\theta(z, a_j^-))} \Big].
\end{equation}
As the dataset is large-scale and diverse, the negative samples $a_j^-$ randomly drawn from other latent states within the same batch can be closely approximated by a uniform distribution over the entire action space. Consequently, the summation in the denominator asymptotically tracks the unweighted partition function, thereby aligning the learned energy function with the log behavior policy density up to a state-dependent constant:
\begin{align*}
    -E(z,a) \approx \log \beta(a|z) + C(z).
\end{align*}
Based on the learned energy, we can reparameterize the reward function by absorbing the implicit energy function as $\tilde{\mathcal{R}} (z_h, a_h) = \mathcal{R}(z_h, a_h) - \frac{\eta}{\gamma^{h-t}} E(z_h, a_h)$, where $\eta$ is the regularized factor. 

\textbf{Diffusion Policy.} Instead of relying on MPPI merely for policy search, we adopt MBDPO to unify search and policy optimization. Unlike the conventional model-free setting, the score function $\phi$ in \eqref{eq:transition_kernel} can be estimated directly from trajectories imagined from the learned world model, without requiring a forward diffusion process:
\begin{align}
    \phi(z_t, a_{t:t+H}^\tau, \tau) = -\frac{a^{\tau}_{t:t+H}}{1-\Bar{\alpha}^\tau} + \frac{\sqrt{\Bar{\alpha}^\tau}}{1-\Bar{\alpha}^\tau} \mathbb{E}_{a_{t:t+H}^{0|\tau} \sim \mathcal{N}(\frac{1}{\sqrt{\Bar{\alpha}^\tau}} a_{t:t+H}^\tau, \frac{1-\Bar{\alpha}^\tau}{\Bar{\alpha}^\tau} I )}  \big[w(a^{0|\tau}_{t:t+H}) \cdot a_{t:t+H}^{0|\tau}  \big]. \label{eq:score_matching}
\end{align}
Here, $a_{t:t+H}^{0|\tau}$ is a random variable that follows the sample posterior distribution $\mathcal{N}(\frac{1}{\sqrt{\Bar{\alpha}^\tau}} a_{t:t+H}^\tau, \frac{1-\Bar{\alpha}^\tau}{\Bar{\alpha}^\tau} I )$, and $\bar{\alpha}^\tau = \prod_{k=1}^\tau \alpha^k$. The importance weight $w(a_{t:t+H}^{0|\tau})$ is defined as: 
\begin{equation} \label{eq:importance_weight}
    w(a_{t:t+H}^{0|\tau}) = \frac{\exp(\frac{\tilde{G}(z_t, a_{t:t+H}^{0|\tau})}{\kappa})}{\mathbb{E}_{a_{t:t+H}^{0|\tau} \sim \mathcal{N}(\frac{1}{\sqrt{\Bar{\alpha}^\tau}} a_{t:t+H}^\tau, \frac{1-\Bar{\alpha}^\tau}{\Bar{\alpha}^\tau} I )}[\exp(\frac{\tilde{G}(z_t, a_{t:t+H}^{0|\tau})}{\kappa})]},
\end{equation}
where $\kappa$ is the temperature factor, with $$\tilde{G}(z_t, a_{t:t+H}^{0|\tau}) \coloneqq \sum_{h = t}^{t+H-1} \gamma^{h - t} \tilde{\mathcal{R}}(z_h, a^{0|\tau}_h) + \gamma^{H} \hat{Q}(z^{0|\tau}_{t+H}, a^{0|\tau}_{t+H}) - \eta E(z^{0|\tau}_{t+H}, a^{0|\tau}_{t+H})$$ representing the cumulative return regularized with the implicit energy. By substituting the score function into the reverse-time transition kernel in \eqref{eq:transition_kernel}, the action sequence $a_{t:t+H}^{\tau-1}$ at $\tau-1$ is
\begin{equation}\label{eq:sampling}
    a_{t:t+H}^{\tau-1}
=
\frac{1}{\sqrt{\alpha^\tau}}
\left(
a_{t:t+H}^{\tau}
+
(1-\alpha^{\tau})
\phi(z_t, a_{t:t+H}^{\tau}, \tau)
\right)
+
\sigma(\tau)\epsilon,
\quad
\epsilon \sim \mathcal{N}(0, I)
\end{equation}
Lemma~\ref{lemma:score_function_proof} in the Appendix presents the detailed derivation of equations~\eqref{eq:score_matching}--\eqref{eq:sampling}. 
% This formulation clarifies the role of the world model. 
% \textit{The Gaussian proposal in~\eqref{eq:score_matching} determines which clean action sequences are plausible denoising candidates, while the world model determines how these candidates should be weighted by evaluating their predicted trajectory returns. 
% Thus, the world model enters the reverse transition kernel through the score function (shown in \eqref{eq:transition_kernel}): high-return imagined trajectories receive larger weights, and the resulting score biases the reverse transition kernel toward action sequences that are both diffusion-consistent and high-value under the latent dynamics.}
To generate the optimized action sequence, we start from a Gaussian prior $a_{t:t+H}^{N}\sim \mathcal{N}(0,I)$ and iteratively apply the reverse-time transition kernel defined in~\eqref{eq:policy_refinement} and \eqref{eq:transition_kernel}. At each denoising step, the score in~\eqref{eq:score_matching} shapes the transition kernel based on the returns evaluated from imagined trajectories within the world model. In this sense, the world model acts as a score-field corrector for policy optimization. Once the exact score function is matched in~\eqref{eq:score_matching}, the optimal policy for the KL-constrained objective in~\eqref{eq:Lagrangian} is derived, as stated in the following theorem.
% To generate the optimized action sequence, we initialize from a Gaussian prior $a_{t:t+H}^{N}\sim \mathcal{N}(0,I)$ and iteratively apply the reverse-time transition kernel in~\eqref{eq:policy_refinement} and \eqref{eq:transition_kernel}. 
% At each denoising step, the score in~\eqref{eq:score_matching} reshapes the current action distribution using evaluated returns from the learned world model. Consequently, reverse diffusion acts as model-based policy optimization in action-sequence space: it progressively transforms a Gaussian prior into an optimized policy distribution concentrated around high-value trajectories.  
% Moreover, this sampling procedure is equivalent to solving a KL-constrained policy optimization problem for objective~\eqref{eq:policy_objective}, stated as the following theorem. 
% This expectation is estimated via Monte Carlo sampling using trajectories imagined by the world model $z_{h+1} = \mathcal{F}(z_h, a_h)$. To generate the optimized action sequence, we initialize from $a_{t:t+H}^N \sim \mathcal{N}(0, I)$ and iteratively apply the reverse diffusion step in \eqref{eq:transition_kernel} by evaluating the score function. The derivation of the score in \eqref{eq:score_matching} is detailed in Lemma~\ref{lemma:score_function_proof} in Appendix. Notably, this sampling procedure is equivalent to performing KL-constrained policy optimization for objective~\eqref{eq:policy_objective}, where the resulting score function is distilled into a neural network to generate the optimized policy $\pi_{\phi}$.

\begin{theoremwithframe}[Equivalence of Score Matching and Policy Optimization]\label{theorem:score_matching_policy}
Consider the policy refinement task in \eqref{eq:policy_refinement} and the energy minimization objective in \eqref{eq:minimization_of_energy}. For a given temperature $\kappa > 0$ and any latent state $z_t \in \mathcal{Z}$, let $\pi^*_\phi$ be the solution to the KL-constrained policy optimization problem formulated in \eqref{eq:Lagrangian}. 
The optimum is attained if and only if the score matching condition in \eqref{eq:score_matching} is satisfied. Specifically, under the transition dynamics of the world model, the optimal policy $\pi^*_\phi$ can be approximated by the following Gibbs distribution:
\begin{align} \label{eq:Gibbs}
\pi^*_{\phi}(a_{t:t+H}| z_t) \propto \left(\prod_{h=t}^{t+H} \beta(a_h|z_h)\right) \exp\left(\frac{G(z_t,a_{t:t+H})}{\kappa}\right) \prod_{h=t}^{t+H-1} \mathbb{1}_{\{z_{h+1}=\mathcal F(z_h,a_h)\}},
\end{align}
where $\mathbb{1}_{\{\cdot\}}$ denotes the indicator function enforcing the latent dynamics, and the partition function is omitted for brevity. A detailed proof is provided in Appendix~\ref{append:proof_theorem_1}.
\end{theoremwithframe}
Theorem~\ref{theorem:score_matching_policy} provides a unified perspective that bridges search and policy optimization: rather than explicitly parameterizing a policy network, optimization can be implicitly executed through iterative rollouts and score-based sampling within the world model.

\textbf{Why Diffusion Policy?} The diffusion policy intrinsically bridges score matching within the world model and true policy optimization, unlocking the scaling potential of world models through a dual mechanism.
First, the world model continuously bootstraps trajectories by generating imagined rollouts to facilitate score matching; crucially, as established in Theorem~\ref{theorem:score_matching_policy}, this score matching process is mathematically equivalent to performing true policy optimization. Second, while a standard KL-constrained objective in \eqref{eq:Lagrangian} is typically intractable due to the intractable density function $\beta$, this framework bypasses the intractable integration by training an implicit energy function that is naturally absorbed into imagined trajectories within world models.

Since the optimal policy follows the Gibbs distribution in~\eqref{eq:Gibbs}, MBDPO directly addresses the bottlenecks identified in Problems~\ref{prob_1} and \ref{prob_2}. Specifically, it orchestrates both value iteration and policy improvement by ensuring that policy updates remain strictly bounded within a local trust region, thereby overcoming the suboptimality bottlenecks inherent in previous methods.

% \textit{From this perspective, the score matching in \eqref{eq:score_matching} based on the learned world model is equivalent to performing policy optimization by adjusting the stepwise transition kernel in \eqref{eq:transition_kernel}. Ultimately, this framework unifies search and policy optimization, executing the policy optimization via search (rollout trajectories) within the world model.} Theorem~\ref{theorem:score_matching_policy} addresses the bottlenecks identified in Problems~\ref{prob_1} and \ref{prob_2}: (1) it ensures that policy improvement remains within a trust region (constrained by $\epsilon$-KL divergence), enabling monotonic policy improvement as value estimation error shrinks, consistent with the gap analysis in~\eqref{eq:updating_gap}; (2) MBDPO enables true policy optimization directly within the world model, where the diffusion policy derived from the score function is optimized in tandem with the refinement of the learned world model and value estimation. Formal proof and technical discussion are provided in Appendix~\ref{append:proof_theorem_1}. 

\subsection{Practical Algorithm} \label{sec:practical_algorithm}
For the practical implementation of MBDPO, we adopt a joint-training paradigm that synchronizes world model learning with diffusion policy optimization (see Algorithm~\ref{alg:MBDPO} in Appendix~\ref{append:practical_implement}). We maintain a replay buffer $\mathcal{B}$ to store trajectory segments, supporting various training configurations including online, offline, and offline-to-online settings. The framework is jointly trained through an alternating optimization scheme that iteratively updates (i) the world model and value iteration, and (ii) the diffusion policy optimization. 

\textbf{(i) World Model and Value Iteration:}  Utilizing trajectory segments sampled from the buffer $\mathcal{B}$, we optimize the core components of the world model defined in \eqref{eq:components_mapping}, namely the \textit{\textcolor{cyan}{encoder $\mathcal{E}$}}, \textit{\textcolor{cyan}{latent dynamics $\mathcal{F}$}}, and \textit{\textcolor{cyan}{reward predictor $\mathcal{R}$}}. Simultaneously, the \textit{\textcolor{cyan}{value function $\hat{Q}$}} is updated via one-step temporal difference (TD) learning, while the \textit{\textcolor{cyan}{implicit energy $E_{\theta}$}} is refined according to the objective in \eqref{eq:minimization_of_energy}. The joint minimizer for $\mathcal{E, F, R}$, $\hat{Q}$ and $E_\theta$ can be written as 
\begin{equation} \label{eq:joint_minimizer}
    \min \mathbb{E}_{\tau \sim \mathcal{B}} \bigg[ \sum_{h = t}^{t+H} \gamma^{h-t} \bigg( \| \mathcal{F}(z_h, a_h) - \text{sg}(\mathcal{E}(s_{h+1})) \|_2^2 +  \text{CE}(\hat{r}_h, r(s_h, a_h)) + \text{TD}(\hat{Q}(z_h, a_h)) + \mathcal{L}(E_{\theta}) \bigg) \bigg], 
\end{equation}
where $\text{sg}(\cdot)$ denotes the \texttt{stop-grad} operator and $\text{CE}(\cdot, \cdot)$ represents the cross-entropy loss, utilized to capture sparse or discrete rewards across diverse environments. Following Lillicrap et al. \citep{lillicrap2016continuous}, the one-step TD loss is defined as $\text{TD}(\hat{Q}(z_h, a_h)) \coloneqq \text{CE}\big(\hat{Q}(z_h, a_h), r_h + \gamma \bar{Q}(z_{h+1}, a_{h+1})\big)$, where $\bar{Q}$ is the target network maintained as an exponential moving average (EMA) of $\hat{Q}$. The energy loss $\mathcal{L}(E_{\theta})$ is optimized according to \eqref{eq:minimization_of_energy}, using actions sampled from other latent states as negative pairs. Other training configurations, such as normalization and encoder architectures, strictly follow the TD-MPC framework to highlight the significance of scalable policy learning within world models. 

\textbf{(ii) Diffusion Policy Optimization.} As shown in \eqref{eq:score_matching} and Theorem~\ref{theorem:score_matching_policy}, the policy optimization relies on the imagined samples from the learned world models under a weighted expectation (red shadow of~\eqref{eq:monte_carlo}). Since the expectation cannot be computed analytically, we approximate it using Monte Carlo sampling. MBDPO employs a weighted Monte Carlo estimator based on a proposal distribution $q(a_{t:t+H}^{0|\tau} | z_t)$, chosen as a Gaussian $ \mathcal{N}(\frac{1}{\sqrt{\Bar{\alpha}^\tau}} a_{t:t+H}^\tau, \frac{1-\Bar{\alpha}^\tau}{\Bar{\alpha}^\tau} I )$ according to \eqref{eq:importance_weight}. The asymptotically unbiased expectation can thus be approximated using $T$ Monte Carlo samples $\{ a_{t:t+H}^{0|\tau, (i)} \}_{i = 1}^T$ drawn from $q(a_{t:t+H}^{0|\tau} | z_t)$:
\begin{equation} \label{eq:monte_carlo}
    \phi(z_t, a_{t:t+H}^\tau, \tau) \approx -\frac{a^{\tau}_{t:t+H}}{1-\Bar{\alpha}^\tau} + \frac{\sqrt{\Bar{\alpha}^\tau}}{1-\Bar{\alpha}^\tau}  \cdot \sum_{i = 1}^T \colorbox{red!10}{$w(a^{0|\tau, (i)}_{t:t+H})$} \cdot a_{t:t+H}^{0|\tau, (i)}
\end{equation}
with \colorbox{red!10}{$w(a^{0|\tau, (i)}_{t:t+H}) = \frac{\exp(\frac{\tilde{G}(z_t, a_{t:t+H}^{0|\tau, (i)})}{\kappa})}{\sum_{i=1}^T \exp(\frac{\tilde{G}(z_t, a_{t:t+H}^{0|\tau, (i)})}{\kappa}) }$} for samples $a_{t:t+H}^{0|\tau, (i)} \sim q(a_{t:t+H}^{0|\tau} | z_t)$. These weights rely on imagined trajectories rolled out by the learned latent dynamics $\mathcal{F}$ to evaluate the cumulative return $\tilde{G}$, as shown in Figure~\ref{fig:mbdpo}. To avoid the heavy computational overhead of online rollouts during deployment, this Monte Carlo estimation is exclusively employed during score network training, such that
\begin{equation} \label{eq:score_fitting_loss}
    \min \mathbb{E}\big[\| \hat{\phi}(z_t, a_{t:t+H}^\tau, \tau) -  \phi(z_t, a_{t:t+H}^\tau, \tau) \|_2^2 \big], 
\end{equation}
where $\hat{\phi}(z_t, a_{t:t+H}^\tau, \tau)$ is the learned score function. During inference, the score network serves as an amortized solution for efficient policy generation, plugging $\hat{\phi}$ into \eqref{eq:sampling} for iterative sampling of the target action sequence. This framework benefits from provable theoretical guarantees: as established in \citep{chen2022sampling, cheng2026does}, the error relative to the target Gibbs distribution in \eqref{eq:Gibbs} decays at a polynomial rate $\text{Poly}(N, T, L_{\phi})$, where $N, T,$ and $L_{\phi}$ denote the diffusion steps, Monte Carlo sample size, and the score function's Lipschitz constant, respectively.

\section{Experiments} 

\textbf{Benchmarks and Tasks.} We evaluate MBDPO across four major benchmarks encompassing 121 diverse control tasks: DMControl \citep{tassa2018deepmind}, MetaWorld \citep{yu2020meta}, ManiSkill2 \citep{gu2023maniskill2}, and MyoSuite \citep{caggiano2022myosuite}. Visual demonstrations for each benchmark are provided in Figures~\ref{fig:demo_dmcontrol}--\ref{fig:demo_myosuite}. Our evaluation includes challenging variations and subsets such as Visual RL (image-based DMControl), Locomotion, and PickYCB. These tasks demand mastery over high-dimensional state-action spaces, sparse rewards, and multi-object manipulation. Furthermore, they include physiologically accurate musculoskeletal control and complex locomotion with Dog and Humanoid embodiments, covering a wide spectrum of difficulty levels. 

\textbf{Baselines.} We compare MBDPO against several state-of-the-art and data-efficient RL algorithms: (1) Soft Actor-Critic (SAC) \citep{haarnoja2018soft}, a representative off-policy baseline; (2) DreamerV3 \citep{hafner2023mastering}, a leading world-model-based approach; (3) TD-MPC \citep{hansen2022temporal}, the first foundational work of the TD-MPC series; and (4) TD-MPC2 \citep{hansen2023td}, which shares its algorithmic core with the most recent advancements in the field \citep{hansen2025learning}. 

\textbf{Result Analysis.} Our experimental evaluation focuses on five key dimensions:
(1) a comparative performance evaluation of MBDPO against leading model-free and model-based RL baselines;
(2) an investigation into the scaling behaviors of MBDPO;
(3) an empirical verification of whether MBDPO mitigates the bottlenecks identified in Problems~\ref{prob_1} and \ref{prob_2};
(4) an assessment of the algorithm’s generalization across offline pre-training, online, and offline-to-online settings; and
(5) a comprehensive evaluation of the computational efficiency of MBDPO.

\begin{figure}[ht]
    \centering
    \includegraphics[width=0.99\linewidth]{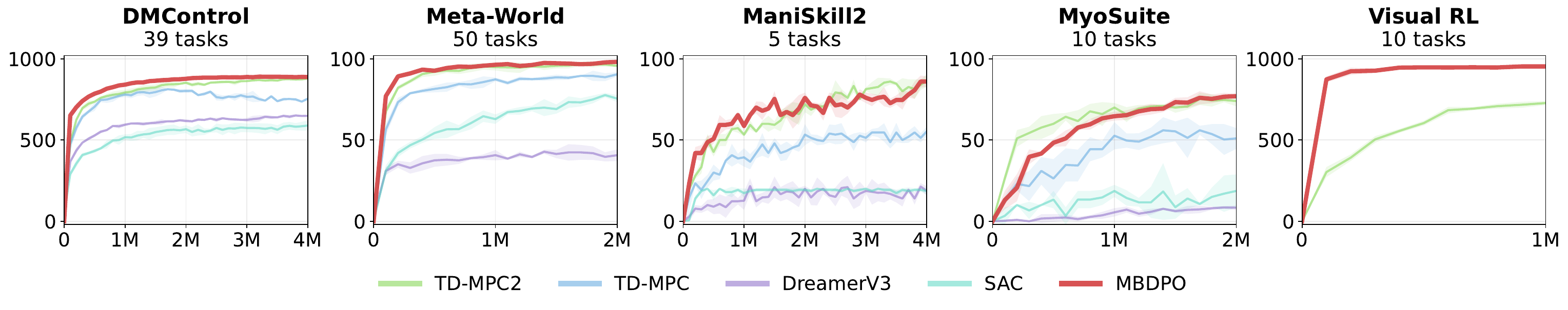}
    \caption{\textbf{Aggregate performance across four benchmarks in the online setting: DMControl, MetaWorld, ManiSkill2, and MyoSuite.} Detailed learning curves for each subtask are provided in Figures~\ref{fig:DMControl}--\ref{fig:myosuite} and Appendix~\ref{append:additional_results_online}.}
    \label{fig:overall_online}
\end{figure}

% \vspace{-1em}

\begin{figure}[ht]
    \centering
    \includegraphics[width=0.92\linewidth]{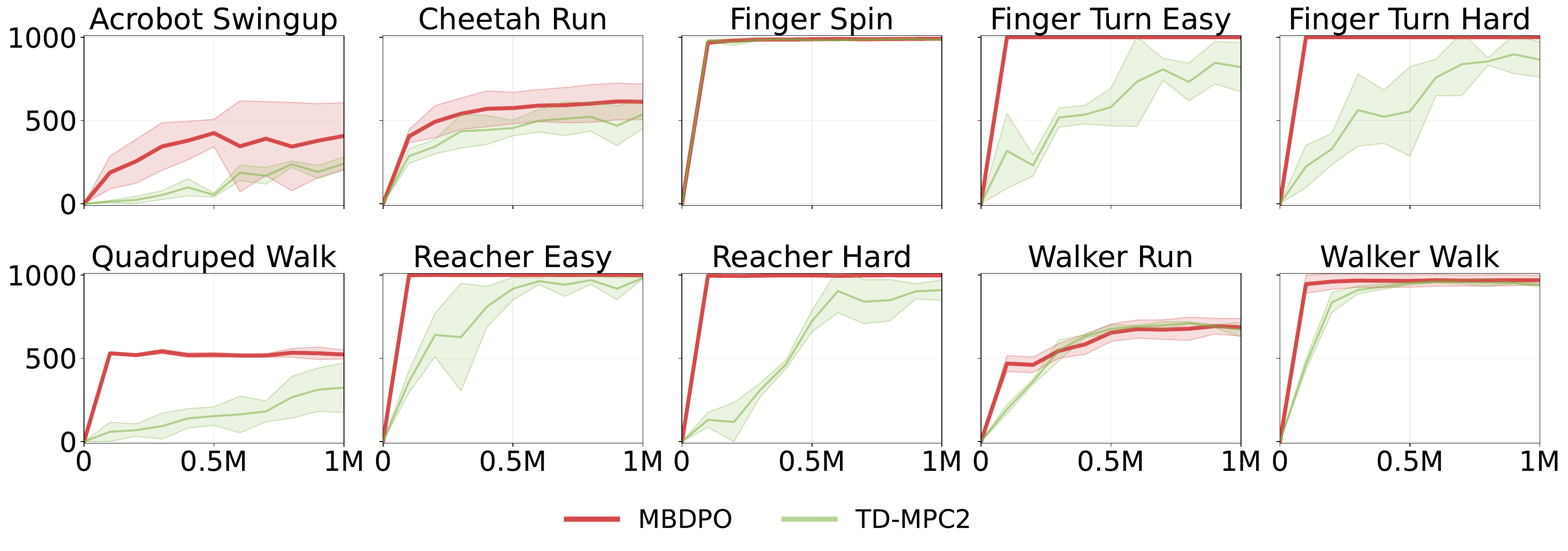}
    \caption{\textbf{Performance comparison between TD-MPC2 and MBDPO across 10 visual control tasks.} Results are averaged over $5$ random seeds for each task.}
    \label{fig:visual}
\end{figure}

\textit{1. MBDPO consistently achieves better results across all evaluated benchmarks, demonstrating significant performance gains over both model-free and model-based baselines.} As illustrated in the aggregate performance in Figure~\ref{fig:overview}, our algorithm consistently attains higher rewards in online settings. Notably, in high-dimensional visual control tasks, MBDPO significantly elevates the performance bottleneck compared to prior methods. Another compelling observation is that across nearly all tasks, MBDPO exhibits accelerated policy improvement with substantially lower variance during the early training stages (see Figures~\ref{fig:overall_online} and \ref{fig:visual}). We attribute this advantage to our principled policy optimization framework, which, unlike pure search methods based on a sampling-based planner, provides a more stable basis for diffusion policy refinement from the onset of training (seamlessly improved with world models and value functions).

\begin{wrapfigure}{r}{0.7\linewidth}
\centering
% \vspace{-10pt}
\includegraphics[width=\linewidth]{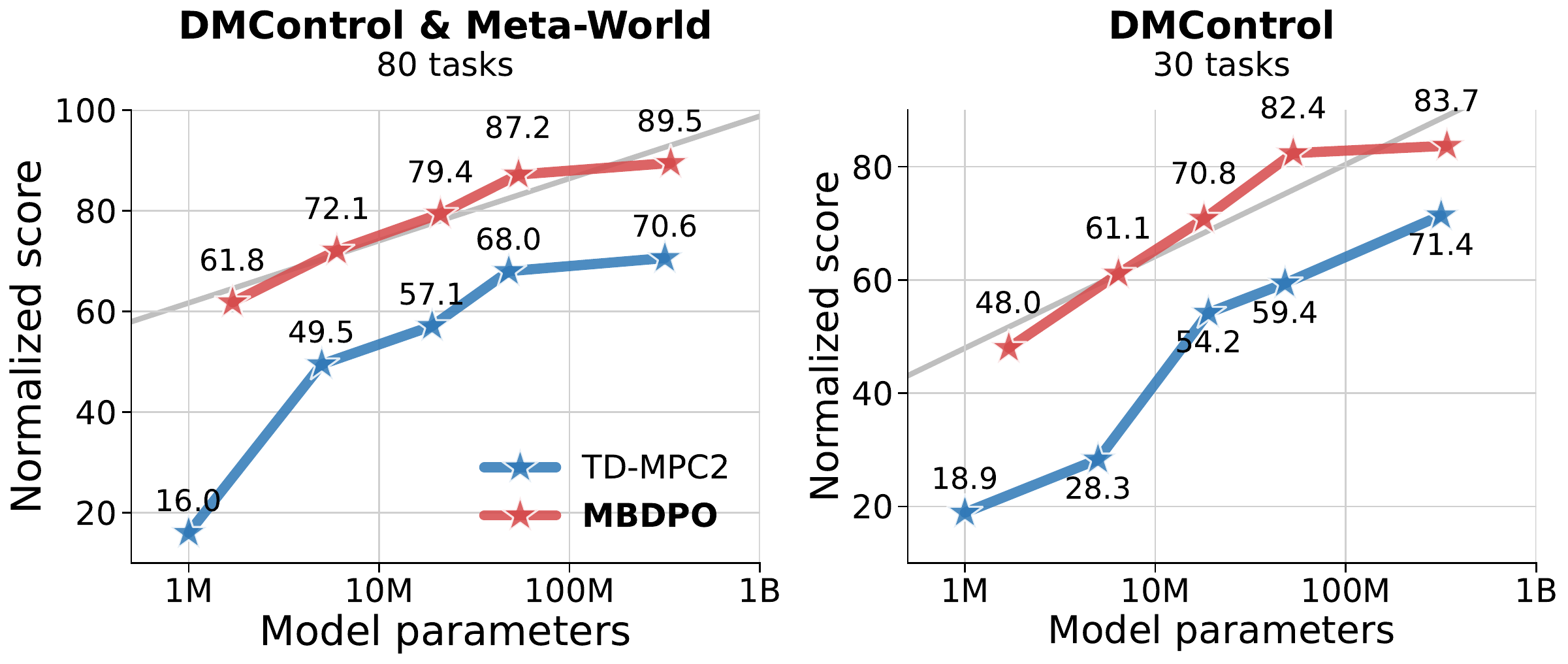}
\vspace{-10pt}
\caption{\textbf{Massively multi-task world models in the offline pretraining setting.} Normalized score as a function of model
size on the two 80-task and 30-task datasets. MBDPO shows sharper scaling behavior with model capacity.}
\label{fig:offline}
\vspace{-10pt}
\end{wrapfigure}

\textit{2. MBDPO unlocks the potential of world models for policy learning, enabling effective multi-task offline pretraining and demonstrating a monotonic scaling curve as model capacity scales from $1.7$M to $340$M.} As illustrated in Figure~\ref{fig:offline}, several key scaling properties emerge. Notably, MBDPO demonstrates a more pronounced advantage at smaller model scales. We attribute this performance to our principled policy optimization, which effectively regularizes the policy (e.g., via the implicit KL constraint with the behavioral policy), ensuring stable performance even when the world model is not highly accurate. As model capacity increases, the performance gap between MBDPO and TD-MPC2 gradually narrows. This is likely because the enhanced accuracy of a larger world model significantly benefits the trajectory-based policy search used in TD-MPC2. In contrast, MBDPO focuses on direct diffusion policy optimization, which consistently evolves alongside the world model. Interestingly, on the 30-task dataset with $340$M parameters, the performance of both methods appears to decrease. We hypothesize that at this scale, the dataset capacity has been potentially exhausted by the model, leading to saturated performance.

\begin{figure}[ht]
    \centering
    \includegraphics[width=0.99\linewidth]{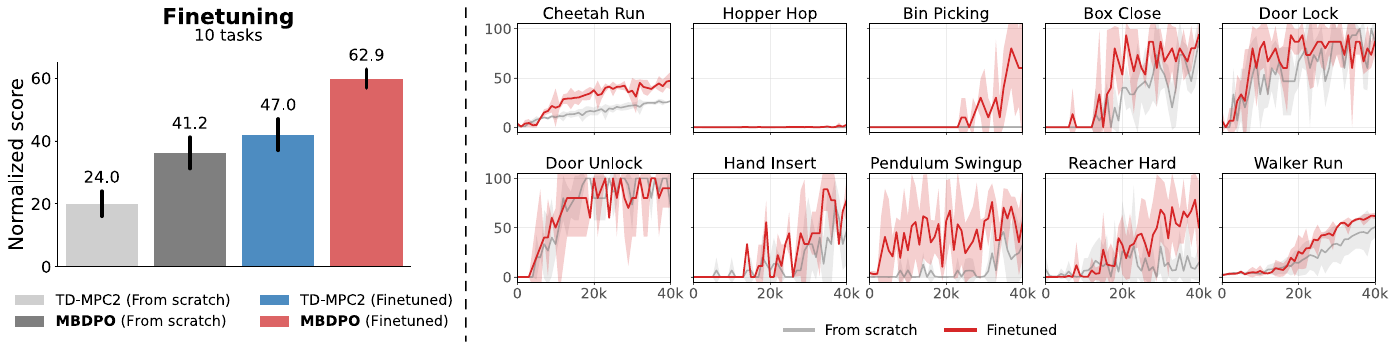}
    \caption{\textbf{Overview of offline-to-online (O2O) performance.} (Left) Comparison between MBDPO and TD-MPC baselines. (Right) Comparison between training from scratch and O2O fine-tuning. Results show that fine-tuning a generalist agent yields superior performance with significantly less data, highlighting the high sample efficiency and transferability of our framework. Note: The suboptimal results in ``Hopper Hop'' can be attributed to a severe deviation in the representation's data distribution (as illustrated by the latent embeddings in Figure~\ref{fig:embeddings}), which subsequently impedes effective policy optimization.}
    \label{fig:o2o_specific_task}
\end{figure}

\textit{3. MBDPO mitigates the fundamental bottlenecks (Problems \ref{prob_1} and \ref{prob_2}) induced by the misalignment between search and value learning.}
We empirically validate this by monitoring both the cross TD error and the average relative policy drift throughout training, as illustrated in Figure~\ref{fig:insight_misaligment}. Crucially, MBDPO exhibits a significantly lower and more stable cross TD error compared to TD-MPC2. This directly demonstrates that our framework effectively minimizes the value iteration error described in \eqref{eq:bellman_gap} (Problem~\ref{prob_1}), preventing the vicious cycle of value error amplification in out-of-distribution regions. Concurrently, we evaluate the policy drift, quantified as the log-ratio between the search-based distribution $\pi$ and the current policy network $\beta$. As training progresses, our diffusion policy regularized by the implicit energy function exhibits minimal volatility and smoothly converges toward the optimized distribution (approaching zero after $1$M steps). In sharp contrast, heuristic search via MPPI maintains high volatility and severe policy drift throughout the entire training stage. This persistent misalignment confirms that unconstrained heuristic search fails to achieve true, stable policy optimization, leading to the distribution shift analyzed in \eqref{eq:updating_gap} (Problem~\ref{prob_2}). This phenomenon is particularly magnified in offline settings (Figure~\ref{fig:offline}), where MBDPO maintains a highly stable optimization path even when deployed with limited-capacity world models. By constraining the KL divergence through the implicit energy, MBDPO suppresses both the value iteration gap and policy distributional drift, achieving a principled policy optimization that is highly congruent with our theoretical analysis in Theorem~\ref{theorem:score_matching_policy}.

\begin{table}[ht]
\centering
\small
\caption{\textbf{Profiling training efficiency results for online-from-scratch and multi-task offline pretraining.}
(a) Online training time is measured per environment step (milliseconds) on a single NVIDIA A800-80GB, see the task-specific results in Table~\ref{tab:profiling_training_results}.
(b) Offline training time is reported in hours for 10M training steps on a single NVIDIA-H200.
Values are reported as mean $\pm$ std.}
\label{tab:profiling_summary_combined}
\setlength{\tabcolsep}{4pt}
\renewcommand{\arraystretch}{1.08}
\begin{minipage}[t]{0.48\linewidth}
\centering
\caption*{\textbf{(a) Online-from-scratch training time}}
\begin{tabular}{lccc}
\toprule
\textbf{Domain} & \textbf{\#} & \textbf{TD-MPC2} & \textbf{MBDPO} \\
\midrule
MetaWorld   & 50 & $1.08 \pm 0.16$ & $1.08 \pm 0.15$ \\
DMControl   & 39 & $0.78 \pm 0.71$ & $0.76 \pm 0.71$ \\
ManiSkill2  & 5  & $6.50 \pm 1.83$ & $6.91 \pm 2.51$ \\
MyoSuite    & 10 & $1.58 \pm 0.22$ & $1.54 \pm 0.19$ \\
Visual RL   & 10 & $3.55 \pm 0.51$ & $3.04 \pm 0.37$ \\
\bottomrule
\end{tabular}
\end{minipage}
\hfill
\begin{minipage}[t]{0.50\linewidth}
\centering
\caption*{\textbf{(b) Multi-task offline training time}}
% \vspace{-0.5em}
\begin{tabular}{lccc}
\toprule
\textbf{Parameters}  & \textbf{30 tasks} & \textbf{80 tasks} & \textbf{80/30} \\
\midrule
1.7M  & $34.8 \pm 0.6$  & $39.9 \pm 0.4$  & $1.15\times$ \\
6M    & $35.0 \pm 0.5$  & $40.5 \pm 0.6$  & $1.16\times$ \\
21M   & $39.8 \pm 0.3$  & $42.6 \pm 0.4$  & $1.07\times$ \\
54M   & $52.6 \pm 0.1$  & $69.6 \pm 0.7$  & $1.32\times$ \\
340M  & $202.9 \pm 0.9$ & $206.8 \pm 0.8$ & $1.02\times$ \\
\bottomrule
\end{tabular}
\end{minipage}

\end{table}

\textit{4. MBDPO is highly versatile, providing unified support for offline pre-training, online learning, and offline-to-online (O2O) transfer settings.} Across these diverse paradigms, MBDPO consistently achieves superior performance, as evidenced in Figures~\ref{fig:overall_online}, \ref{fig:offline}, and \ref{fig:o2o_specific_task}. Notably, our method demonstrates significant advantages in the offline regime, where stable optimization is more critical than in other settings. To evaluate generalization, we extended MBDPO to an O2O fine-tuning scenario—pre-training on 70 tasks followed by fine-tuning on 10 unseen tasks with only $40k$ interactions per task. We found that MBDPO significantly outperforms training from scratch, verifying that our pre-trained generalist agent can rapidly adapt to become an expert in novel environments. This enhanced efficiency in policy optimization underscores MBDPO’s effectiveness across varied scenarios. Specifically, the comparison between from-scratch and fine-tuning curves in Figure~\ref{fig:o2o_specific_task} (right) demonstrates that our framework can be seamlessly extended to new tasks with minimal data requirements.

\textit{5. MBDPO achieves efficient training during both the online-from-scratch and multi-task offline paradigms.} According to Table~\ref{tab:profiling_summary_combined}, MBDPO incurs negligible overhead in online-from-scratch training. Across $121$ tasks from five domains, its per-step training time is comparable to TD-MPC2, and is faster on DMControl, MyoSuite, and Visual RL, reducing the Visual RL average from $3.55$ ms to $3.04$ ms. For multi-task offline training, MBDPO scales efficiently from 30 to 80 tasks, with only a $1.02\times$--$1.32\times$ increase in total training time. Notably, the 340M model shows nearly unchanged cost ($202.9$ vs. $206.8$ hours), indicating strong scalability to broader multi-task settings.

\begin{wrapfigure}{r}{0.5\linewidth}
\centering
% \vspace{-10pt}
\includegraphics[width=\linewidth]{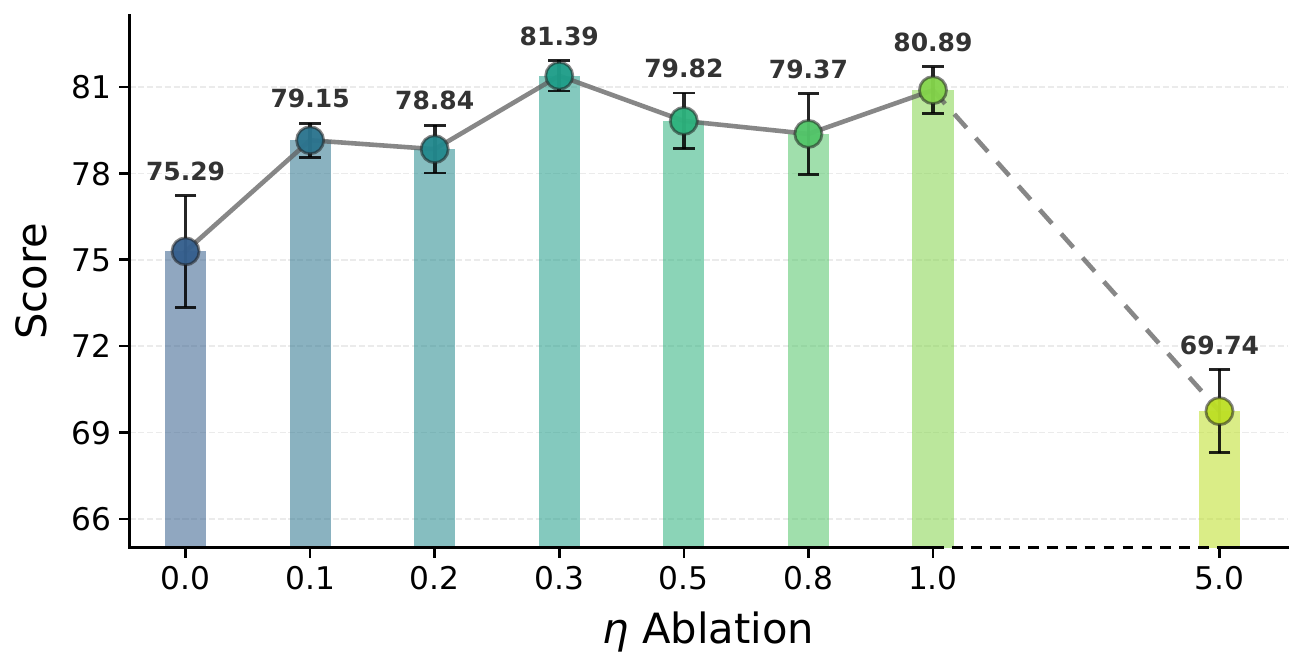}
\vspace{-15pt}
\caption{\textbf{Ablation study of the factor $\eta$.} Experiments are conducted on the 80-task multi-task setting with a 21M parameters model, where $\eta$ is varied within the range $[0, 5]$.
}
\label{fig:abla_temperature}
\vspace{-20pt}
\end{wrapfigure}

\textbf{Ablation Study 1: Regularized Factor $\eta$.}
Figure~\ref{fig:abla_temperature} demonstrates the importance of the implicit energy function $E$ in diffusion policy optimization. In the absence of this energy function, the policy lacks an effective KL constraint, leading to suboptimal performance as identified in Problems~\ref{prob_1} and \ref{prob_2}. While the performance remains stable for $\eta$ within the range $[0.3, 1]$, scaling $\eta$ up to $5$ causes the performance to drop below 70. This degradation occurs because an  excessively large $\eta$ induces an overly dominant behavioral cloning effect, which ultimately restricts the policy's exploitation of higher-value regions within the learned world model.  An interesting phenomenon is that even when the regularized factor is set to $\eta=0$, our method still outperforms TD-MPC2 (75.29 vs. 57.1). This reflects that diffusion policy optimization, by its nature, can lead to a better policy than heuristic search.

\textbf{Ablation Study 2: Diffusion Denoise Timesteps $N$ and Monte Carlo Samples $T$.} We conduct ablation studies on several challenging tasks from DMControl and MetaWorld, with results illustrated in Figure~\ref{fig:Monte_Carlo_samples} and Figure~\ref{fig:abla_diffusion_step}. The findings are straightforward: increasing the number of diffusion timesteps and Monte Carlo samples consistently improves performance. This aligns with our theoretical error analysis—specifically, the polynomially decaying error $\text{Poly}(N, T, L_{\phi})$—confirming that the distributional error decreases polynomially with both parameters. Conversely, because sampling is executed via a discrete timestep, inference runtime scales linearly with the number of diffusion timesteps. However, due to the parallelizable nature of Monte Carlo sampling, increasing $T$ has only a marginal impact on execution time. 

\begin{figure}[ht]
    \centering
    \includegraphics[width=0.78\linewidth]{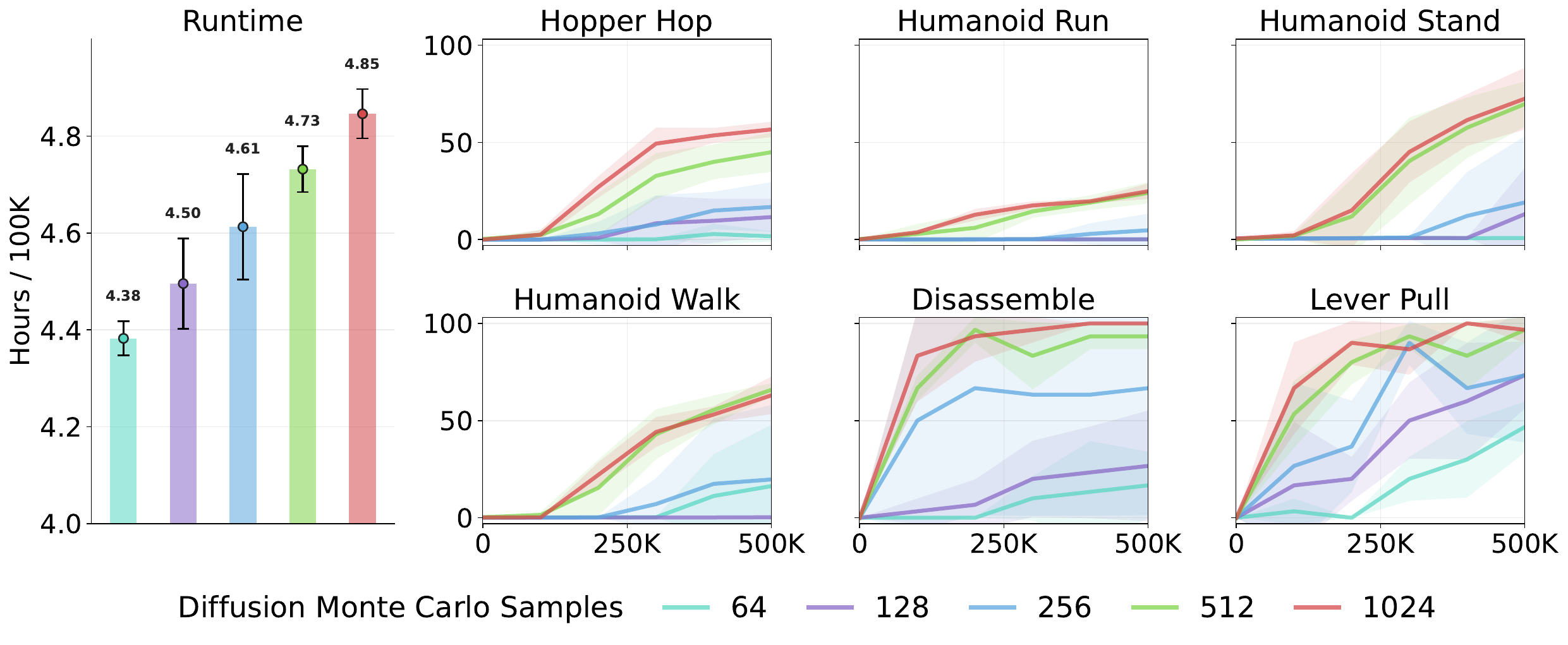}
    \caption{\textbf{The ablation study of Monte Carlo samples in the policy with 3 random seeds.} (Left) Training runtime comparison under different sample numbers. (Right) Episode reward versus training steps and sample numbers.}
    \label{fig:Monte_Carlo_samples}
\end{figure}

\begin{figure}[ht]
    \centering
    \includegraphics[width=0.78\linewidth]{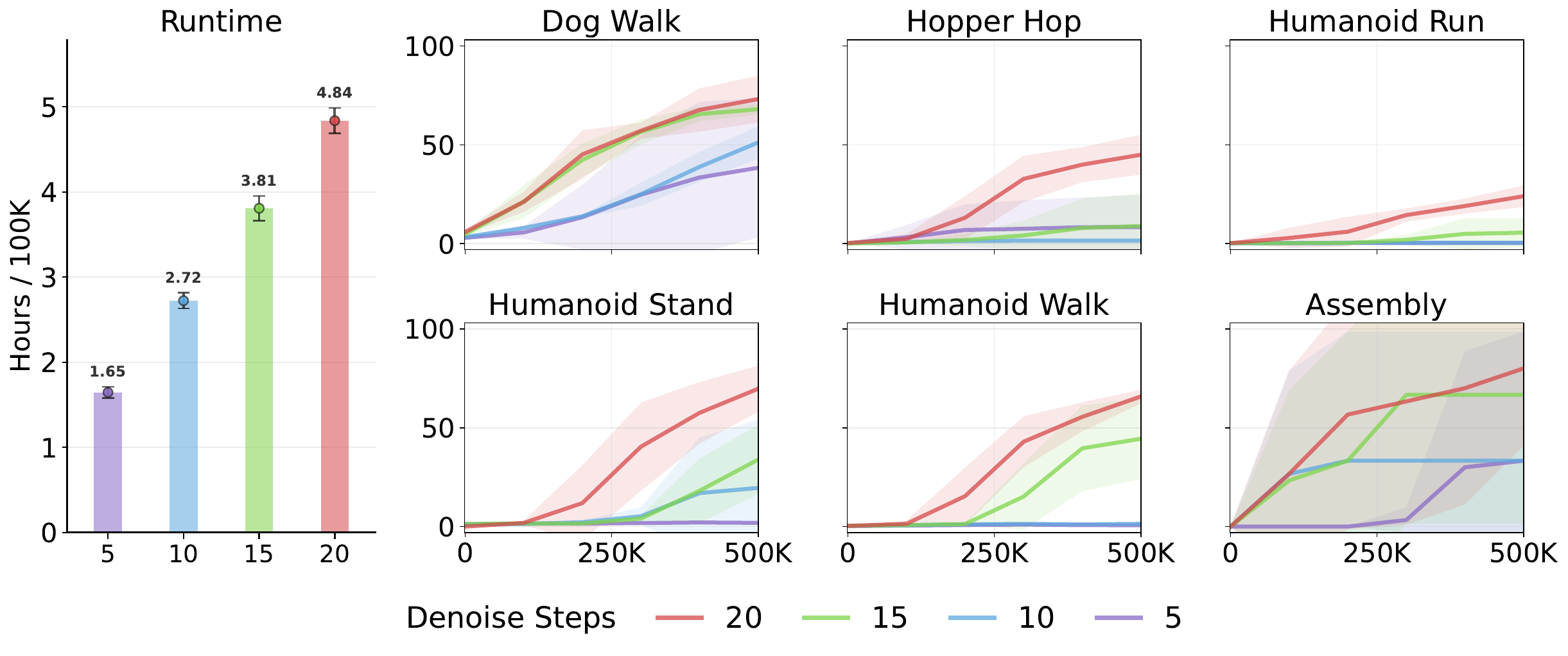}
    \caption{\textbf{The ablation study of diffusion denoise timesteps in the policy with 3 random seeds.}
(Left) Training runtime comparison under different numbers of diffusion timesteps.
(Right) Episode reward versus training steps for different diffusion timesteps.}
    \label{fig:abla_diffusion_step}
\end{figure}

\vspace{-1em}

\section{What We Learned?}

\textbf{1. Scalability of Diffusion Policy Built Upon World Model.}
Our primary insight is that diffusion models act as an intrinsic interface bridging world models and policy optimization. By leveraging imagined trajectories, the diffusion policy can bootstrap its behaviors to achieve continual value iteration and policy improvement. Crucially, we find that the joint learning of the world model and RL policy under a simple loss function enforces a stronger capture of underlying causality within the latent space.
This emergent causality is distinctly mirrored in our latent visualizations, where the policy naturally uncovers periodic closed-loop manifolds for cyclical tasks and goal-oriented trajectories for sequential manipulations. Remarkably, even when evaluated under controlled and intervened latent rollouts, the diffusion policy consistently maintains these physically aligned behaviors. This empirical evidence suggests that the joint optimization in MBDPO helps structure the latent space in a way that better adheres to underlying physical constraints, making the policy less prone to rendering physically implausible decisions when encountering distribution shifts. Looking forward, as the world model achieves sufficient fidelity, this paradigm opens up a promising avenue toward fully self-improving autonomous agents.
\begin{figure}[ht]
    \centering
    \includegraphics[width=0.99\linewidth]{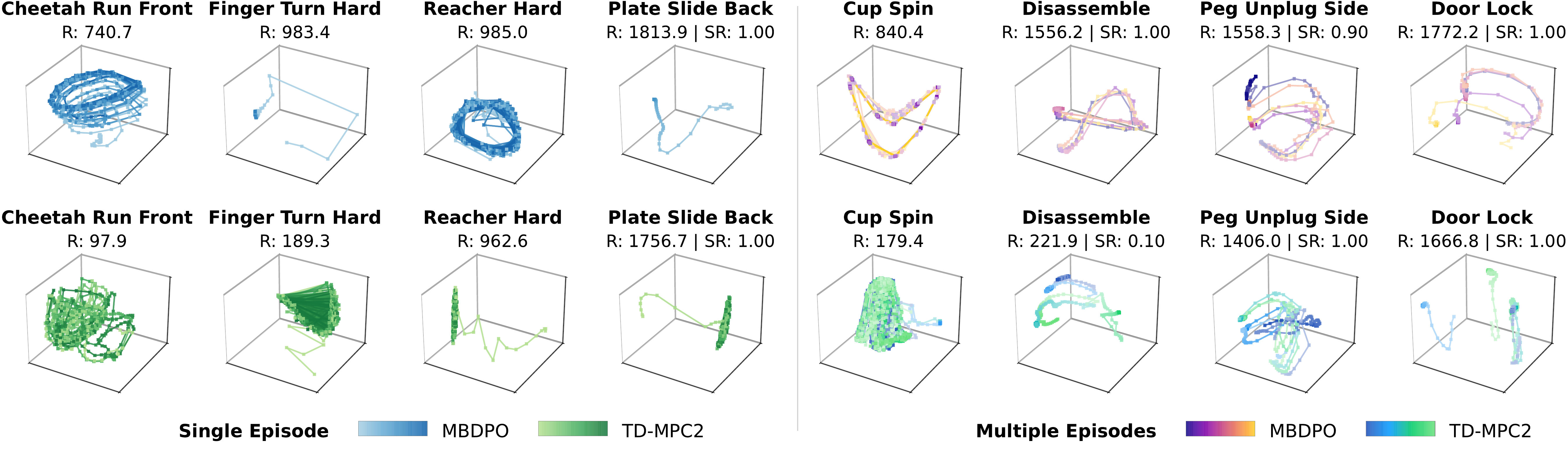}
    \caption{\textbf{Visualizing latent trajectories via a locally linear embedding.} We plot the latent state trajectories across single and multiple episodes in various simulated environments. The colorbar gradient indicates the temporal progression of the trajectories, while R and SR denote the accumulated reward and success rate, respectively.  (a) For cyclical tasks such as ``Cheetah Run Front'', ``Reacher Hard'', and ``Cup Spin'', MBDPO with the diffusion policy generates distinctly closed-loop structures, effectively capturing the physical and repetitive nature of these skills. For other MetaWorld manipulation tasks, MBDPO's latent states delineate smooth and directed state-progressing trajectories that represent a clear causal transition from initial state to target goal. (b) In contrast, under the same controlled latent trajectory visualization, TD-MPC2 exhibits noisier, more disorganized, and less physically aligned behaviors. This indicates that its sampled policy is not physically meaningful, failing to adhere to the underlying causal dynamics necessary for robust decision-making. The controlled latent trajectories for $80$ tasks are visualized in Appendix Figures~\ref{fig:single_epo_diffusion}-\ref{fig:multi_epo_tdmpc}. 
}
    \label{fig:latent_manifold_trajecotry}
\end{figure}

\textbf{2. Rethinking Policy Learning in World Models.} 
Our second higher-level insight compels a fundamental rethinking of how RL should be formulated within world models to truly unlock their future large-scale scaling potential. As mathematically exposed in Problems~\ref{prob_1} and \ref{prob_2}, simply scaling up naive bootstrapping or unconstrained rollout search (such as MPPI, CEM \citep{pinneri2021sample}) is inherently problematic; instead of compounding performance gains, it merely aggregates value overestimation and action drift (Figure~\ref{fig:insight_misaligment}). 
What we truly learned is that the path toward scaling policy learning lies not in designing more complex heuristic search pipelines, but in establishing a simple but principled foundation that inherently links world model with foundational RL paradigms. Crucially, the KL-constrained objective in \eqref{eq:Lagrangian} serves as a direct mathematical bridge. On one hand, its implicit energy function suppresses out-of-distribution value peaks, capturing the essential conservatism of offline RL \citep{levine2020offline}. On the other hand, it bounds the policy shift within a  KL trust region to ensure policy improvement, echoing the core principle of TRPO and PPO \citep{schulman2015trust, schulman2017proximal}. Ultimately, this approach reveals that these seemingly separate methods actually point back to the same fundamental challenges of RL. 

\textbf{3. Representation Matters.}  Our third insight underscores that the quality of the learned latent representation serves as the bedrock for policy optimization, a dependency that manifests through three critical phenomena.  First, a robust, generalist representation unlocks exceptional transferability and data efficiency. As demonstrated in our offline-to-online (O2O) transfer experiments, pre-training MBDPO on 70 diverse tasks constructs an adaptable latent space; this allows the generalist agent to rapidly fine-tune and become an expert on 10 completely unseen tasks with only $40k$ interactions per task, significantly outperforming training from scratch (Figure~\ref{fig:o2o_specific_task}, right). Second, conversely, a corrupted or misaligned representation directly and severely cripples policy optimization. As isolated in the suboptimal performance of ``Hopper Hop'', a severe distributional deviation in the latent embedding (Figure~\ref{fig:embeddings}) imposes a fundamental bottleneck that completely impedes the diffusion policy from searching or sampling effectively. Third, during large-scale pre-training, we observe that the controlled latent trajectories evolve to align with physical intuition (see Figure~\ref{fig:latent_manifold_trajecotry}). This empirical alignment proves that scaling the representation is not merely about preserving training statistics, but about capturing the underlying structure of the world. Ultimately, these findings reveal that policy versatility across offline, online, and O2O paradigms is intrinsically a representation problem, where a physically meaningful latent space is the premier prerequisite for scalable policy learning.

\section{Conclusion} 
In this work, we addressed a critical yet underexplored bottleneck in model-based RL: the structural misalignment between policy search and value learning. To overcome this limitation, we introduced Model-Based Diffusion Policy Optimization (MBDPO). In this framework, the diffusion policy intrinsically links the world model and policy optimization, unlocking the potential of world models for policy learning. By leveraging an implicit energy function to refine the score field, MBDPO anchors policy updates within a trust region, effectively mitigating training inconsistencies and compounding errors. Our extensive evaluations across multi-task offline pretraining, online learning, and offline-to-online fine-tuning demonstrate that MBDPO consistently outperforms leading baselines. Notably, MBDPO exhibits strong scalability, yielding monotonic performance gains as model capacity increases. Beyond quantitative improvements, visualizations of the controlled latent trajectories generated by MBDPO align naturally with physical intuition. Ultimately, MBDPO establishes an effective new paradigm for policy learning within world models.

\textbf{Limitations and Future Directions.} 
While MBDPO demonstrates strong scaling properties and performance, several avenues remain for future exploration.  First, due to bounded computational resources, we have not yet scaled our framework to extremely large parameter sizes or massive foundation-model scales; investigating its full scaling limits remains a direct next step.  Second, our evaluations are currently focused on simulated environments. Moving forward, deploying MBDPO on real-world robots will be critical to testing its robust transferability and practical viability.  Finally, an exciting prospective direction is to integrate MBDPO with even more powerful, pretrained representations (e.g., visual or multimodal foundation models as encoders). Leveraging such richer visual priors could improve the latent representation and unlock truly large-scale RL within world models.

\bibliographystyle{unsrt}
\bibliography{ref}

\clearpage
\appendix

\clearpage
\section*{Notation}
\begin{center}
\begin{tabular}{ll} 
\toprule
\textbf{Notation} & \textbf{Meaning} \\
\midrule
$a$ & action \\
$e$ & learnable task embedding \\
$s$ & state \\
% $\mathbf{s}$ & score function \\
$t$ & time step \\
$r$ & reward function \\
$z$ & latent state \\
$\mathcal{A}$ & space of action  \\
$E$ & implicit energy function \\
$\mathcal{E}$ & encoder of world model \\
$\mathcal{E}_{env}$ & space of learnable task embedding \\
$\mathcal{F}$ & latent dynamics of world model \\
$G$ &  cumulative return over an $H$-step action rollout \\ 
$\tilde{G}$ & energy-regularized cumulative return over an $H$-step action rollout\\ 
$J$ & standard discounted returns \\
$\mathcal{M}$ & Markov decision process \\
$\mathcal{S}$ & space of state \\
$\mathcal{T}$ & transition dynamics \\
$Q$ & $Q$ value function \\
$\hat{Q}$ & value function learned from buffer dataset \\
% $t$ & time step \\
$\mathcal{R}$ & latent reward function \\
$\tilde{\mathcal{R}}$ & energy-regularized latent reward function \\
$T^\Box$ & Bellman update \\
$\mathcal{Z}$ & space of latent state \\
$\alpha, \bar{\alpha}$ & scalar factor of noise injection \\
$\beta$ & behavioral policy distribution \\ 
$\kappa$ & temperature factor \\ 
$\pi$ & policy \\
$\eta$ & energy regularized factor \\
$\rho$ & initial distribution \\
$\sigma^2(\tau)$ & variance induced by the reverse diffusion transition kernel \\
$\pi_{\phi}$ & policy generated by score function  \\
$\phi$  & denoted symbol of score function \\
$\gamma$ & discount factor \\
$\mathbb{1}_{\{\cdot\}}$ & indicator function \\
\bottomrule
\end{tabular}
\end{center}

\clearpage
\section*{A Reader's Guide to the Appendix}
Appendices can often feel like a dense maze of mathematical proofs and exhaustive plots. To respect your time and make navigating this supplementary material as seamless as possible, we have organized the content into distinct, self-contained modules. Depending on your background and primary interests, here is a quick roadmap to help you find exactly what you are looking for:

\begin{itemize}
    \item \textbf{For visual intuitions and empirical rigor:}
    If you want to see how MBDPO behaves across different environments, head to Section~\ref{append:additional_results_online} for exhaustive per-task learning curves and action drift metrics. To visualize what the model actually learns, Section~\ref{append:all_latent_traj} provides 3D manifolds showing how our method forms clean, closed-loop latent trajectories compared to baselines. Exhaustive compute profiling and timing evaluations are also provided to ensure full transparency in Section~\ref{append:time_eval}. 

    \item \textbf{For the theoretical foundation:}
    If you are interested in the mathematical mechanisms, jump straight to Section~\ref{append:theoretical_analysis}. We break down the intrinsic bottlenecks of existing model-based policy search paradigms (Section~\ref{append:bottleneck_analysis}) and provide step-by-step proofs for the exact score function derivations and Theorem~\ref{theorem:score_matching_policy} (Section~\ref{section:score_function_proof}).

    \item \textbf{For reproduction:}
    If you are looking to implement MBDPO or understand its computational footprint, Section~\ref{append:practical_implement} provides the complete pseudocode (Algorithm~\ref{alg:MBDPO}) and hyperparameter settings alongside practical details for online, offline, and fine-tuning variants. 
\end{itemize}

Please feel free to bypass the sections outside your immediate scope and jump directly to the modules most relevant to your expertise.

\clearpage
\section{More Experimental Results}

\subsection{Training Environments}
Our evaluation covers 121 diverse control challenges from four standard benchmarks: DMControl, MetaWorld, ManiSkill2, and MyoSuite \citep{tassa2018deepmind, yu2020meta, gu2023maniskill2, caggiano2022myosuite}. 

\begin{figure}[ht]
    \centering
    \includegraphics[width=0.99\linewidth]{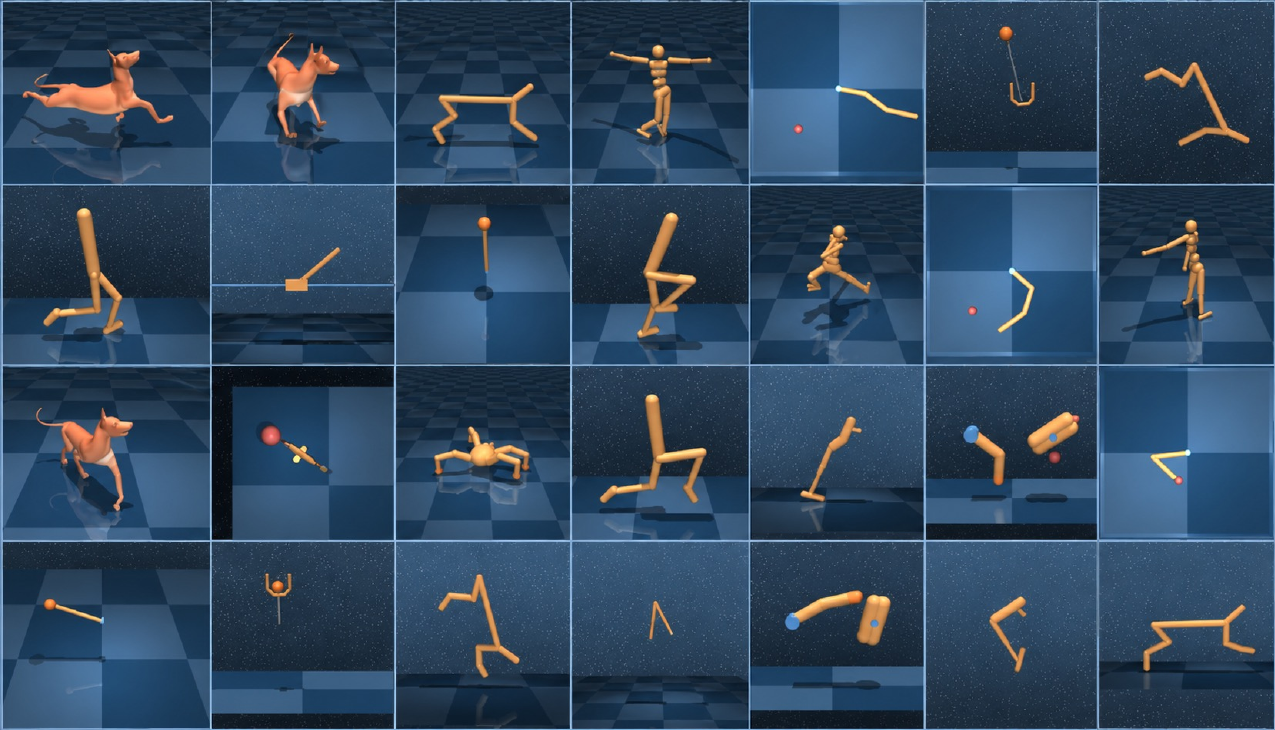}
    \caption{\textbf{Demonstration of DMControl tasks.}}
    \label{fig:demo_dmcontrol}
\end{figure}

\begin{figure}[ht]
    \centering
    \includegraphics[width=0.99\linewidth]{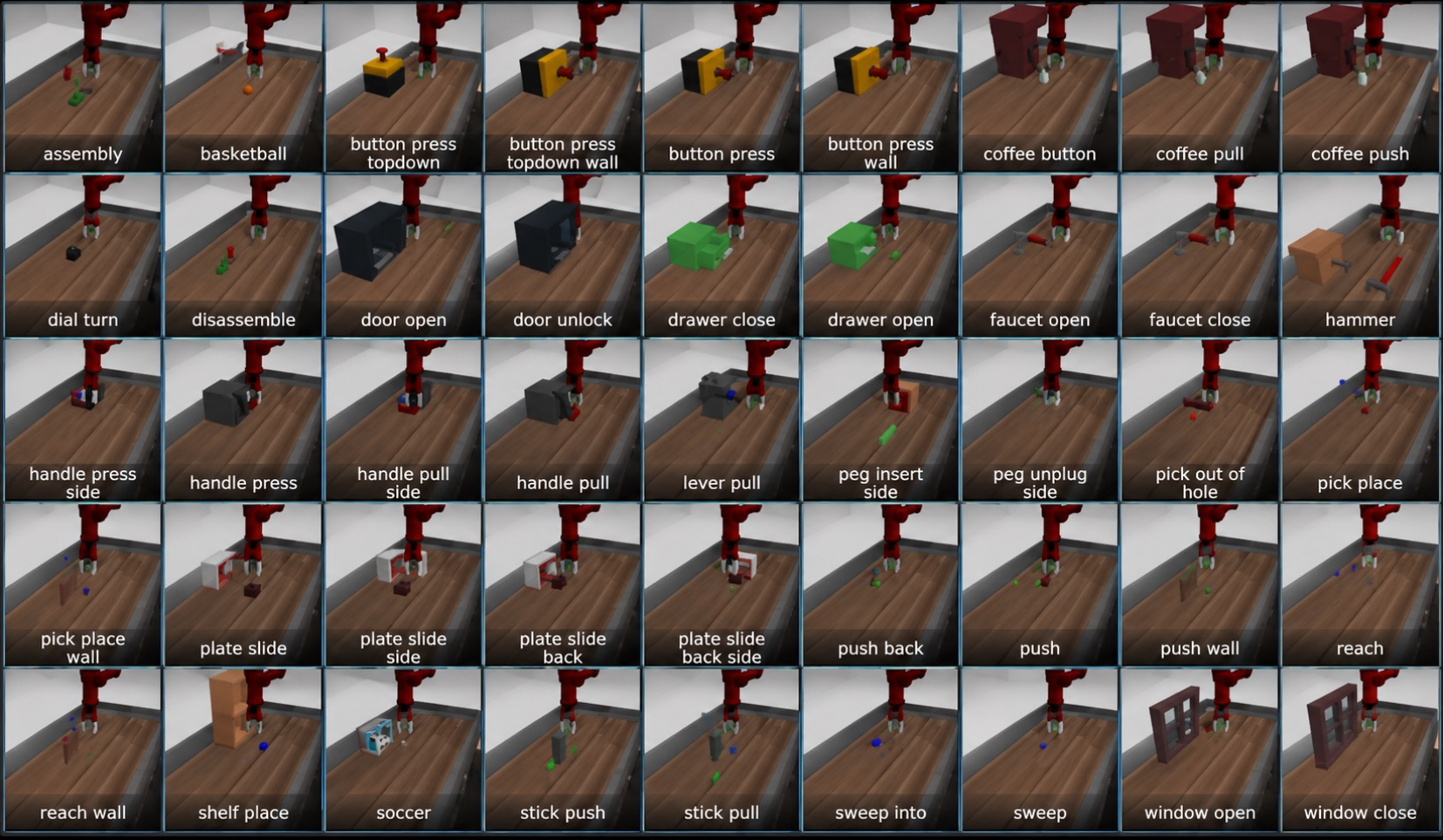}
    \caption{\textbf{Demonstration of MetaWorld tasks.}}
    \label{fig:demo_metaworld}
\end{figure}

\begin{figure}[ht]
    \centering
    \includegraphics[width=0.99\linewidth]{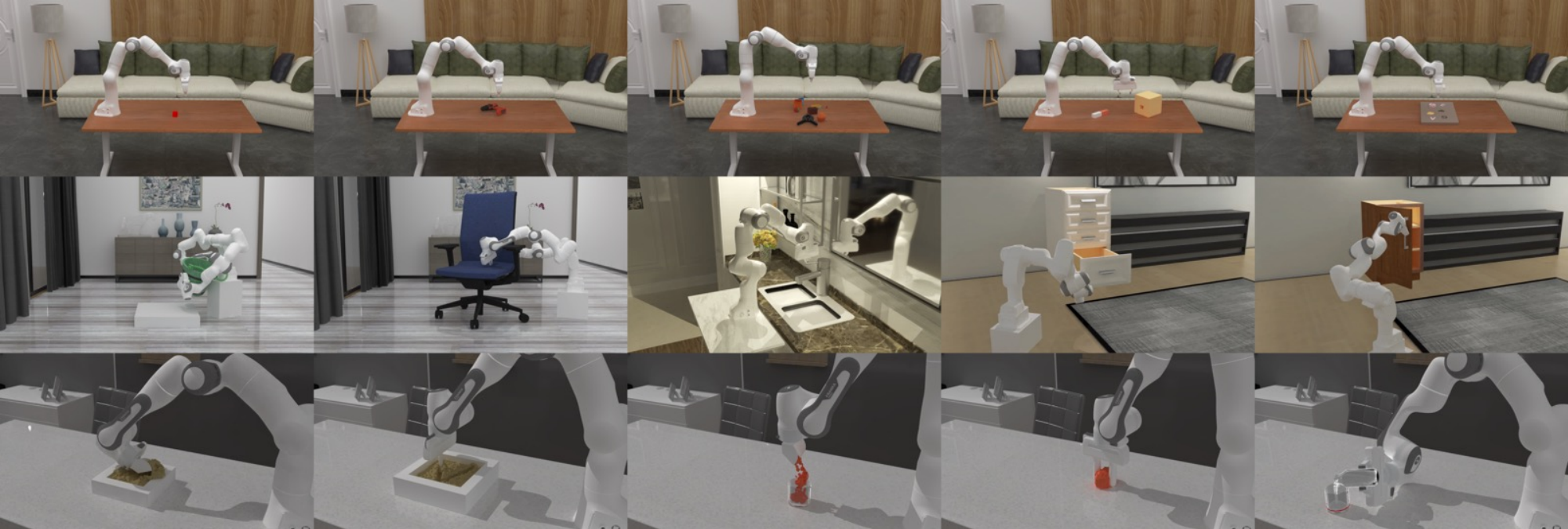}
    \caption{\textbf{Demonstration of ManiSkill2 tasks.}}
    \label{fig:demo_maniskill2}
\end{figure}

\begin{figure}[ht]
    \centering
    \includegraphics[width=0.99\linewidth]{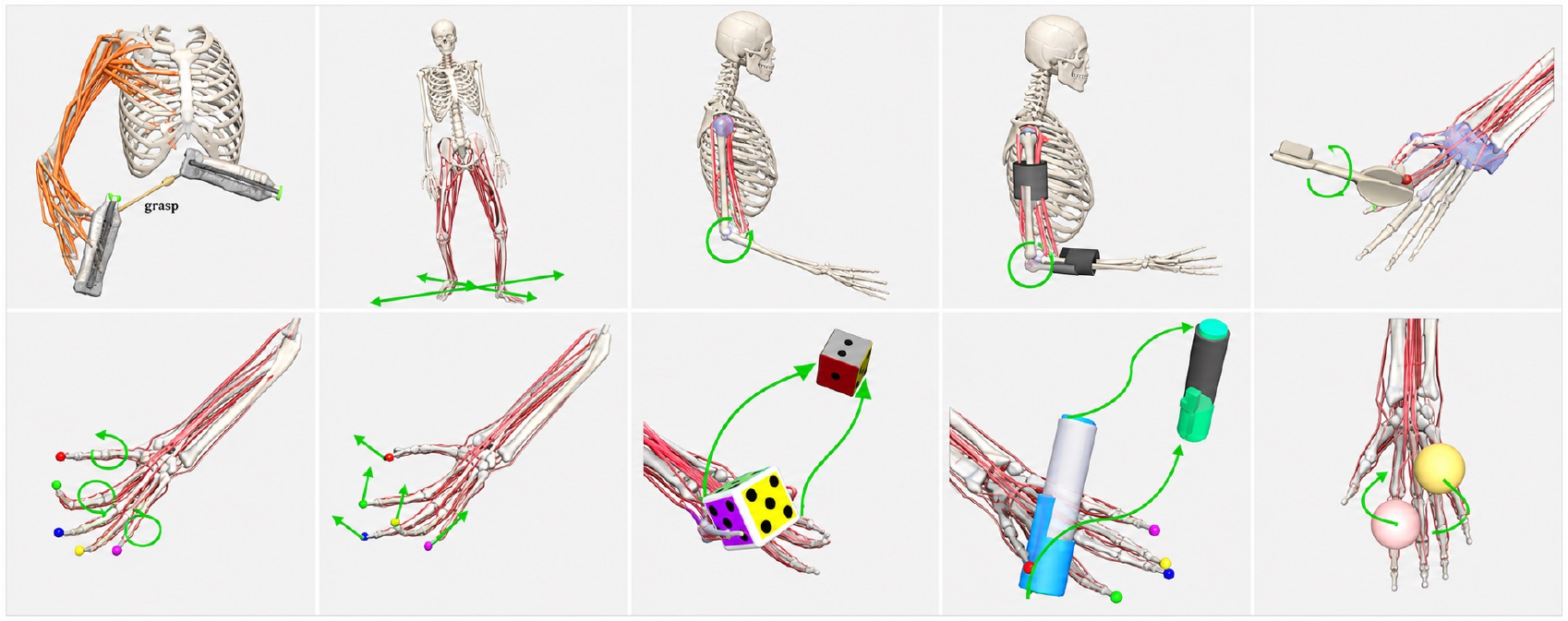}
    \caption{\textbf{Demonstration of MyoSuite tasks.}}
    \label{fig:demo_myosuite}
\end{figure}

\clearpage
\subsection{Additional Results of Online Training from Scratch} \label{append:additional_results_online}

We evaluate MBDPO across 121 tasks from 4 major benchmarks (Figures~\ref{fig:DMControl}-\ref{fig:myosuite}), with each task averaged over 3 random seeds. The aggregate results in Figure~\ref{fig:overview} show that MBDPO consistently achieves state-of-the-art performance across all environments. Compared to both model-free and model-based baselines, our framework demonstrates superior final rewards, higher sample efficiency, and greater training stability of policy optimization.  

\begin{figure}[ht]
    \centering
    \includegraphics[width=0.99\linewidth]{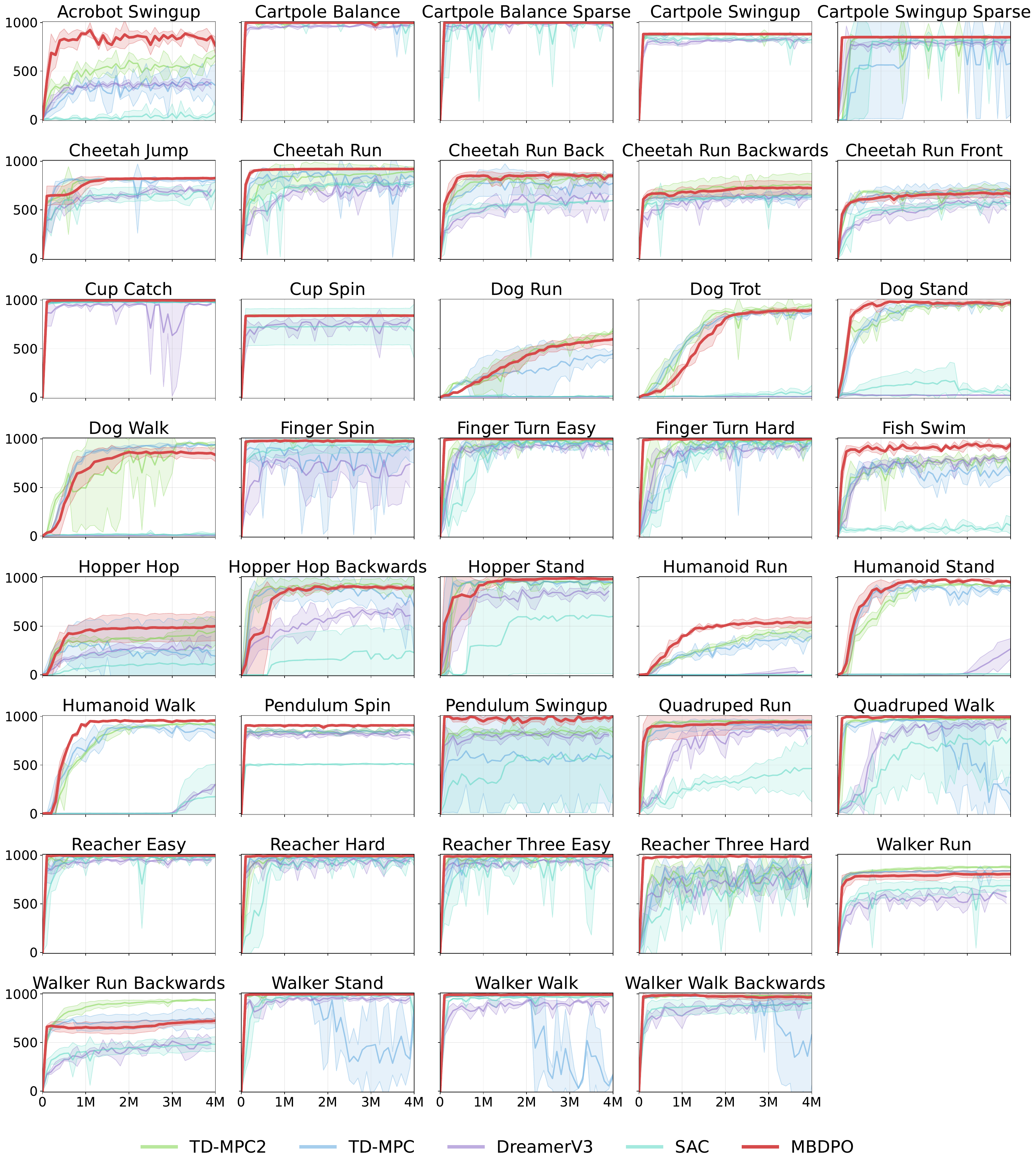}
    \caption{\textbf{Single-task DMControl results.} Episode return as a function of environment steps. The first $4$M environment steps are shown for each task.}
    \label{fig:DMControl}
\end{figure}

\begin{figure}[H]
    \centering
    \includegraphics[width=0.99\linewidth]{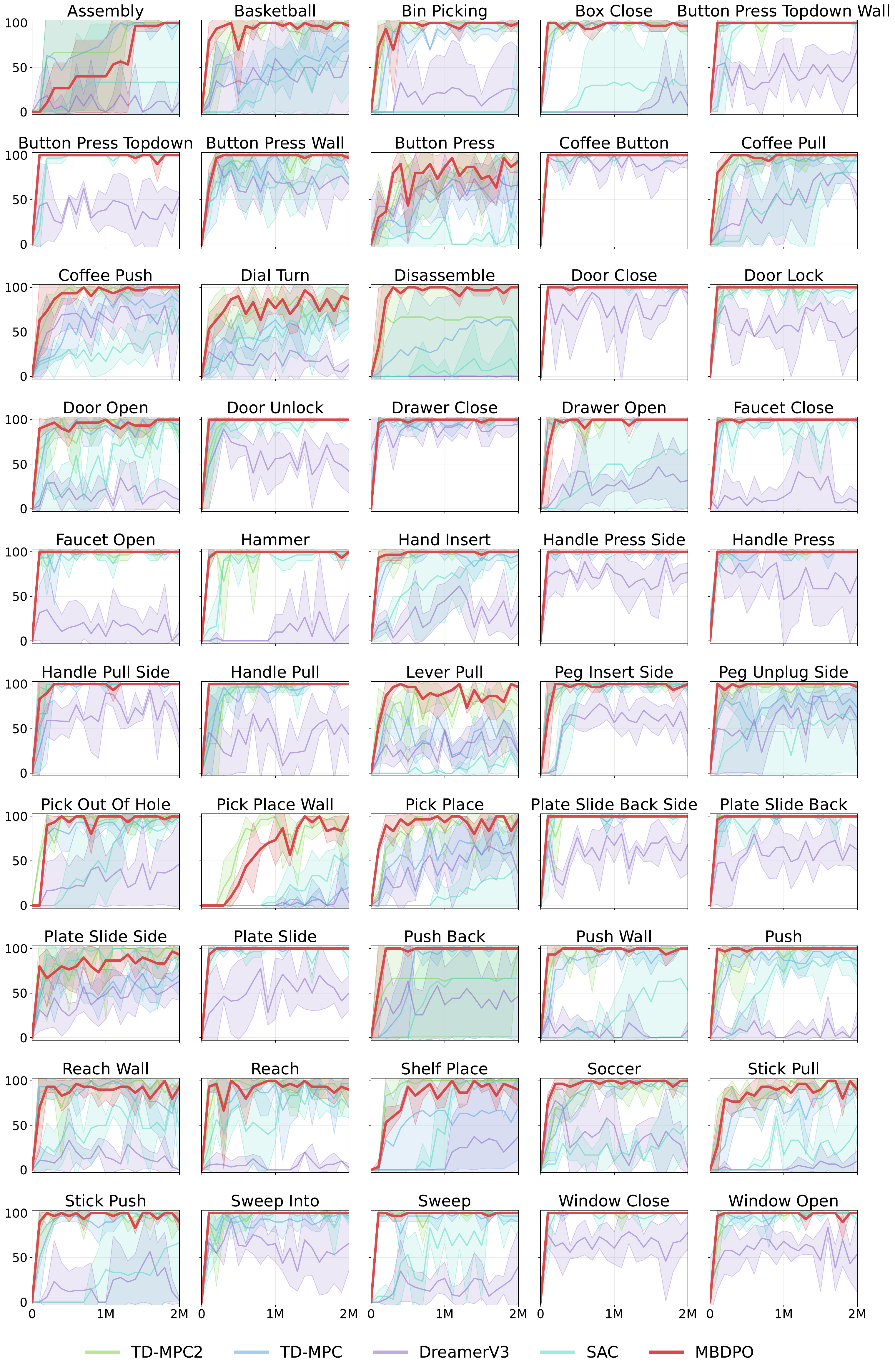}
    \caption{\textbf{Single-task MetaWorld results.} Success rate (\%) as a function of environment steps. MBDPO achieves the best averaged performance over all MetaWorld tasks, while outperforming other methods on hard tasks such as \textit{Pick Place Wall} and \textit{Shelf Place}. DreamerV3 often fails to converge.}
    \label{fig:Meta_world}
\end{figure}

\begin{figure}[H]
    \centering
    \includegraphics[width=0.99\linewidth]{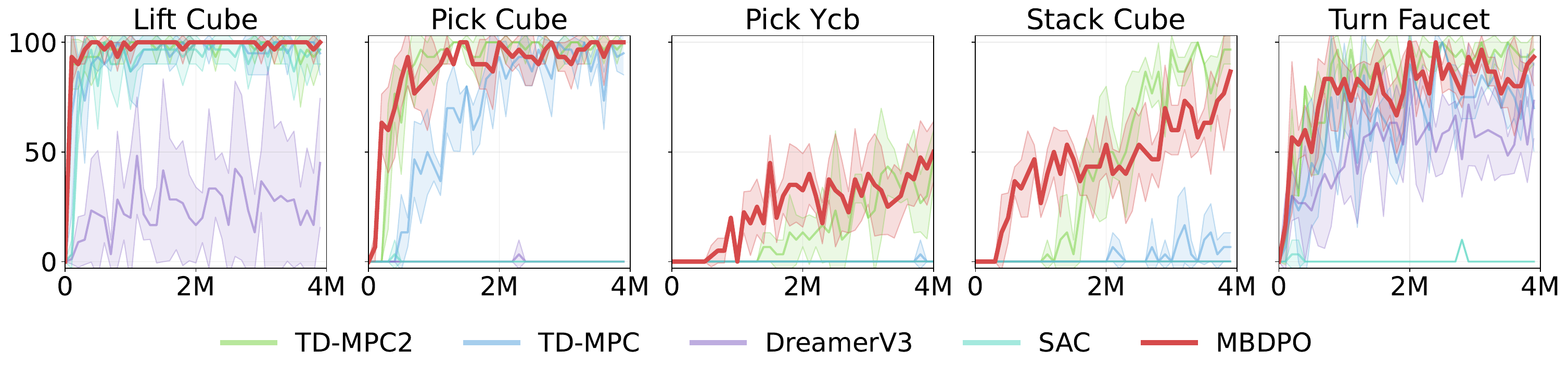}
    \caption{\textbf{Single-task ManiSkill2 results.}  Success rate (\%) as a function of environment steps on 5 object manipulation tasks from ManiSkill2. Pick YCB is the hardest task and involves manipulation of all 74 objects from the YCB \citep{calli2015ycb}. We report $4$M environment steps for each task. MBDPO achieves a success rate above $75\%$  on the Pick YCB task, whereas other methods fail to learn within the given budget.}
    \label{fig:mani_skill2}
\end{figure}

\begin{figure}[H]
    \centering
    \includegraphics[width=0.99\linewidth]{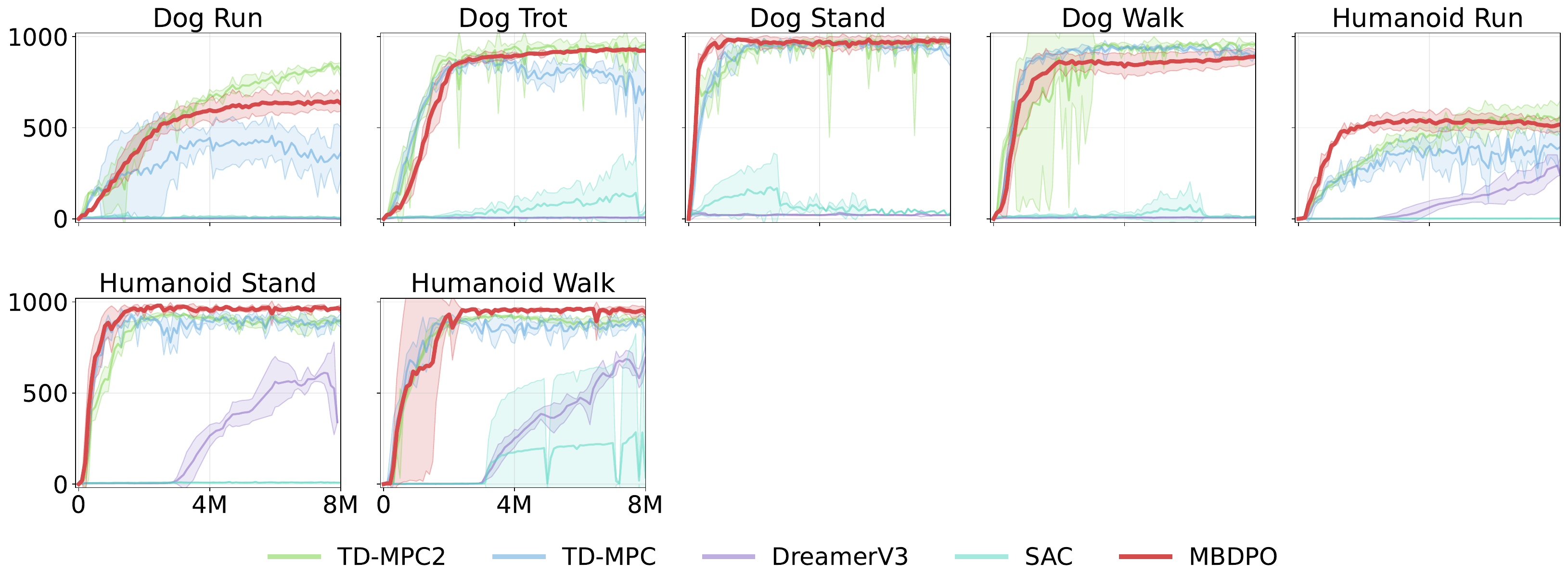}
    \caption{\textbf{High-dimensional locomotion results.} Episode return as a function of environment steps on all 7 “Locomotion” benchmark tasks. }
    \label{fig:high_dim}
\end{figure}

\begin{figure}[H]
    \centering
    \includegraphics[width=0.99\linewidth]{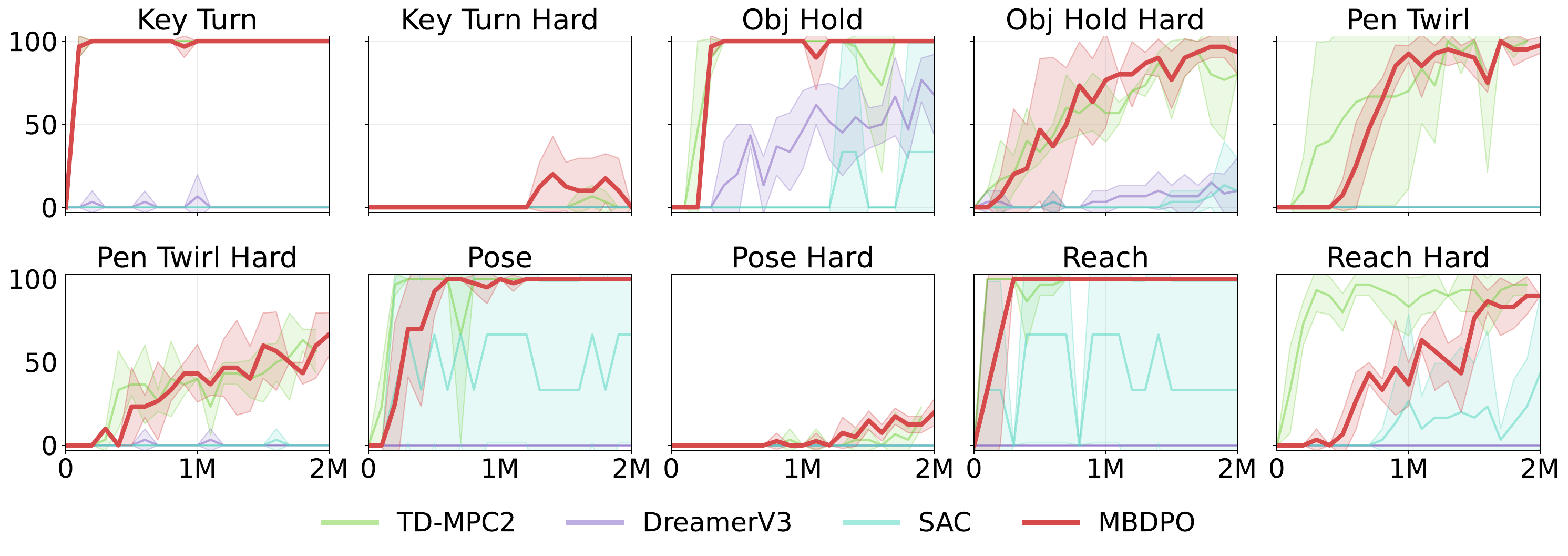}
    \caption{\textbf{Single-task MyoSuite results.} Success rate (\%) as a function of environment steps. This task domain includes high-dimensional contact-rich musculoskeletal motor control with a physiologically accurate robot hand. Goals are randomized in tasks designated as “Hard”. MBDPO achieves comparable or better performance than existing methods on all tasks from this benchmark.}
    \label{fig:myosuite}
\end{figure}

\clearpage
\subsection{Additional Results of Action Drift During Training}~\label{append:additional_result}

\begin{figure}[ht]
    \centering
    \includegraphics[width=0.99\linewidth]{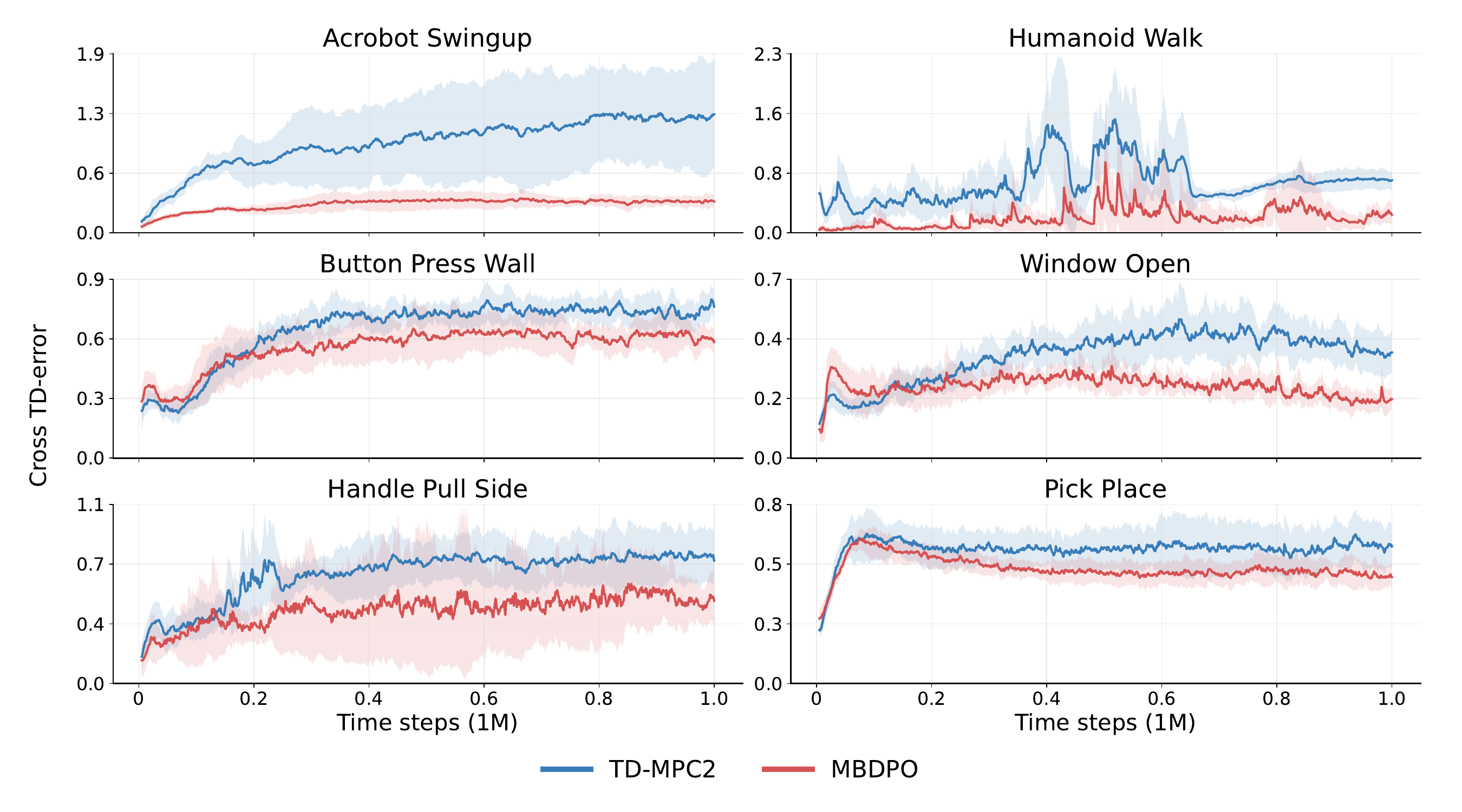}
    \caption{\textbf{Per-task cross TD-error during training.}  EMA-smoothed cross TD-errors are reported for representative online tasks. The cross TD-error measures the temporal-difference error incurred when evaluating trajectories under a policy different from the one used for value-function training, thereby reflecting the distribution mismatch between policy improvement and value learning. Across most tasks, MBDPO exhibits lower cross TD-error than TD-MPC2, suggesting that diffusion policy optimization reduces out-of-distribution value queries and leads to more stable value iteration during training.}
    \label{fig:all_TD_error}
\end{figure}

\begin{figure}[ht]
    \centering
    \includegraphics[width=0.99\linewidth]{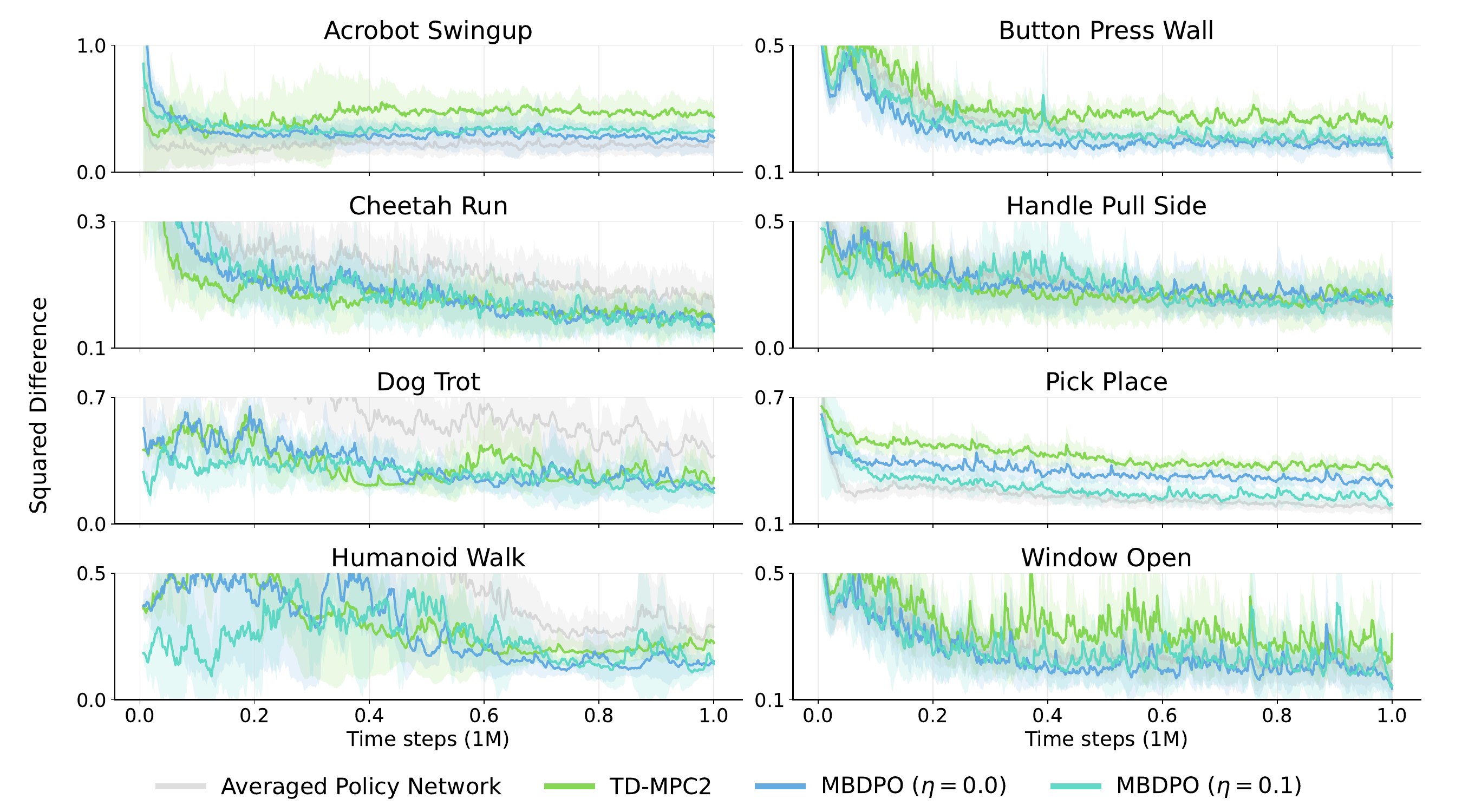}
    \caption{\textbf{Per-task action drift during training.}
    EMA-smoothed action differences are reported for each of the $8$ online tasks.
    The Averaged Policy Network (Grey) curve reports the mean action drift of the policy networks from TD-MPC2 and the two MBDPO variants.
    Across these tasks, MBDPO exhibits smaller drift than TD-MPC2, indicating improved temporal stability under diffusion policy optimization.
}
    \label{fig:all_method_drift}
\end{figure}

\clearpage

\subsection{Latent Task Embedding}
Figure~\ref{fig:embeddings} below shows the task embedding $\mathcal{E}_{env}$ for $70$-task pretraining with $10$ unseen tasks for fine-tuning.  Each point corresponds to one task embedding projected into a two-dimensional space using t-SNE for visualization.  The red circles denote MetaWorld tasks, while the blue squares denote DMControl tasks.  As shown in the figure, the learned task embeddings exhibit a clear domain-level separation: MetaWorld manipulation tasks form a compact cluster in the upper region, whereas DMControl locomotion and control tasks are mainly grouped in the lower region.  This separation indicates that $\mathcal{E}_{env}$ captures task-specific environment information and organizes tasks according to their underlying dynamics, observation distributions, and action semantics.

Most tasks within the same benchmark are embedded close to semantically or dynamically related tasks.  For example, MetaWorld tasks involving similar manipulation primitives, such as pushing, pulling, opening, closing, and object placement, tend to be located near each other.  Similarly, DMControl tasks with related morphology or control structure, such as Walker, Cheetah, Cartpole, and Reacher variants, also form local clusters.  This suggests that the learned task embedding is not merely distinguishing task identities, but also reflects meaningful similarities across tasks. Notably, the task ``Hopper Hop'' is highlighted by the red circle.  Although it belongs to the DMControl suite, its embedding is relatively isolated from the main DMControl cluster. This indicates that ``Hopper Hop'' has a substantially different data distribution from the other pretraining tasks.  Such an outlying embedding may lead to less effective sharing of representations and policies during multi-task pretraining, thereby making policy optimization more challenging for this task. 
Overall, the visualization supports our hypothesis that the learned task embedding $\mathcal{E}_{env}$ provides a structured representation of task distributions, while also revealing tasks whose distributions deviate significantly from the majority of the pretraining set.
\begin{figure}[ht]
    \centering
    \includegraphics[width=0.5\linewidth]{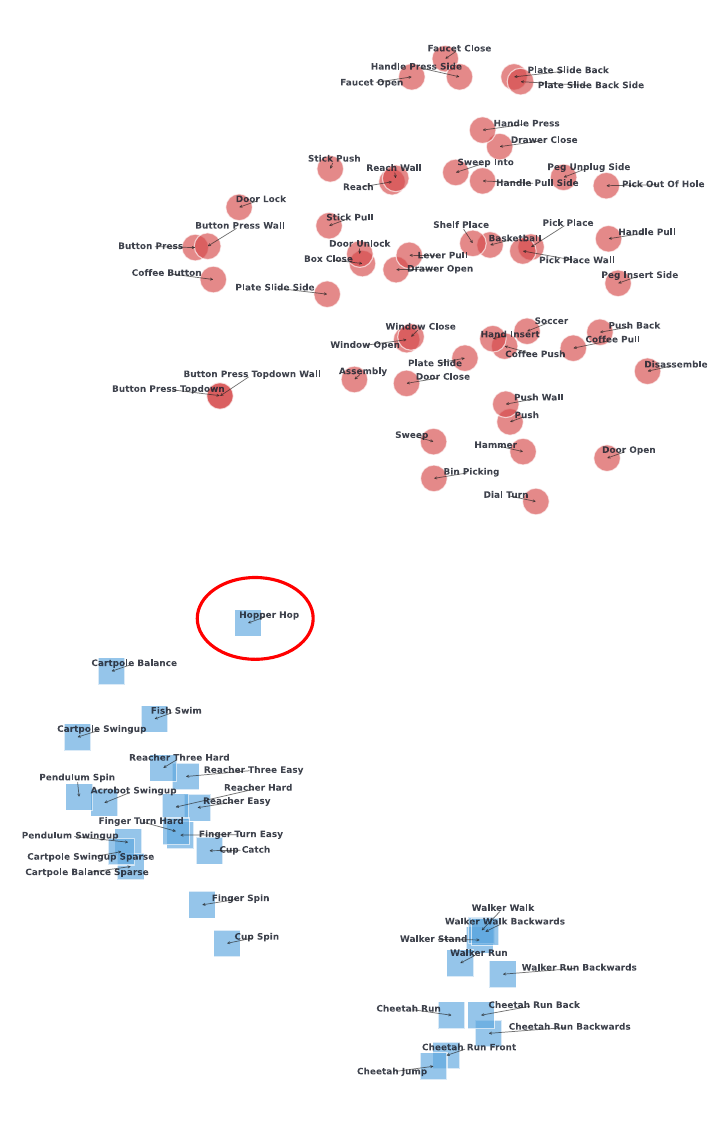}
    \caption{\textbf{T-SNE visualization \citep{van2008visualizing} of task embeddings learned by MBDPO trained on $80$ tasks from DMControl and MetaWorld.} Detailed labels are shown for clarity. According to our hypothesis, the task embedding of Hopper Hop can be attributed to a severe deviation in the representation’s data distribution (see the red circle), which subsequently impedes effective policy optimization.
    }
    \label{fig:embeddings}
\end{figure}

% \clearpage
\subsection{Latent Trajectories Comparison} 
\label{append:all_latent_traj}
We provide additional visualizations of latent state trajectories for all $80$ tasks under both MBDPO and TD-MPC2.  Specifically, Figures~\ref{fig:single_epo_diffusion} and~\ref{fig:single_epo_tdmpc} compare single-episode trajectories, while Figures~\ref{fig:multi_epo_diffusion} and~\ref{fig:multi_epo_tdmpc} compare multi-episode trajectories.  Across both settings, MBDPO produces more structured and temporally coherent latent manifolds, including closed-loop patterns for cyclic control tasks and smooth goal-directed transitions for manipulation tasks.  In contrast, TD-MPC2 often exhibits noisier and less organized trajectories, suggesting weaker alignment between the sampled policy and the underlying task dynamics.

\begin{figure}[ht]
    \centering
    \includegraphics[width=0.9\linewidth]{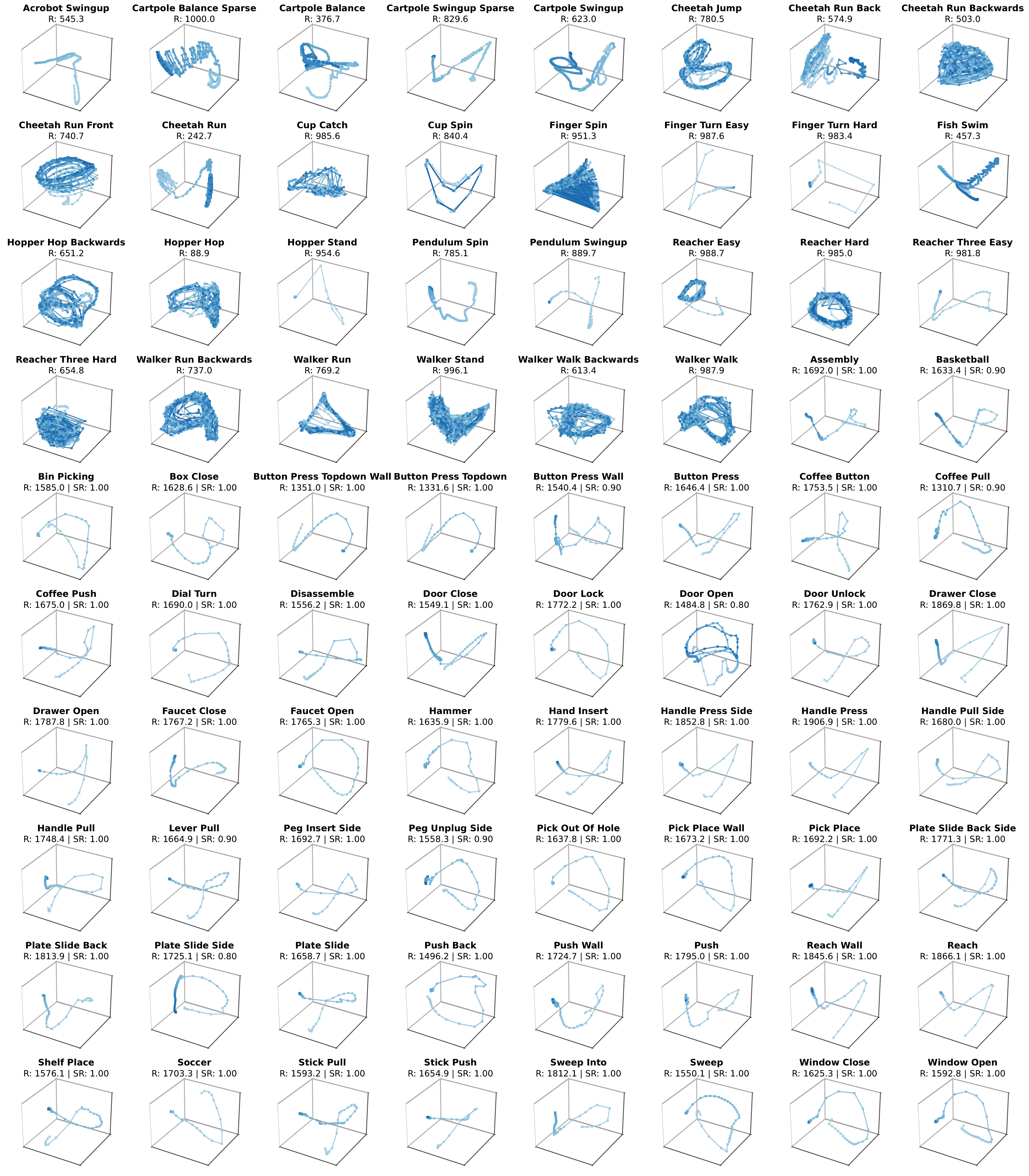}
    \caption{\textbf{Single-episode latent trajectory visualization for MBDPO.}
    We visualize the latent state trajectories produced by MBDPO with the diffusion policy over one episode for all $80$ tasks using a locally linear embedding. The color gradient indicates temporal progression along the trajectory. For cyclical control tasks, the latent trajectories often form closed-loop structures that reflect the repetitive physical dynamics of the skills (DMControl tasks). For goal-directed manipulation tasks (MetaWorld tasks), the trajectories typically evolve smoothly from the initial state toward the target state, suggesting that the diffusion policy induces structured and physically meaningful latent transitions.}
    \label{fig:single_epo_diffusion}
\end{figure}

\begin{figure}[ht]
    \centering
    \includegraphics[width=0.9\linewidth]{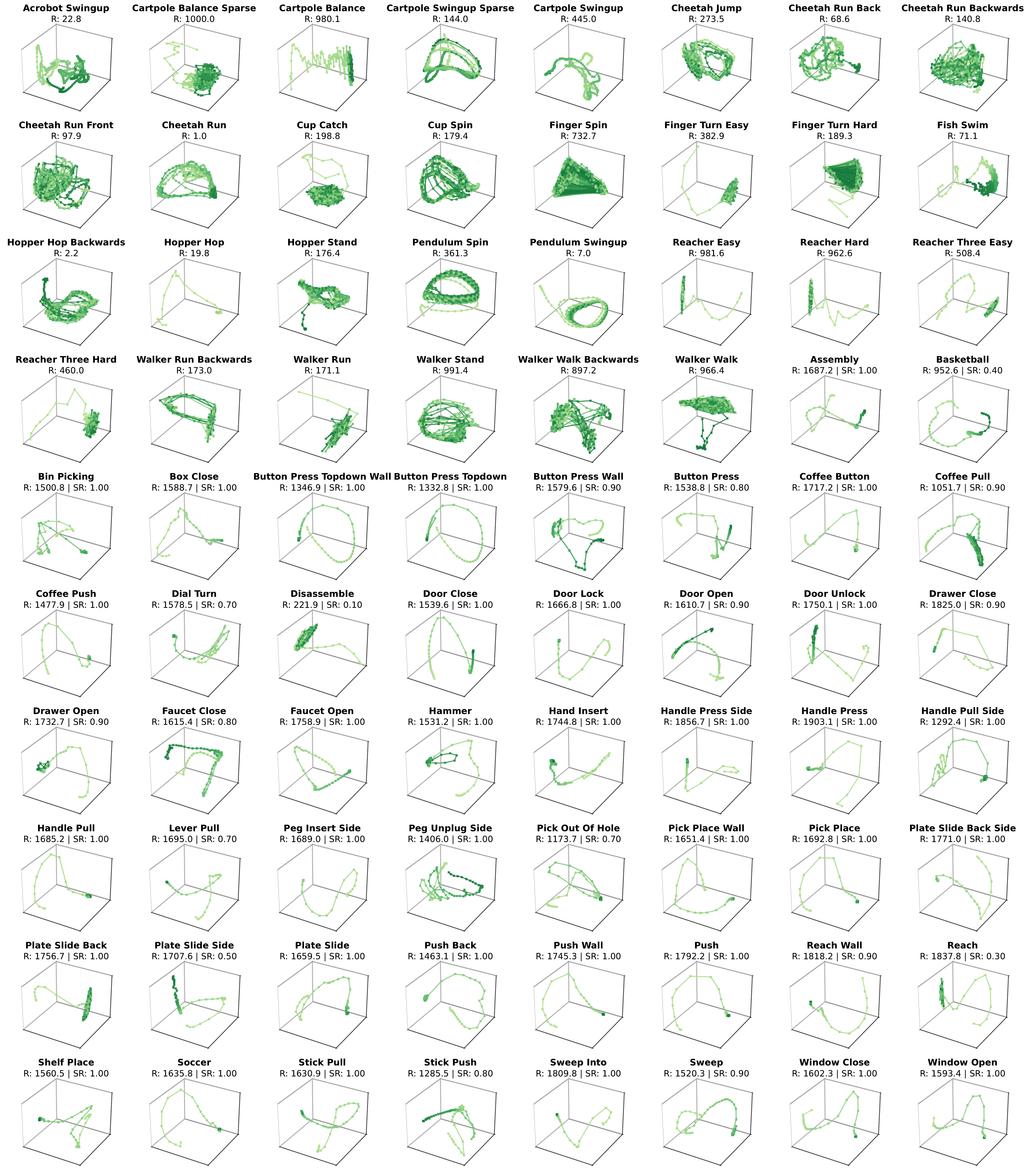}
    \caption{\textbf{Single-episode latent trajectory visualization for TD-MPC2.}
    We visualize the latent state trajectories produced by TD-MPC2 over one episode for all $80$ tasks using the same locally linear embedding protocol.
    The color gradient indicates temporal progression along the trajectory.
    Compared with MBDPO, TD-MPC2 produces trajectories that are often more scattered, noisy, and less aligned with the physical or goal-directed structure of the tasks.
    This suggests that its sampled policy may induce less coherent latent dynamics during decision-making.
}
    \label{fig:single_epo_tdmpc}
\end{figure}

\begin{figure}[ht]
    \centering
    \includegraphics[width=0.99\linewidth]{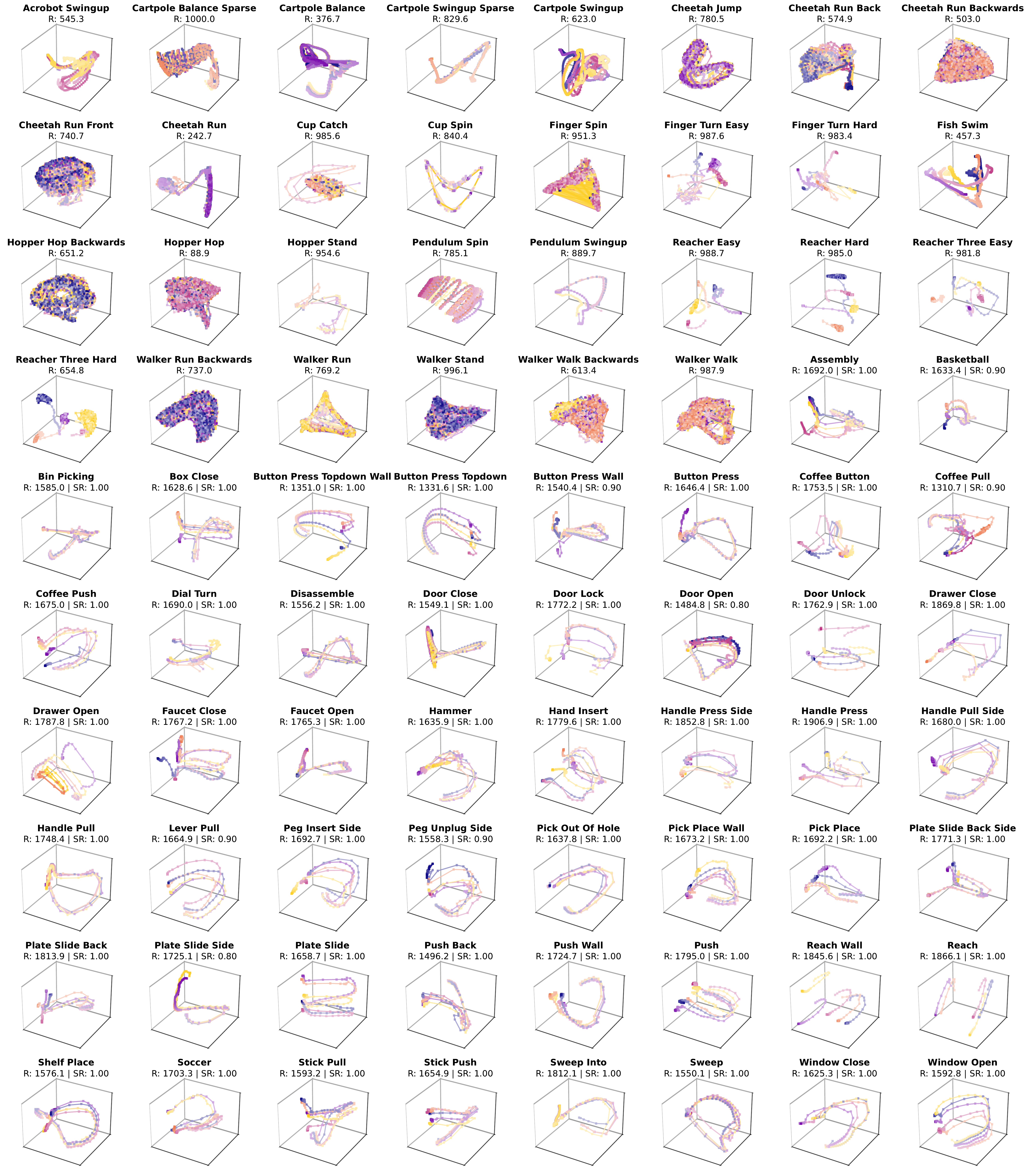}
    \caption{\textbf{Multi-episode latent trajectory visualization for MBDPO.}  We visualize latent state trajectories produced by MBDPO with the diffusion policy over multiple episodes for all $80$ tasks using a locally linear embedding. The color gradient denotes temporal progression within each episode. Across repeated rollouts, MBDPO maintains consistent and structured trajectory manifolds, with closed-loop patterns for cyclic tasks and smooth directed transitions for goal-conditioned manipulation tasks. These results indicate that the diffusion policy captures stable task-specific latent dynamics across episodes.
}
    \label{fig:multi_epo_diffusion}
\end{figure}

\begin{figure}[ht]
    \centering
    \includegraphics[width=0.99\linewidth]{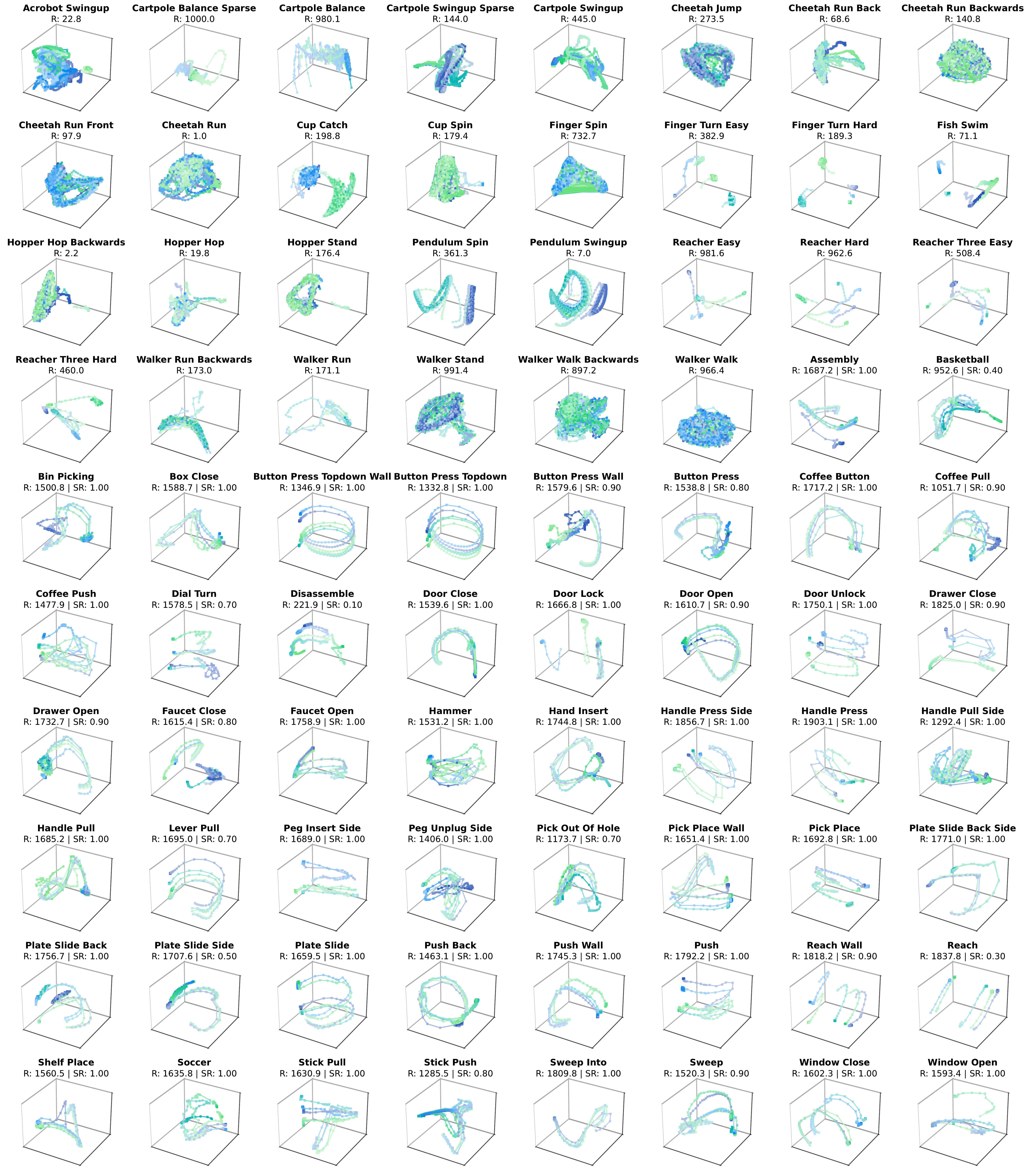}
    \caption{\textbf{Multi-episode latent trajectory visualization for TD-MPC2.}   We visualize latent state trajectories produced by TD-MPC2 over multiple episodes for all $80$ tasks using the same locally linear embedding protocol.   The color gradient denotes temporal progression within each episode.  Compared with MBDPO, TD-MPC2 exhibits less consistent trajectory geometry across rollouts, with noisier, more dispersed, and more irregular latent paths. This indicates that its sampled policy produces less stable latent dynamics and weaker alignment with the underlying causal structure of the environments.
}
    \label{fig:multi_epo_tdmpc}
\end{figure}

\clearpage
\subsection{Comprehensive Training and Evaluation Time}~\label{append:time_eval}

\begin{table}[ht]
\centering
\scriptsize
\setlength{\tabcolsep}{2pt}
\renewcommand{\arraystretch}{0.86}
\caption{\textbf{Profiling results of online-task training time per environment step across domains} (milliseconds per env step). Values are reported as mean $\pm$ std. Each experiment runs a single task on one NVIDIA A800-SXM4-80GB GPU.}
\label{tab:profiling_training_results}
\begin{adjustbox}{max width=\linewidth, max totalheight=0.95\textheight, center}
\begin{tabular}{@{}lcc@{\hspace{0.8em}}lcc@{}}
\toprule
\textbf{Task} & \textbf{TD-MPC2} & \textbf{MBDPO} & \textbf{Task} & \textbf{TD-MPC2} & \textbf{MBDPO} \\
\midrule
\multicolumn{6}{c}{\textbf{MetaWorld}} \\
\midrule
\texttt{mw-assembly} & \prof{1.35}{0.012} & \prof{1.35}{0.008} & \texttt{mw-basketball} & \prof{1.07}{0.019} & \prof{1.09}{0.022} \\
\texttt{mw-button-press-topdown} & \prof{0.87}{0.007} & \prof{0.89}{0.006} & \texttt{mw-button-press-topdown-wall} & \prof{0.90}{0.016} & \prof{0.89}{0.006} \\
\texttt{mw-button-press} & \prof{0.89}{0.017} & \prof{0.92}{0.028} & \texttt{mw-button-press-wall} & \prof{0.94}{0.054} & \prof{0.91}{0.013} \\
\texttt{mw-coffee-button} & \prof{0.99}{0.015} & \prof{1.00}{0.004} & \texttt{mw-coffee-pull} & \prof{1.26}{0.026} & \prof{1.26}{0.007} \\
\texttt{mw-coffee-push} & \prof{1.22}{0.007} & \prof{1.21}{0.006} & \texttt{mw-dial-turn} & \prof{0.94}{0.024} & \prof{0.92}{0.007} \\
\texttt{mw-disassemble} & \prof{1.36}{0.018} & \prof{1.34}{0.020} & \texttt{mw-door-open} & \prof{1.09}{0.013} & \prof{1.08}{0.015} \\
\texttt{mw-door-close} & \prof{1.01}{0.010} & \prof{1.01}{0.016} & \texttt{mw-drawer-close} & \prof{1.09}{0.041} & \prof{1.08}{0.041} \\
\texttt{mw-drawer-open} & \prof{1.02}{0.007} & \prof{1.03}{0.012} & \texttt{mw-faucet-open} & \prof{0.91}{0.022} & \prof{0.90}{0.029} \\
\texttt{mw-faucet-close} & \prof{0.90}{0.005} & \prof{0.89}{0.007} & \texttt{mw-hammer} & \prof{1.33}{0.016} & \prof{1.33}{0.016} \\
\texttt{mw-handle-press-side} & \prof{0.88}{0.035} & \prof{0.87}{0.018} & \texttt{mw-handle-press} & \prof{0.93}{0.038} & \prof{0.92}{0.037} \\
\texttt{mw-handle-pull-side} & \prof{1.00}{0.011} & \prof{1.00}{0.036} & \texttt{mw-handle-pull} & \prof{0.99}{0.009} & \prof{0.98}{0.007} \\
\texttt{mw-lever-pull} & \prof{1.00}{0.014} & \prof{0.99}{0.010} & \texttt{mw-peg-insert-side} & \prof{1.14}{0.005} & \prof{1.14}{0.016} \\
\texttt{mw-peg-unplug-side} & \prof{1.22}{0.079} & \prof{1.20}{0.079} & \texttt{mw-pick-out-of-hole} & \prof{1.08}{0.019} & \prof{1.07}{0.048} \\
\texttt{mw-pick-place} & \prof{1.07}{0.014} & \prof{1.10}{0.032} & \texttt{mw-pick-place-wall} & \prof{1.19}{0.010} & \prof{1.18}{0.006} \\
\texttt{mw-plate-slide} & \prof{0.90}{0.011} & \prof{0.92}{0.013} & \texttt{mw-plate-slide-side} & \prof{0.92}{0.008} & \prof{0.91}{0.013} \\
\texttt{mw-plate-slide-back} & \prof{0.97}{0.011} & \prof{0.94}{0.013} & \texttt{mw-plate-slide-back-side} & \prof{0.96}{0.011} & \prof{0.95}{0.009} \\
\texttt{mw-push-back} & \prof{1.10}{0.008} & \prof{1.10}{0.011} & \texttt{mw-push} & \prof{1.11}{0.010} & \prof{1.11}{0.008} \\
\texttt{mw-push-wall} & \prof{1.17}{0.010} & \prof{1.15}{0.014} & \texttt{mw-reach} & \prof{0.90}{0.011} & \prof{0.89}{0.011} \\
\texttt{mw-reach-wall} & \prof{0.91}{0.005} & \prof{0.92}{0.025} & \texttt{mw-shelf-place} & \prof{1.21}{0.040} & \prof{1.19}{0.038} \\
\texttt{mw-soccer} & \prof{1.10}{0.014} & \prof{1.10}{0.009} & \texttt{mw-stick-push} & \prof{1.35}{0.015} & \prof{1.34}{0.010} \\
\texttt{mw-stick-pull} & \prof{1.43}{0.017} & \prof{1.39}{0.009} & \texttt{mw-sweep-into} & \prof{1.19}{0.030} & \prof{1.20}{0.023} \\
\texttt{mw-sweep} & \prof{1.12}{0.024} & \prof{1.11}{0.004} & \texttt{mw-window-open} & \prof{1.23}{0.018} & \prof{1.22}{0.020} \\
\texttt{mw-window-close} & \prof{1.24}{0.031} & \prof{1.23}{0.008} & \texttt{mw-bin-picking} & \prof{1.12}{0.015} & \prof{1.19}{0.229} \\
\texttt{mw-box-close} & \prof{1.44}{0.092} & \prof{1.41}{0.087} & \texttt{mw-door-lock} & \prof{1.02}{0.067} & \prof{1.00}{0.076} \\
\texttt{mw-door-unlock} & \prof{1.00}{0.042} & \prof{1.00}{0.041} & \texttt{mw-hand-insert} & \prof{1.14}{0.031} & \prof{1.11}{0.026} \\
\midrule
\multicolumn{6}{c}{\textbf{DMControl}} \\
\midrule
\texttt{acrobot-swingup} & \prof{0.40}{0.003} & \prof{0.38}{0.002} & \texttt{cartpole-balance} & \prof{0.51}{0.002} & \prof{0.50}{0.012} \\
\texttt{cartpole-balance-sparse} & \prof{0.40}{0.004} & \prof{0.37}{0.002} & \texttt{cartpole-swingup} & \prof{0.51}{0.013} & \prof{0.51}{0.001} \\
\texttt{cartpole-swingup-sparse} & \prof{0.38}{0.002} & \prof{0.37}{0.001} & \texttt{cheetah-jump} & \prof{0.50}{0.007} & \prof{0.50}{0.009} \\
\texttt{cheetah-run} & \prof{0.40}{0.005} & \prof{0.38}{0.008} & \texttt{cheetah-run-back} & \prof{0.45}{0.003} & \prof{0.45}{0.008} \\
\texttt{cheetah-run-backwards} & \prof{0.40}{0.008} & \prof{0.38}{0.003} & \texttt{cheetah-run-front} & \prof{0.47}{0.007} & \prof{0.44}{0.004} \\
\texttt{cup-catch} & \prof{0.40}{0.011} & \prof{0.39}{0.002} & \texttt{cup-spin} & \prof{0.68}{0.055} & \prof{0.67}{0.050} \\
\texttt{dog-run} & \prof{2.76}{0.147} & \prof{2.70}{0.137} & \texttt{dog-trot} & \prof{2.72}{0.097} & \prof{2.84}{0.118} \\
\texttt{dog-stand} & \prof{2.69}{0.069} & \prof{2.62}{0.090} & \texttt{dog-walk} & \prof{2.85}{0.151} & \prof{2.74}{0.138} \\
\texttt{finger-spin} & \prof{0.37}{0.009} & \prof{0.33}{0.002} & \texttt{finger-turn-easy} & \prof{0.47}{0.005} & \prof{0.45}{0.003} \\
\texttt{finger-turn-hard} & \prof{0.47}{0.005} & \prof{0.46}{0.005} & \texttt{fish-swim} & \prof{0.51}{0.003} & \prof{0.51}{0.003} \\
\texttt{hopper-hop} & \prof{0.45}{0.019} & \prof{0.41}{0.015} & \texttt{hopper-hop-backwards} & \prof{0.53}{0.043} & \prof{0.41}{0.016} \\
\texttt{hopper-stand} & \prof{0.49}{0.014} & \prof{0.46}{0.015} & \texttt{humanoid-run} & \prof{1.03}{0.013} & \prof{1.02}{0.012} \\
\texttt{humanoid-stand} & \prof{1.07}{0.010} & \prof{1.09}{0.009} & \texttt{humanoid-walk} & \prof{1.06}{0.011} & \prof{1.03}{0.012} \\
\texttt{pendulum-spin} & \prof{0.35}{0.011} & \prof{0.33}{0.001} & \texttt{pendulum-swingup} & \prof{0.31}{0.002} & \prof{0.30}{0.001} \\
\texttt{quadruped-run} & \prof{1.11}{0.020} & \prof{1.09}{0.019} & \texttt{quadruped-walk} & \prof{1.09}{0.019} & \prof{1.09}{0.022} \\
\texttt{reacher-easy} & \prof{0.37}{0.003} & \prof{0.34}{0.004} & \texttt{reacher-hard} & \prof{0.36}{0.003} & \prof{0.34}{0.002} \\
\texttt{reacher-three-easy} & \prof{0.38}{0.007} & \prof{0.36}{0.003} & \texttt{reacher-three-hard} & \prof{0.38}{0.006} & \prof{0.35}{0.003} \\
\texttt{walker-run} & \prof{0.61}{0.006} & \prof{0.58}{0.006} & \texttt{walker-run-backwards} & \prof{0.61}{0.005} & \prof{0.59}{0.007} \\
\texttt{walker-stand} & \prof{0.57}{0.010} & \prof{0.55}{0.021} & \texttt{walker-walk} & \prof{0.61}{0.006} & \prof{0.58}{0.003} \\
\texttt{walker-walk-backwards} & \prof{0.60}{0.007} & \prof{0.59}{0.007} &  &  &  \\
\midrule
\multicolumn{6}{c}{\textbf{ManiSkill2}} \\
\midrule
\texttt{lift-cube} & \prof{4.98}{0.301} & \prof{4.85}{0.249} & \texttt{pick-cube} & \prof{4.84}{0.389} & \prof{5.04}{0.251} \\
\texttt{stack-cube} & \prof{5.74}{0.229} & \prof{5.85}{0.351} & \texttt{pick-ycb} & \prof{8.24}{1.921} & \prof{7.99}{1.890} \\
\texttt{turn-faucet} & \prof{8.69}{2.773} & \prof{10.81}{5.652} &  &  &  \\
\midrule
\multicolumn{6}{c}{\textbf{MyoSuite}} \\
\midrule
\texttt{myo-reach} & \prof{1.59}{0.062} & \prof{1.52}{0.029} & \texttt{myo-reach-hard} & \prof{1.59}{0.041} & \prof{1.55}{0.024} \\
\texttt{myo-pose} & \prof{1.49}{0.039} & \prof{1.48}{0.049} & \texttt{myo-pose-hard} & \prof{1.65}{0.046} & \prof{1.59}{0.087} \\
\texttt{myo-obj-hold} & \prof{1.68}{0.058} & \prof{1.61}{0.036} & \texttt{myo-obj-hold-hard} & \prof{1.58}{0.043} & \prof{1.61}{0.040} \\
\texttt{myo-key-turn} & \prof{1.84}{0.081} & \prof{1.70}{0.085} & \texttt{myo-key-turn-hard} & \prof{1.90}{0.097} & \prof{1.86}{0.090} \\
\texttt{myo-pen-twirl} & \prof{1.23}{0.022} & \prof{1.22}{0.039} & \texttt{myo-pen-twirl-hard} & \prof{1.24}{0.037} & \prof{1.24}{0.019} \\
\midrule
\multicolumn{6}{c}{\textbf{Visual RL}} \\
\midrule
\texttt{acrobot-swingup} & \prof{2.99}{0.023} & \prof{2.54}{0.070} & \texttt{cheetah-run} & \prof{3.20}{0.153} & \prof{3.16}{0.553} \\
\texttt{finger-spin} & \prof{4.18}{0.212} & \prof{2.98}{0.438} & \texttt{finger-turn-easy} & \prof{4.10}{0.086} & \prof{2.96}{0.050} \\
\texttt{finger-turn-hard} & \prof{3.38}{0.279} & \prof{3.04}{0.350} & \texttt{quadruped-walk} & \prof{4.42}{0.101} & \prof{3.72}{0.029} \\
\texttt{reacher-easy} & \prof{3.48}{0.280} & \prof{2.80}{0.076} & \texttt{reacher-hard} & \prof{2.91}{0.027} & \prof{2.53}{0.011} \\
\texttt{walker-run} & \prof{3.45}{0.104} & \prof{3.43}{0.439} & \texttt{walker-walk} & \prof{3.44}{0.089} & \prof{3.24}{0.048} \\
\bottomrule
\end{tabular}
\end{adjustbox}
\end{table}

\begin{table}[ht]
\centering
\scriptsize
\setlength{\tabcolsep}{2pt}
\renewcommand{\arraystretch}{0.86}
\caption{\textbf{Profiling results of evaluation time for online tasks} (seconds per evaluation phase, 10 episodes in total) across domains. Values are reported as mean $\pm$ std. All experiments are evaluated with a single task running on a single NVIDIA A800-SXM4-80GB GPU. 
}
\label{tab:profiling_results}
\begin{adjustbox}{max width=\linewidth, max totalheight=0.95\textheight, center}
\begin{tabular}{@{}lcc@{\hspace{0.8em}}lcc@{}}
\toprule
\textbf{Task} & \textbf{TD-MPC2} & \textbf{MBDPO} & \textbf{Task} & \textbf{TD-MPC2} & \textbf{MBDPO} \\
\midrule

\multicolumn{6}{c}{\textbf{MetaWorld}} \\
\midrule
\texttt{mw-assembly} & \prof{10.67}{0.03} & \prof{18.89}{0.02} & \texttt{mw-basketball} & \prof{10.31}{0.01} & \prof{18.76}{0.04} \\
\texttt{mw-button-press-topdown} & \prof{10.18}{0.01} & \prof{18.77}{0.02} & \texttt{mw-button-press-topdown-wall} & \prof{10.32}{0.01} & \prof{18.73}{0.08} \\
\texttt{mw-button-press} & \prof{10.06}{0.01} & \prof{18.39}{0.02} & \texttt{mw-button-press-wall} & \prof{10.16}{0.01} & \prof{18.51}{0.02} \\
\texttt{mw-coffee-button} & \prof{10.33}{0.02} & \prof{18.48}{0.03} & \texttt{mw-coffee-pull} & \prof{10.59}{0.04} & \prof{18.73}{0.16} \\
\texttt{mw-coffee-push} & \prof{10.44}{0.01} & \prof{18.79}{0.02} & \texttt{mw-dial-turn} & \prof{10.25}{0.01} & \prof{18.31}{0.03} \\
\texttt{mw-disassemble} & \prof{10.48}{0.01} & \prof{18.68}{0.09} & \texttt{mw-door-open} & \prof{10.29}{0.01} & \prof{18.69}{0.01} \\
\texttt{mw-door-close} & \prof{10.26}{0.01} & \prof{19.02}{0.04} & \texttt{mw-drawer-close} & \prof{10.36}{0.03} & \prof{18.60}{0.02} \\
\texttt{mw-drawer-open} & \prof{10.32}{0.05} & \prof{18.67}{0.08} & \texttt{mw-faucet-open} & \prof{10.36}{0.06} & \prof{18.38}{0.12} \\
\texttt{mw-faucet-close} & \prof{10.14}{0.02} & \prof{18.15}{0.04} & \texttt{mw-hammer} & \prof{10.62}{0.03} & \prof{18.87}{0.04} \\
\texttt{mw-handle-press-side} & \prof{10.08}{0.00} & \prof{18.53}{0.09} & \texttt{mw-handle-press} & \prof{10.15}{0.02} & \prof{18.61}{0.04} \\
\texttt{mw-handle-pull-side} & \prof{10.21}{0.00} & \prof{18.40}{0.02} & \texttt{mw-handle-pull} & \prof{10.25}{0.03} & \prof{18.52}{0.04} \\
\texttt{mw-lever-pull} & \prof{10.31}{0.02} & \prof{18.73}{0.12} & \texttt{mw-peg-insert-side} & \prof{10.38}{0.05} & \prof{18.79}{0.04} \\
\texttt{mw-peg-unplug-side} & \prof{10.42}{0.03} & \prof{18.73}{0.04} & \texttt{mw-pick-out-of-hole} & \prof{10.66}{0.01} & \prof{18.73}{0.02} \\
\texttt{mw-pick-place} & \prof{10.33}{0.01} & \prof{18.50}{0.07} & \texttt{mw-pick-place-wall} & \prof{10.45}{0.02} & \prof{18.70}{0.03} \\
\texttt{mw-plate-slide} & \prof{10.27}{0.01} & \prof{18.42}{0.03} & \texttt{mw-plate-slide-side} & \prof{10.15}{0.01} & \prof{18.57}{0.30} \\
\texttt{mw-plate-slide-back} & \prof{10.36}{0.04} & \prof{18.67}{0.03} & \texttt{mw-plate-slide-back-side} & \prof{10.42}{0.02} & \prof{19.25}{0.09} \\
\texttt{mw-push-back} & \prof{10.32}{0.01} & \prof{19.04}{0.03} & \texttt{mw-push} & \prof{10.37}{0.02} & \prof{18.64}{0.04} \\
\texttt{mw-push-wall} & \prof{10.62}{0.05} & \prof{18.62}{0.02} & \texttt{mw-reach} & \prof{10.16}{0.01} & \prof{18.60}{0.04} \\
\texttt{mw-reach-wall} & \prof{10.14}{0.03} & \prof{18.42}{0.04} & \texttt{mw-shelf-place} & \prof{10.42}{0.02} & \prof{18.64}{0.02} \\
\texttt{mw-soccer} & \prof{10.45}{0.05} & \prof{19.08}{0.05} & \texttt{mw-stick-push} & \prof{10.57}{0.05} & \prof{19.04}{0.03} \\
\texttt{mw-stick-pull} & \prof{10.65}{0.02} & \prof{19.19}{0.03} & \texttt{mw-sweep-into} & \prof{10.64}{0.05} & \prof{18.53}{0.03} \\
\texttt{mw-sweep} & \prof{10.32}{0.01} & \prof{18.71}{0.02} & \texttt{mw-window-open} & \prof{10.48}{0.02} & \prof{18.48}{0.02} \\
\texttt{mw-window-close} & \prof{10.41}{0.01} & \prof{18.59}{0.03} & \texttt{mw-bin-picking} & \prof{10.36}{0.01} & \prof{19.01}{0.03} \\
\texttt{mw-box-close} & \prof{10.51}{0.02} & \prof{18.44}{0.08} & \texttt{mw-door-lock} & \prof{10.47}{0.05} & \prof{18.65}{0.06} \\
\texttt{mw-door-unlock} & \prof{10.45}{0.02} & \prof{18.86}{0.04} & \texttt{mw-hand-insert} & \prof{10.73}{0.04} & \prof{18.89}{0.04} \\
\midrule

\multicolumn{6}{c}{\textbf{DMControl}} \\
\midrule
\texttt{acrobot-swingup} & \prof{48.96}{0.04} & \prof{91.41}{0.05} & \texttt{cartpole-balance} & \prof{49.54}{0.02} & \prof{93.11}{0.49} \\
\texttt{cartpole-balance-sparse} & \prof{48.59}{0.02} & \prof{89.92}{0.31} & \texttt{cartpole-swingup} & \prof{49.14}{0.02} & \prof{95.71}{0.21} \\
\texttt{cartpole-swingup-sparse} & \prof{48.53}{0.02} & \prof{90.39}{0.64} & \texttt{cheetah-jump} & \prof{48.64}{0.02} & \prof{90.12}{0.42} \\
\texttt{cheetah-run} & \prof{48.04}{0.05} & \prof{88.86}{0.06} & \texttt{cheetah-run-back} & \prof{48.40}{0.07} & \prof{89.86}{0.05} \\
\texttt{cheetah-run-backwards} & \prof{48.28}{0.19} & \prof{89.26}{0.20} & \texttt{cheetah-run-front} & \prof{48.57}{0.18} & \prof{93.52}{0.31} \\
\texttt{cup-catch} & \prof{48.53}{0.06} & \prof{89.35}{0.33} & \texttt{cup-spin} & \prof{49.28}{0.08} & \prof{88.98}{0.08} \\
\texttt{dog-run} & \prof{81.49}{0.49} & \prof{116.13}{0.88} & \texttt{dog-trot} & \prof{82.57}{0.43} & \prof{112.60}{0.74} \\
\texttt{dog-stand} & \prof{80.86}{0.47} & \prof{117.07}{0.78} & \texttt{dog-walk} & \prof{76.16}{0.14} & \prof{102.96}{0.42} \\
\texttt{finger-spin} & \prof{48.39}{0.08} & \prof{88.62}{0.51} & \texttt{finger-turn-easy} & \prof{48.70}{0.07} & \prof{90.84}{0.38} \\
\texttt{finger-turn-hard} & \prof{49.44}{0.47} & \prof{90.07}{0.18} & \texttt{fish-swim} & \prof{50.22}{0.40} & \prof{90.99}{0.75} \\
\texttt{hopper-hop} & \prof{48.45}{0.05} & \prof{90.15}{0.08} & \texttt{hopper-hop-backwards} & \prof{48.56}{0.03} & \prof{90.67}{0.06} \\
\texttt{hopper-stand} & \prof{48.93}{0.04} & \prof{89.40}{0.07} & \texttt{humanoid-run} & \prof{65.93}{0.16} & \prof{93.01}{0.31} \\
\texttt{humanoid-stand} & \prof{66.68}{0.42} & \prof{93.37}{0.98} & \texttt{humanoid-walk} & \prof{65.95}{0.19} & \prof{93.96}{0.15} \\
\texttt{pendulum-spin} & \prof{48.23}{0.04} & \prof{89.49}{0.40} & \texttt{pendulum-swingup} & \prof{48.29}{0.09} & \prof{89.57}{0.93} \\
\texttt{quadruped-run} & \prof{51.07}{0.10} & \prof{92.25}{0.34} & \texttt{quadruped-walk} & \prof{50.48}{0.06} & \prof{94.07}{0.14} \\
\texttt{reacher-easy} & \prof{48.58}{0.05} & \prof{89.32}{0.39} & \texttt{reacher-hard} & \prof{48.50}{0.19} & \prof{89.07}{0.38} \\
\texttt{reacher-three-easy} & \prof{48.62}{0.04} & \prof{89.35}{0.04} & \texttt{reacher-three-hard} & \prof{48.29}{0.08} & \prof{88.68}{0.42} \\
\texttt{walker-run} & \prof{49.95}{0.03} & \prof{90.58}{0.17} & \texttt{walker-run-backwards} & \prof{49.51}{0.03} & \prof{89.00}{0.12} \\
\texttt{walker-stand} & \prof{49.05}{0.03} & \prof{92.53}{0.13} & \texttt{walker-walk} & \prof{49.31}{0.08} & \prof{90.07}{0.08} \\
\texttt{walker-walk-backwards} & \prof{49.61}{0.04} & \prof{90.30}{0.06} &  &  &  \\
\midrule

\multicolumn{6}{c}{\textbf{ManiSkill2}} \\
\midrule
\texttt{lift-cube} & \prof{15.27}{0.32} & \prof{24.99}{1.19} & \texttt{pick-cube} & \prof{15.62}{0.27} & \prof{26.34}{0.76} \\
\texttt{stack-cube} & \prof{16.08}{0.30} & \prof{27.02}{0.65} & \texttt{pick-ycb} & \prof{18.83}{1.01} & \prof{31.01}{2.20} \\
\texttt{turn-faucet} & \prof{18.84}{1.11} & \prof{25.06}{0.55} &  &  &  \\
\midrule

\multicolumn{6}{c}{\textbf{MyoSuite}} \\
\midrule
\texttt{myo-reach} & \prof{13.90}{0.03} & \prof{19.22}{0.02} & \texttt{myo-reach-hard} & \prof{14.00}{0.04} & \prof{19.45}{0.09} \\
\texttt{myo-pose} & \prof{13.90}{0.07} & \prof{19.00}{0.02} & \texttt{myo-pose-hard} & \prof{13.95}{0.04} & \prof{19.08}{0.03} \\
\texttt{myo-obj-hold} & \prof{13.88}{0.05} & \prof{19.79}{0.02} & \texttt{myo-obj-hold-hard} & \prof{14.01}{0.05} & \prof{19.34}{0.05} \\
\texttt{myo-key-turn} & \prof{14.27}{0.08} & \prof{18.84}{0.04} & \texttt{myo-key-turn-hard} & \prof{14.10}{0.03} & \prof{19.22}{0.05} \\
\texttt{myo-pen-twirl} & \prof{13.66}{0.05} & \prof{18.81}{0.03} & \texttt{myo-pen-twirl-hard} & \prof{13.67}{0.07} & \prof{18.72}{0.02} \\
\midrule

\multicolumn{6}{c}{\textbf{Visual RL}} \\
\midrule
\texttt{acrobot-swingup} & \prof{66.51}{0.42} & \prof{108.95}{0.30} & \texttt{cheetah-run} & \prof{66.29}{0.37} & \prof{108.64}{0.31} \\
\texttt{finger-spin} & \prof{66.35}{0.47} & \prof{108.86}{0.28} & \texttt{finger-turn-easy} & \prof{66.36}{0.46} & \prof{113.39}{0.09} \\
\texttt{finger-turn-hard} & \prof{65.37}{2.59} & \prof{105.12}{0.20} & \texttt{quadruped-walk} & \prof{66.66}{1.69} & \prof{106.39}{0.67} \\
\texttt{reacher-easy} & \prof{65.29}{2.59} & \prof{104.03}{0.94} & \texttt{reacher-hard} & \prof{65.41}{2.63} & \prof{104.97}{0.48} \\
\texttt{walker-run} & \prof{66.02}{0.15} & \prof{104.89}{0.63} & \texttt{walker-walk} & \prof{66.13}{0.23} & \prof{105.64}{0.20} \\

\bottomrule
\end{tabular}
\end{adjustbox}
\end{table}

% \begin{table}[h]
% \centering
% \caption{Profiling results of offline multi-task training time (hours per 10M steps). Values are reported as mean $\pm$ std. All experiments are evaluated on a single NVIDIA H200 GPU.}
% \label{tab:offline_training_time}
% \setlength{\tabcolsep}{8pt}
% \renewcommand{\arraystretch}{1.15}
% \begin{tabular}{lccccc}
% \toprule
%  & \multicolumn{5}{c}{Model Size} \\
% \cmidrule(lr){2-6}
%  Offline Task & 1.7M & 6M & 21M & 54M & 340M \\
% \midrule
% \parbox[c][2.8em][c]{3.1cm}{\raggedright DMControl (30 tasks)} & 
% $34.8 \pm 0.6$ & $35.0 \pm 0.5$ & $39.8 \pm 0.3$ & $52.6 \pm 0.1$ & $202.9 \pm 0.9$ \\
% \cmidrule(lr){1-6}
% \parbox[c][2.8em][c]{3.1cm}{\raggedright MetaWorld \&\\DMControl (80 tasks)} & 
% $39.9 \pm 0.4$ & $40.5 \pm 0.6$ & $42.6 \pm 0.4$ & $69.6 \pm 0.7$ & $206.8 \pm 0.8$ \\
% \bottomrule
% \end{tabular}
% \end{table}

\clearpage
\section{Technical Definition} \label{append:technical_definition}
In this section, we provide further technical details and formal definitions to supplement the descriptions in the main text.

\subsection{Diffusion Policy based on DDPM}
We briefly review the fundamental mechanism of Denoising Diffusion Probabilistic Models (DDPM) \citep{ho2020denoising, song2020score}. In our setting, the diffusion process based on DDPM is used for sampling and generating target policies. Let $\tau \in \{1, \dots, N\}$ denote the diffusion timestep and $\bar{\alpha}^{\tau} \in (0, 1]$ be the cumulative noise schedule. The forward and reverse diffusion processes for the action sequence are defined as: 
\begin{align} 
\text{(Forward)} \quad & p(a_{t:t+H}^{\tau} | a_{t:t+H}^{0}, z_t) = \mathcal{N}\left(a_{t:t+H}^{\tau}; \sqrt{\bar{\alpha}^{\tau}} a_{t:t+H}^{0}, (1-\bar{\alpha}^{\tau}) I\right), \label{eq:forward_process} \\
\text{(Reverse)} \quad & a_{t:t+H}^{\tau-1} = \frac{1}{\sqrt{\alpha^\tau}}
\left(
a_{t:t+H}^{\tau}
+
(1-\alpha^{\tau})
\phi(z_t, a_{t:t+H}^{\tau}, \tau)
\right)
+
\sigma(\tau)\epsilon,
\quad
\epsilon \sim \mathcal{N}(0, I), \label{eq:reverse_process}
\end{align}
where $p(a_{t:t+H}^{\tau} | a_{t:t+H}^{0}, z_t)$ is the transition kernel in the forward process, and $\phi(z_t, a_{t:t+H}^{\tau}, \tau) = \nabla_{a_{t:t+H}^{\tau}} \log \pi^{\tau}(a_{t:t+H}^{\tau} | z_t)$ denotes the score function at diffusion step $\tau$. In our framework, as policy optimization is performed within the latent world model, the explicit forward diffusion process is not strictly required for training. Nevertheless, the forward process remains essential for computing the transition kernel in the reverse process, as formulated in \eqref{eq:transition_kernel}.

% \subsection{Bellman Operator}
\subsection{Bellman Operator}
To formally analyze value estimation within the world model, we define the Bellman operator as follows \citep{sutton1998reinforcement}:

\begin{definition}[Bellman Operator] \label{def:bellman_operator}
Consider the space of $Q$-functions $\mathcal{Q}$ equipped with the $L_\infty$ norm\footnote{Formally, $\mathcal{Q} = \{Q | Q: \mathcal{Z} \times \mathcal{A} \to \mathbb{R}\}$ is the set of all bounded real-valued functions defined on the latent state-action space. The $L_\infty$ norm is defined as $\|Q\|_\infty = \sup_{(z, a) \in \mathcal{Z} \times \mathcal{A}} |Q(z, a)|$.}. For a given policy $\pi$, the Bellman evaluation operator $T^\pi: \mathcal{Q} \to \mathcal{Q}$ is defined as:
\begin{equation}
    (T^\pi Q)(z_t, a_t) = \mathcal{R}(z_t, a_t) + \gamma \mathbb{E}_{z_{t+1} \sim \mathcal{F}(z_t, a_t), a_{t+1} \sim \pi}[Q(z_{t+1}, a_{t+1})],
\end{equation}
where $\mathcal{R}(z_t, a_t)$ denotes the latent reward function, $\mathcal{F}$ represents the latent dynamics within the learned world model, and $\gamma \in (0, 1)$ is the discount factor.
\end{definition}
This operator serves exclusively for policy evaluation and value function update. By repeatedly applying $T^\pi$, the $Q$-function converges to $Q^\pi$ (according to the Banach fixed-point theorem \citep{bellman2013stability}), the true value of the current policy $\pi$. This evaluation stage is crucial as it provides an accurate value estimate, which MBDPO then uses to guide the policy optimization process toward higher-reward regions.

\clearpage
% \section{Theoretical Analysis}
\section{Theoretical Analysis} \label{append:theoretical_analysis}

In this section, we provide a theoretical analysis of the proposed framework. We first examine the misalignment between search and value learning, highlighting how discrepancies between the search policy and the learned value function can affect optimization. We then analyze MBDPO from a theoretical perspective, characterizing its learning dynamics and discussing why it can mitigate such misalignment. Finally, we present an error analysis of the overall algorithm, identifying the key sources of approximation error and their impact on the resulting policy.

\subsection{Bottlenecks in the Existing Paradigm} \label{append:bottleneck_analysis}

Before presenting our main results, we first provide a detailed mathematical analysis of policy search in the TD-MPC series~\citep{hansen2022temporal, hansen2025learning, zhan2025bootstrap, lin2025tdmpc2improvingtemporaldifference}.

\subsubsection{How Does Policy Search Work in the TD-MPC Series?} \label{append:previous_method_stragety}

The policy search procedure can be decomposed into two steps:

\begin{itemize}
    \item \textbf{Step 1: Initialize with a Prior Policy.}
    The search distribution is initialized with an isotropic Gaussian prior policy 
    $p_0(\cdot|z) \sim \mathcal{N}(\mu(z), \sigma^2)$ (usually parameterized by a policy network), where the mean $\mu$ is estimated by maximizing the objective in~\eqref{eq:policy_objective}.

    \item \textbf{Step 2: Policy Search via Model Predictive Path Integral Planning (MPPI).}
    Starting from the initialized prior policy $\mathcal{N}(\mu(z), \sigma^2)$, the algorithm iteratively samples candidate action sequences and evaluates them using model-based rollouts \citep{williams2017model}. These samples are then reweighted according to their estimated returns (abstracted in \eqref{eq:mppi_principle}), 
    \begin{align} \label{eq:mppi_principle}
        \underbrace{p_0(a_{t:t+H}|z_t)}_{\text{Gaussian prior}}
\xrightarrow{\text{sample}}
\underbrace{\{a_{t:t+H}^{(i)}\}_{i=1}^{T}}_{\text{candidate actions}}
\xrightarrow{\text{evaluated by $H$-step } \hat{Q}}
\underbrace{\{\hat{Q}^{(i)}\}_{i=1}^{T}}_{\text{estimated values}}
\xrightarrow{\text{reweight}}
\underbrace{\pi(a_{t:t+H}|z_t)}_{\text{search target}}, 
    \end{align}
    where $T$ represents the sample number, which yields an improved policy target through importance-weighted aggregation.
\end{itemize}

This procedure highlights a key limitation of policy search in the TD-MPC series. The resulting target policy is induced by the search distribution rather than obtained through an explicit policy optimization problem constrained around an anchored policy distribution. As a result, the searched target can become effectively uncontrolled: its update is primarily determined by the return estimates used for importance weighting. When the learned value function or model-based return estimates are inaccurate, erroneously high $Q$-values may assign excessive weights to suboptimal action sequences. Consequently, the aggregated policy target can be biased toward poor actions, causing policy learning to imitate a bad search-induced target rather than ensuring a monotonic improvement over the prior policy. In this sense, the search step does not constitute a principled policy optimization procedure; it constructs a heuristic target from weighted samples of learned world models, without explicitly optimizing a regularized policy objective or enforcing an improvement guarantee.

\subsubsection{Analysis of Bottleneck~\ref{prob_1}
% and \ref{prob_2}
}~\label{append:proof_prob_1}
We start from the theoretical analysis from the Bellman update. In the setting of a latent world model, the $\hat{Q}$ value function is learned from the (behavioral) policy $\beta$ (parameterized with a policy network). 

% \textbf{Analysis of Problem~\ref{prob_1}.} 
\begin{proof}
By Definition~\ref{def:bellman_operator}, for any latent state-action pair $(z_t,a_t)$, we have
\begin{equation}
(T^\pi \hat Q)(z_t,a_t)
=
\mathcal{R}(z_t,a_t)
+
\gamma 
\mathbb{E}_{z_{t+1}\sim \mathcal{F}(z_t, a_t)}
\mathbb{E}_{a_{t+1}\sim \pi(\cdot|z_{t+1})}
[\hat Q(z_{t+1},a_{t+1})],
\end{equation}
and
\begin{equation}
(T^\beta \hat Q)(z_t,a_t)
=
\mathcal{R}(z_t,a_t)
+
\gamma 
\mathbb{E}_{z_{t+1}\sim \mathcal{F}(z_t, a_t)}
\mathbb{E}_{a_{t+1}\sim \beta(\cdot|z_{t+1})}
[\hat Q(z_{t+1},a_{t+1})].
\end{equation}
Therefore,
\begin{align}
&\left|(T^\pi \hat Q)(z_t,a_t) - (T^\beta \hat Q)(z_t,a_t)\right| \nonumber \\
&=
\gamma
\left|
\mathbb{E}_{z_{t+1}\sim \mathcal{F}(z_t, a_t)}
\left[
\sum_{a_{t+1}}
\left(
\pi(a_{t+1}|z_{t+1})
-
\beta(a_{t+1}|z_{t+1})
\right)
\hat Q(z_{t+1},a_{t+1})
\right]
\right| \nonumber \\
&\leq
\gamma
\mathbb{E}_{z_{t+1}\sim \mathcal{F}(z_t, a_t)}
\left[
\left|
\sum_{a_{t+1}}
\left(
\pi(a_{t+1}|z_{t+1})
-
\beta(a_{t+1}|z_{t+1})
\right)
\hat Q(z_{t+1},a_{t+1})
\right|
\right] \nonumber \\
&\leq
\gamma \|\hat Q\|_\infty
\mathbb{E}_{z_{t+1}\sim \mathcal{F}(z_t, a_t)}
\left[
\int_{a_{t+1}} 
\left|
\pi(a_{t+1}|z_{t+1})
-
\beta(a_{t+1}|z_{t+1})
\right|da_{t+1}
\right]  \nonumber \\
&=
\gamma \|\hat Q\|_\infty
\mathbb{E}_{z_{t+1}\sim \mathcal{F}(z_t, a_t)}
\left[
\|\pi(\cdot|z_{t+1})-\beta(\cdot|z_{t+1})\|_1
\right].
\end{align}
By Pinsker's inequality,
\begin{equation}
\|\pi(\cdot|z_{t+1})-\beta(\cdot|z_{t+1})\|_1
\leq
\sqrt{
2D_{\mathrm{KL}}
\left(
\pi(\cdot|z_{t+1})\|\beta(\cdot|z_{t+1})
\right)
}.
\end{equation}
Thus,
\begin{equation}
\left|(T^\pi \hat Q)(z_t,a_t) - (T^\beta \hat Q)(z_t,a_t)\right|
\leq
\gamma \|\hat Q\|_\infty
\mathbb{E}_{z_{t+1}\sim \mathcal{F}(z_t, a_t)}
\left[
\sqrt{
2D_{\mathrm{KL}}
\left(
\pi(\cdot|z_{t+1})\|\beta(\cdot|z_{t+1})
\right)}
\right].
\end{equation}
Taking the supremum over all latent state-action pairs $(z_t,a_t) \in \mathcal{Z}\times \mathcal{A}$ gives
\begin{equation}
\begin{split}
    \|T^\pi \hat Q - T^\beta \hat Q\|_{\infty}
& \leq
\gamma \|\hat Q\|_{\infty}
\sup_{z_t,a_t}
\mathbb{E}_{z_{t+1}\sim \mathcal{F}(z_t, a_t)}
\left[
\sqrt{
2D_{\mathrm{KL}}
\left(
\pi(\cdot|z_{t+1})\|\beta(\cdot|z_{t+1})
\right)}
\right] \\
& \leq \gamma \|\hat Q\|_{\infty} \sqrt{2D_{\text{KL}}^{\max}(\pi\|\beta)},
\end{split}
\end{equation}
with $D_{\text{KL}}^{\max}(\pi\|\beta)\coloneqq \sup_{z\in \mathcal{Z}} D_{\text{KL}}(\pi(\cdot|z)\|\beta(\cdot|z))$, which completes the proof.
\end{proof}

\subsection{Analysis of Bottleneck~\ref{prob_2}} \label{append:proof_prob_2}
\begin{proof}
We consider the latent MDP induced by the world model transition 
$z_{t+1}=\mathcal{F}(z_t,a_t)$. Equivalently, the transition kernel is deterministic:
\begin{equation}
P_{\mathcal{F}}(z'|z,a)=\mathbb{1}_{\{z'=\mathcal{F}(z,a)\}},
\end{equation}
where $\mathbb{1}_{\{ \cdot \}}$ is the indicator function. Let $d^{\pi}$ and $d^{\beta}$ be the discounted latent-state occupancy distributions under policies $\pi$ and $\beta$, respectively.

By the performance difference lemma \citep{mahadevan1996average}, the true improvement of $\pi$ over $\beta$ can be written as
\begin{equation}
J(\pi)-J(\beta)
=
\frac{1}{1-\gamma}
\mathbb{E}_{z\sim d^{\pi}, a\sim \pi(\cdot|z)}
\left[
A^\beta(z,a)
\right],
\end{equation}
where
\begin{equation}
A^\beta(z,a)
=
Q^\beta(z,a)-V^\beta(z),
\qquad
V^\beta(z)
=
\mathbb{E}_{a\sim \beta(\cdot|z)}
[Q^\beta(z,a)].
\end{equation}
Equivalently,
\begin{align}
J(\pi)-J(\beta)
&=
\frac{1}{1-\gamma}
\mathbb{E}_{z\sim d^{\pi}}
\left[
\mathbb{E}_{a\sim \pi(\cdot|z)}Q^\beta(z,a)
-
\mathbb{E}_{a\sim \beta(\cdot|z)}Q^\beta(z,a)
\right].
\end{align}

We now decompose the right-hand side into a surrogate term evaluated under 
$d^{\beta}$ and a distribution-shift term:
\begin{align}
J(\pi)-J(\beta)
&=
\frac{1}{1-\gamma}
\mathbb{E}_{z\sim d^{\beta}}
\left[
\mathbb{E}_{a\sim \pi(\cdot|z)}Q^\beta(z,a)
-
\mathbb{E}_{a\sim \beta(\cdot|z)}Q^\beta(z,a)
\right]
+
\Delta_{\mathrm{occ}},
\end{align}
where
\begin{align}
\Delta_{\mathrm{occ}}
&=
\frac{1}{1-\gamma}
\left(
\mathbb{E}_{z\sim d^{\pi}}
-
\mathbb{E}_{z\sim d^{\beta}}
\right)
\left[
\mathbb{E}_{a\sim \pi(\cdot|z)}Q^\beta(z,a)
-
\mathbb{E}_{a\sim \beta(\cdot|z)}Q^\beta(z,a)
\right].
\end{align}
Under the standard boundedness assumption 
$\|Q^\beta\|_\infty \leq Q_{\max}$, the distribution-shift term can be bounded by a trust-region style argument \citep{schulman2015trust}:
\begin{equation}
|\Delta_{\mathrm{occ}}|
\leq
C_1 D_{\mathrm{KL}}^{\max}(\pi\|\beta),
\end{equation}
where
\begin{equation}
D_{\mathrm{KL}}^{\max}(\pi\|\beta)
\coloneqq
\sup_z D_{\mathrm{KL}}
\left(
\pi(\cdot|z)\|\beta(\cdot|z)
\right).
\end{equation}
This bound follows from the fact that the discrepancy between the occupancy distributions 
$d^{\pi}$ and $d^{\beta}$ is controlled by the maximum policy divergence at each latent state, together with Pinsker's inequality.

Therefore,
\begin{align}
J(\pi)-J(\beta)
&\geq
\frac{1}{1-\gamma}
\mathbb{E}_{z\sim d^{\beta}}
\left[
\mathbb{E}_{a\sim \pi(\cdot|z)}Q^\beta(z,a)
-
\mathbb{E}_{a\sim \beta(\cdot|z)}Q^\beta(z,a)
\right]
-
C_1 D_{\mathrm{KL}}^{\max}(\pi\|\beta).
\end{align}

Next, we relate the true value function $Q^\beta$ to its approximation $\hat Q$. 
For any policy $\mu$, we have
\begin{equation}
\left|
\mathbb{E}_{z\sim d^{\beta}, a\sim \mu(\cdot|z)}
\left[
Q^\beta(z,a)-\hat Q(z,a)
\right]
\right|
\leq
\|Q^\beta-\hat Q\|_\infty.
\end{equation}
Applying this inequality to both $\mu=\pi$ and $\mu=\beta$ gives
\begin{align}
&\mathbb{E}_{z\sim d^{\beta}}
\left[
\mathbb{E}_{a\sim \pi(\cdot|z)}Q^\beta(z,a)
-
\mathbb{E}_{a\sim \beta(\cdot|z)}Q^\beta(z,a)
\right] \nonumber \\
&\geq
\mathbb{E}_{z\sim d^{\beta}}
\left[
\mathbb{E}_{a\sim \pi(\cdot|z)}\hat Q(z,a)
-
\mathbb{E}_{a\sim \beta(\cdot|z)}\hat Q(z,a)
\right]
-
2\|Q^\beta-\hat Q\|_\infty.
\end{align}
By the definition of the surrogate objective, the first term on the right-hand side is exactly
\begin{equation}
\hat J_\beta(\pi)-\hat J_\beta(\beta).
\end{equation}
Thus,
\begin{align}
J(\pi)-J(\beta)
&\geq
\frac{1}{1-\gamma}
\left(
\hat J_\beta(\pi)-\hat J_\beta(\beta)
\right)
-
C_1 D_{\mathrm{KL}}^{\max}(\pi\|\beta)
-
C_2
\|Q^\beta-\hat Q\|_\infty,
\end{align}
% Absorbing the factor $\frac{1}{1-\gamma}$ into the constants gives
% \begin{equation}
% J(\pi)-J(\beta)
% \geq
% \hat J_\beta(\pi)-\hat J_\beta(\beta)
% -
% C_1 D_{\mathrm{KL}}^{\max}(\pi\|\beta)
% -
% C_2\|Q^\beta-\hat Q\|_\infty,
% \end{equation}
where $C_1,C_2 = \frac{2}{1 - \gamma}>0$ are constants depending on $\gamma$ and the boundedness of the value function. 
This completes the proof.
\end{proof}

\paragraph{Implication: search policy should be anchored to the non-search behavioral policy.}
The above two results reveal a common source of failure when the search policy $\pi$ is optimized without being anchored to the behavioral policy $\beta$. 
From~\eqref{eq:bellman_gap}, the Bellman update induced by $\pi$ deviates from the Bellman update induced by $\beta$ by an amount controlled by the policy divergence:
\begin{equation}
\|T^\pi \hat Q - T^\beta \hat Q\|_{\infty}
\leq
\gamma \|\hat Q\|_{\infty}
\sqrt{2D_{\mathrm{KL}}^{\max}(\pi\|\beta)}.
\end{equation}
This means that when $\pi$ moves far away from $\beta$, the value function $\hat Q$, which is learned from data generated by $\beta$, is no longer updated consistently with the distribution on which it was trained. 
As a result, applying $T^\pi$ can propagate value estimates into latent regions that are poorly covered by the behavioral policy, where $\hat Q$ may be inaccurate.

Meanwhile,~\eqref{eq:updating_gap} shows that the true policy improvement is lower bounded by the estimated improvement only up to two error terms:
\begin{equation}
J(\pi)-J(\beta)
\geq
\frac{1}{1-\gamma}
\left(
\hat J_\beta(\pi)-\hat J_\beta(\beta)
\right)
-
C_1D_{\mathrm{KL}}^{\max}(\pi\|\beta)
-
C_2\|Q^\beta-\hat Q\|_\infty.
\end{equation}
The first penalty is the TRPO-style distribution-shift term \citep{schulman2015trust}, which indicates that policy improvement is inherently local around the behavioral (or reference) policy $\beta$. 
The second penalty captures the value estimation gap introduced by using the learned value function $\hat Q$ instead of the true value function $Q^\beta$. 
Together, these two terms imply that a large surrogate improvement under $\hat J_\beta(\pi)$ does not necessarily translate into true improvement when $\pi$ deviates substantially from $\beta$.

Therefore, the two bottlenecks reinforce each other. 
An unconstrained search policy may exploit spurious high-value regions of $\hat Q$ and move away from the data-supported behavior distribution. 
This increases the KL divergence between $\pi$ and $\beta$, enlarges the Bellman update mismatch in~\eqref{eq:bellman_gap}, and further increases the value estimation error appearing in~\eqref{eq:updating_gap}. 
Consequently, the search procedure can enter a self-reinforcing cycle of distribution shift and value overestimation, where inaccurate value estimates attract the policy, and the shifted policy further amplifies value errors.

This analysis suggests that search in latent world models should not be performed as an unconstrained policy improvement step. 
Instead, the search policy should be regularized or anchored to the behavioral policy $\beta$, for example through a KL constraint or a policy-prior term. 
Such an anchoring mechanism keeps the search policy within the region where the learned value function is reliable, thereby reducing Bellman mismatch, controlling distribution shift, and enabling more stable monotonic improvement within the world model.

\clearpage
\subsection{Derivation of Exact Score Function in~\eqref{eq:score_matching}}  \label{section:score_function_proof}
\begin{lemma}[Exact Score Function] \label{lemma:score_function_proof}
Consider the forward and reverse diffusion processes defined in 
\eqref{eq:forward_process} and \eqref{eq:reverse_process} in the latent world model over 
$\mathcal{Z}$. 
For a latent state $z_t$ and a noisy action sequence $a_{t:t+H}^{\tau}$ at diffusion step 
$\tau$, define a proposal distribution
\begin{equation}
    q^{\tau}(a_{t:t+H}^0 | a_{t:t+H}^{\tau},z_t)
    =
    \mathcal{N}\left(a_{t:t+H}^0 ;
    \frac{1}{\sqrt{\bar{\alpha}^{\tau}}}a_{t:t+H}^{\tau},
    \frac{1-\bar{\alpha}^{\tau}}{\bar{\alpha}^{\tau}} I
    \right).
\end{equation}
Then the exact score function with respect to $a_{t:t+H}^{\tau}$ can be written as
\begin{align}
    \phi(z_t, a_{t:t+H}^\tau, \tau)
    &=
    -\frac{a^{\tau}_{t:t+H}}{1-\bar{\alpha}^{\tau}}
    +
    \frac{\sqrt{\bar{\alpha}^{\tau}}}{1-\bar{\alpha}^{\tau}}
    \mathbb{E}_{a_{t:t+H}^0 \sim q_{\tau}(\cdot | a_{t:t+H}^{\tau})}
    \left[
    w(a^{0|\tau}_{t:t+H}) a_{t:t+H}^{0|\tau}
    \right],
    \label{eq:score_matching_append}
\end{align}
where the normalized importance weight is
\begin{equation} \label{eq:importance_weight_append}
    w(a_{t:t+H}^{0|\tau})
    =
    \frac{
    \exp\left(\tilde{G}(z_t, a_{t:t+H}^0)/\kappa\right)
    }{
    \mathbb{E}_{\tilde a_{t:t+H}^0 \sim q_{\tau}(\cdot | a_{t:t+H}^{\tau})}
    \left[
    \exp\left(\tilde{G}(z_t, \tilde a_{t:t+H}^0)/\kappa\right)
    \right]
    }.
\end{equation}
\end{lemma}

\begin{proof}
We start from the definition of the score function. For a fixed latent state $z_t$, the marginal distribution of the noisy action sequence $a_{t:t+H}^{\tau}$ is given by
\begin{equation}
    \pi_\phi^\tau(a_{t:t+H}^{\tau}| z_t)
    =
    \int 
    p(a_{t:t+H}^{\tau}| a_{t:t+H}^{0}, z_t)
    \pi_\phi^0(a_{t:t+H}^{0}| z_t)
    da_{t:t+H}^{0}.
\end{equation}
Therefore,
\begin{equation}
\begin{split}
\phi(z_t, a_{t:t+H}^\tau, \tau)
&=
\nabla_{a_{t:t+H}^{\tau}}
\log \pi^\tau_\phi(a_{t:t+H}^{\tau} | z_t) \\
&=
\frac{
\nabla_{a_{t:t+H}^{\tau}}
\int 
p(a_{t:t+H}^{\tau}| a_{t:t+H}^{0}, z_t)
\pi_{\phi}^{0}(a_{t:t+H}^{0}| z_t)
da_{t:t+H}^{0}
}{
\int 
p(a_{t:t+H}^{\tau}| a_{t:t+H}^{0}, z_t)
\pi_{\phi}^{0}(a_{t:t+H}^{0}| z_t)
da_{t:t+H}^{0}
} \\
&=
\frac{
\int 
\nabla_{a_{t:t+H}^{\tau}}
p(a_{t:t+H}^{\tau}| a_{t:t+H}^{0}, z_t)
\pi_{\phi}^{0}(a_{t:t+H}^{0}| z_t)
da_{t:t+H}^{0}
}{
\int 
p(a_{t:t+H}^{\tau}| a_{t:t+H}^{0}, z_t)
\pi_{\phi}^{0}(a_{t:t+H}^{0}| z_t)
da_{t:t+H}^{0}
}.
\label{eq:base_score}
\end{split}
\end{equation}

Under the DDPM forward process, the transition kernel is
\begin{equation} \label{eq:inter_prob}
\begin{split}
p(a_{t:t+H}^{\tau}| a_{t:t+H}^{0}, z_t)
&=
\mathcal{N}
\left(
a_{t:t+H}^{\tau};
\sqrt{\bar{\alpha}^{\tau}}a_{t:t+H}^{0},
(1-\bar{\alpha}^{\tau})I
\right) \\
&\propto
\exp
\left(
-\frac{1}{2}
\frac{
\left(
a_{t:t+H}^{\tau}
-
\sqrt{\bar{\alpha}^{\tau}}a_{t:t+H}^{0}
\right)^\top
\left(
a_{t:t+H}^{\tau}
-
\sqrt{\bar{\alpha}^{\tau}}a_{t:t+H}^{0}
\right)
}{
1-\bar{\alpha}^{\tau}
}
\right).
\end{split}
\end{equation}
Hence,
\begin{equation}
\nabla_{a_{t:t+H}^{\tau}}
p(a_{t:t+H}^{\tau}| a_{t:t+H}^{0}, z_t)
=
-\frac{
a_{t:t+H}^{\tau}
-
\sqrt{\bar{\alpha}^{\tau}}a_{t:t+H}^{0}
}{
1-\bar{\alpha}^{\tau}
}
p(a_{t:t+H}^{\tau}| a_{t:t+H}^{0}, z_t).
\end{equation}
Substituting this identity into~\eqref{eq:base_score}, we obtain
\begin{equation}
\begin{split}
\phi(z_t, a_{t:t+H}^\tau, \tau)
&=
\frac{
\int 
-\frac{
a_{t:t+H}^{\tau}
-
\sqrt{\bar{\alpha}^{\tau}}a_{t:t+H}^{0}
}{
1-\bar{\alpha}^{\tau}
}
p(a_{t:t+H}^{\tau}| a_{t:t+H}^{0}, z_t)
\pi_{\phi}^{0}(a_{t:t+H}^{0}| z_t)
da_{t:t+H}^{0}
}{
\int 
p(a_{t:t+H}^{\tau}| a_{t:t+H}^{0}, z_t)
\pi_{\phi}^{0}(a_{t:t+H}^{0}| z_t)
da_{t:t+H}^{0}
} \\
&=
-\frac{a_{t:t+H}^{\tau}}{1-\bar{\alpha}^{\tau}}
+
\frac{\sqrt{\bar{\alpha}^{\tau}}}{1-\bar{\alpha}^{\tau}}
\frac{
\int 
p(a_{t:t+H}^{\tau}| a_{t:t+H}^{0}, z_t)
\pi_{\phi}^{0}(a_{t:t+H}^{0}| z_t)
a_{t:t+H}^{0}
da_{t:t+H}^{0}
}{
\int 
p(a_{t:t+H}^{\tau}| a_{t:t+H}^{0}, z_t)
\pi_{\phi}^{0}(a_{t:t+H}^{0}| z_t)
da_{t:t+H}^{0}
}.
\end{split}
\end{equation}
Thus, computing the exact score reduces to estimating the posterior mean
\begin{equation} \label{eq:posterior_expectation}
\begin{split}
\mathbb{E}
\left[
a_{t:t+H}^{0}
|
a_{t:t+H}^{\tau}, z_t
\right]
&=
\frac{
\int 
p(a_{t:t+H}^{\tau}| a_{t:t+H}^{0}, z_t)
\pi_{\phi}^{0}(a_{t:t+H}^{0}| z_t)
a_{t:t+H}^{0}
da_{t:t+H}^{0}
}{
\int 
p(a_{t:t+H}^{\tau}| a_{t:t+H}^{0}, z_t)
\pi_{\phi}^{0}(a_{t:t+H}^{0}| z_t)
da_{t:t+H}^{0}
}.
\end{split}
\end{equation}

To express this posterior mean using proposal sampling, we introduce the Gaussian proposal distribution
\begin{equation} \label{eq:proposal}
q_{\tau}(a_{t:t+H}^{0}| a_{t:t+H}^{\tau}, z_t)
=
\mathcal{N}
\left(
a_{t:t+H}^{0};
\frac{1}{\sqrt{\bar{\alpha}^{\tau}}}a_{t:t+H}^{\tau},
\frac{1-\bar{\alpha}^{\tau}}{\bar{\alpha}^{\tau}}I
\right).
\end{equation}
This proposal is proportional to the DDPM likelihood 
$p(a_{t:t+H}^{\tau}| a_{t:t+H}^{0}, z_t)$ when viewed as a function of $a_{t:t+H}^{0}$. 
Therefore,~\eqref{eq:posterior_expectation} can be rewritten as
\begin{equation}
\begin{split}
\mathbb{E}
\left[
a_{t:t+H}^{0}
|
a_{t:t+H}^{\tau}, z_t
\right]
&=
\frac{
\mathbb{E}_{a_{t:t+H}^{0|\tau}\sim q_{\tau}(\cdot| a_{t:t+H}^{\tau},z_t)}
\left[
\pi_{\phi}^{0}(a_{t:t+H}^{0|\tau}| z_t)
a_{t:t+H}^{0|\tau}
\right]
}{
\mathbb{E}_{a_{t:t+H}^{0|\tau}\sim q_{\tau}(\cdot| a_{t:t+H}^{\tau},z_t)}
\left[
\pi_{\phi}^{0}(a_{t:t+H}^{0|\tau}| z_t)
\right]
}.
\end{split}
\end{equation}

Substituting the posterior mean back into the score expression gives
\begin{equation}
\phi(z_t, a_{t:t+H}^\tau, \tau)
=
-\frac{a_{t:t+H}^{\tau}}{1-\bar{\alpha}^{\tau}}
+
\frac{\sqrt{\bar{\alpha}^{\tau}}}{1-\bar{\alpha}^{\tau}}
\mathbb{E}_{a_{t:t+H}^{0}\sim q_{\tau}(\cdot | a_{t:t+H}^{\tau},z_t)}
\left[
w(a_{t:t+H}^{0|\tau})a_{t:t+H}^{0|\tau}
\right],
\end{equation}
where the normalized importance weight is defined as
\begin{equation}
w(a_{t:t+H}^{0|\tau})
=
\frac{
\pi_{\phi}^{0}(a_{t:t+H}^{0|\tau}| z_t)
}{
\mathbb{E}_{\tilde a_{t:t+H}^{0|\tau}\sim q_{\tau}(\cdot | a_{t:t+H}^{\tau},z_t)}
\left[
\pi_{\phi}^{0}(\tilde a_{t:t+H}^{0|\tau}| z_t)
\right]
}.
\end{equation}
Finally, under the entropy-regularized policy induced by the trajectory return,
\begin{equation}
\pi_{\phi}^{0}(a_{t:t+H}^{0|\tau}| z_t)
\propto
\exp
\left(
\frac{\tilde{G}(z_t,a_{t:t+H}^{0|\tau})}{\kappa}
\right).
\end{equation}
Hence, the normalized importance weight becomes
\begin{equation} \label{eq:importance_weight_append}
w(a_{t:t+H}^{0|\tau})
=
\frac{
\exp\left(\tilde{G}(z_t,a_{t:t+H}^{0|\tau})/\kappa\right)
}{
\mathbb{E}_{\tilde a_{t:t+H}^{0|\tau}\sim q_{\tau}(\cdot | a_{t:t+H}^{\tau},z_t)}
\left[
\exp\left(\tilde{G}(z_t,\tilde a_{t:t+H}^{0|\tau})/\kappa\right)
\right]
}.
\end{equation}
This completes the proof.
\end{proof}

% \end{proof}
\paragraph{What is the role of the world model in the computation of the score function?} The world model provides the latent dynamics used to evaluate each candidate clean action sequence sampled during the score computation.  In the derivation above, the proposal distribution 
$q_{\tau}(a_{t:t+H}^{0|\tau}\mid a_{t:t+H}^{\tau},z_t)$ is obtained from the DDPM forward kernel and is used to sample clean action sequences $a_{t:t+H}^{0|\tau}$ that are compatible with the noisy action sequence $a_{t:t+H}^{\tau}$. 
For each sampled action sequence, the world model recursively rolls out the latent dynamics
\begin{equation}
    z_{h+1}^{0|\tau} = \mathcal{F}(z_h^{0|\tau},a_h^{0|\tau}), 
    \qquad h=t,\ldots,t+H-1,
\end{equation}
starting from the current latent state $z_t^{0|\tau}=z_t$. 
This rollout produces a latent trajectory 
\begin{equation}
    \left(
    z_t,
    a_t^{0|\tau},
    z_{t+1}^{0|\tau},
    a_{t+1}^{0|\tau},
    \ldots,
    z_{t+H}^{0|\tau},
    a_{t+H}^{0|\tau}
    \right),
\end{equation}
which is then used to compute the trajectory-level return
\begin{equation}
    \tilde{G}(z_t, a_{t:t+H}^{0|\tau}) 
    \coloneqq 
    \sum_{h = t}^{t+H-1} 
    \gamma^{h - t} 
    \tilde{\mathcal{R}}(z_h^{0|\tau}, a_h^{0|\tau}) 
    + 
    \gamma^{H}
    \hat{Q}(z_{t+H}^{0|\tau}, a_{t+H}^{0|\tau}) 
    - 
    \eta E(z_{t+H}^{0|\tau}, a_{t+H}^{0|\tau}).
\end{equation}
The return $\tilde{G}(z_t, a_{t:t+H}^{0|\tau})$ determines the normalized importance weight
\begin{equation}
w(a_{t:t+H}^{0|\tau})
=
\frac{
\exp\left(\tilde{G}(z_t,a_{t:t+H}^{0|\tau})/\kappa\right)
}{
\mathbb{E}_{\tilde a_{t:t+H}^{0|\tau}\sim q_{\tau}(\cdot |a_{t:t+H}^{\tau},z_t)}
\left[
\exp\left(\tilde{G}(z_t,\tilde a_{t:t+H}^{0|\tau})/\kappa\right)
\right]
}.
\end{equation}
Thus, the world model does not directly define the Gaussian proposal used for denoising. 
Instead, it provides the model-based evaluation signal that reweights the sampled clean action sequences according to their predicted trajectory returns. 
The exact score is then computed as a return-weighted posterior mean:
\begin{equation}
\phi(z_t,a_{t:t+H}^{\tau},\tau)
=
-\frac{a^{\tau}_{t:t+H}}{1-\bar{\alpha}^{\tau}}
+
\frac{\sqrt{\bar{\alpha}^{\tau}}}{1-\bar{\alpha}^{\tau}}
\mathbb{E}_{a_{t:t+H}^{0|\tau} \sim q_{\tau}(\cdot |a_{t:t+H}^{\tau},z_t)}
\left[
w(a^{0|\tau}_{t:t+H}) a_{t:t+H}^{0|\tau}
\right].
\end{equation}
In this sense, the world model converts each sampled clean action sequence into a trajectory-level value estimate, and the score function shifts the noisy action sequence toward clean action sequences that achieve higher predicted return under the latent dynamics.

A more straightforward abstraction is that the reverse transition kernel defines an iterative policy refinement process. 
Starting from a simple Gaussian prior over action sequences, the reverse diffusion process gradually transforms this prior into the optimized action-sequence policy:
\begin{equation}
\underbrace{\pi_{\phi}^{N}(a_{t:t+H}^{N}|z_t)}_{\text{Gaussian prior}}
\xrightarrow{\pi_{\phi}^{N-1}(a_{t:t+H}^{N-1}|a_{t:t+H}^{N},z_t)}
\pi_{\phi}^{N-1}(a_{t:t+H}^{N-1}|z_t)
\cdots
\xrightarrow{\pi_{\phi}^{0}(a_{t:t+H}^{0}|a_{t:t+H}^{1},z_t)}
\underbrace{\pi_{\phi}^{0}(a_{t:t+H}^{0}|z_t)}_{\text{clean policy}}.
\end{equation}
Equivalently, the final refined policy is obtained by marginalizing over all intermediate denoising steps:
\begin{equation}
    \pi_{\phi}^{0} (a_{t:t+H}^{0}|z_t)
    =
    \int_{a_{t:t+H}^{1:N}} 
    \pi^{N}_{\phi}(a_{t:t+H}^{N}|z_t) 
    \prod_{\tau = N}^{1} 
    \pi_{\phi}^{\tau-1} 
    (a_{t:t+H}^{\tau-1}|a_{t:t+H}^{\tau}, z_t) 
    da_{t:t+H}^{1:N}.
\end{equation}
At each reverse step, the transition kernel 
$\pi_{\phi}^{\tau-1}(a_{t:t+H}^{\tau-1}|a_{t:t+H}^{\tau},z_t)$ is guided by the score function $\phi(z_t,a_{t:t+H}^{\tau},\tau)$. This score is computed by sampling candidate clean action sequences from the DDPM posterior proposal  $q_{\tau}(\cdot|a_{t:t+H}^{\tau},z_t)$, rolling them out through the world model, and reweighting them according to their predicted returns.  \textit{Consequently, the world model redefines the target for score matching by introducing model-based evaluations. This shift directly propagates to the stepwise transition kernel, thereby rendering the entire reverse diffusion process equivalent to policy optimization. }

From this perspective, reverse diffusion can be viewed as model-based policy optimization in action-sequence space. 
The Gaussian prior provides exploration over possible action sequences, while the world-model-evaluated score progressively reshapes this prior into a policy that concentrates on action sequences with higher predicted return. 
Therefore, the policy refinement in~\eqref{eq:policy_refinement} is not merely a generative denoising procedure; it is also a KL-regularized policy optimization process where the world model supplies the trajectory-level objective used to guide each denoising step (stated in Theorem~\ref{theorem:score_matching_policy}).

\subsection{Proof of Theorem~\ref{theorem:score_matching_policy}} \label{append:proof_theorem_1}

\begin{proof} We prove Theorem~\ref{theorem:score_matching_policy} in two steps: (1) demonstrating that the optimal distribution of the KL-constrained policy optimization problem formulated in~\eqref{eq:Lagrangian} follows a Gibbs distribution, and (2) showing that score matching within the world model is equivalent to sampling from this optimal policy.

\textbf{Step 1. Deriving the Optimal Gibbs Policy.} 

For a fixed latent state $z_t$, the latent trajectory is deterministically generated by the learned world model
\begin{equation}
    z_{h+1}=\mathcal F(z_h,a_h), 
    \qquad h=t,\ldots,t+H-1.
\end{equation}
Therefore, once $z_t$ and the action sequence $a_{t:t+H}$ are given, the trajectory return $G(z_t,a_{t:t+H})$ is completely determined by the rollout of the world model. Equivalently, the world-model constraint restricts feasible trajectories to those satisfying $z_{h+1}=\mathcal F(z_h,a_h)$.

We first consider the Lagrangian formulation of the KL-constrained policy optimization problem~\eqref{eq:Lagrangian}. Let $\eta>0$ be the KKT multiplier associated with the KL divergence constraint. The regularized objective is:
\begin{equation}
\max_{\pi_\phi(\cdot|z_t)}
\left\{
\mathbb{E}_{a_{t:t+H}\sim \pi_\phi(\cdot|z_t)}
\left[
G(z_t,a_{t:t+H})
\right]
-
\eta 
D_{\mathrm{KL}}
\left(
\pi_\phi(\cdot|z_t)\|\beta(\cdot|z_t)
\right)
\right\}.
\label{eq:kl_regularized_policy_proof}
\end{equation}

Expanding the KL term, the objective becomes
\begin{align}
\mathcal J(\pi_\phi)
&=
\int 
\pi_\phi(a_{t:t+H}|z_t)
G(z_t,a_{t:t+H})
da_{t:t+H}
\nonumber \\
&\quad
-
\eta
\int 
\pi_\phi(a_{t:t+H}|z_t)
\log
\frac{
\pi_\phi(a_{t:t+H}|z_t)
}{
\beta(a_{t:t+H}|z_t)
}
da_{t:t+H}.
\end{align}

Note that the behavior prior $\beta$ over the trajectory incorporates the deterministic latent dynamics constraint of world models:
\begin{equation}
    \beta(a_{t:t+H}|z_t) = \prod_{h=t}^{t+H} \beta(a_h|z_h) \cdot \prod_{h=t}^{t+H-1}\mathbb{1}_{\{z_{h+1} = \mathcal{F}(z_h,a_h)\}}.
\end{equation}

To enforce the probability normalization constraint $\int \pi_\phi(a_{t:t+H}|z_t)da_{t:t+H}=1$, we introduce a second Lagrange multiplier $\lambda$ and define the full Lagrangian:
\begin{align}
\mathcal L(\pi_\phi,\lambda)
&=
\int 
\pi_\phi(a_{t:t+H}|z_t)
G(z_t,a_{t:t+H})
da_{t:t+H}
\nonumber \\
&\quad
-
\eta
\int 
\pi_\phi(a_{t:t+H}|z_t)
\log
\frac{
\pi_\phi(a_{t:t+H}|z_t)
}{
\beta(a_{t:t+H}|z_t)
}
da_{t:t+H}
\nonumber \\
&\quad
+
\lambda
\left(
\int 
\pi_\phi(a_{t:t+H}|z_t)
da_{t:t+H}
-1
\right).
\end{align}

Taking the functional derivative with respect to $\pi_\phi(a_{t:t+H}|z_t)$ and setting it to zero gives
\begin{equation}
0
=
G(z_t,a_{t:t+H})
-
\eta
\left(
\log
\frac{
\pi_\phi(a_{t:t+H}|z_t)
}{
\beta(a_{t:t+H}|z_t)
}
+1
\right)
+
\lambda.
\end{equation}

Rearranging for the density ratio yields
\begin{equation}
\log
\frac{
\pi_\phi(a_{t:t+H}|z_t)
}{
\beta(a_{t:t+H}|z_t)
}
=
\frac{
G(z_t,a_{t:t+H})
}{\eta}
+
\frac{\lambda-\eta}{\eta}.
\end{equation}

Exponentiating both sides, we obtain the optimal policy:
\begin{equation}
\pi_\phi^*(a_{t:t+H}|z_t)
=
\frac{1}{Z(z_t)}
\beta(a_{t:t+H}|z_t)
\exp
\left(
\frac{
G(z_t,a_{t:t+H})
}{\eta}
\right),
\label{eq:gibbs_policy_proof}
\end{equation}
where $Z(z_t)$ is the partition function. Substituting the factorized form of $\beta$, the optimal distribution can be expressed as:
\begin{equation}
\pi_\phi^*(a_{t:t+H}|z_t)
\propto
\prod_{h=t}^{t+H} \beta(a_h|z_h) \cdot \exp\left( \frac{G(z_t, a_{t:t+H})}{\eta} \right) \cdot \prod_{h=t}^{t+H-1}\mathbb{1}_{\{z_{h+1} = \mathcal{F}(z_h, a_h)\}}.
\end{equation}

\textbf{Step 2. Generation of the Gibbs Distribution via Score Matching.} 

By Lemma~\ref{lemma:score_function_proof}, an expression of the form~\eqref{eq:score_matching} is analytically identical to the exact score function $\nabla_{a_{t:t+H}^{\tau}} \log \pi_{\phi}^\tau(a_{t:t+H}^{\tau}|z_t)$ of a forward-diffused distribution, whose underlying base distribution $\pi_{\phi}^0(a_{t:t+H}^0|z_t)$ is proportional to the unnormalized numerator of the weight $w(a_{t:t+H}^0)$ in \eqref{eq:importance_weight_append}. Therefore, according to the properties of variance-preserving stochastic differential equations \citep{song2020score}, the score matching condition implicitly defines the following base distribution:
\begin{equation}
\pi_{\phi}^0(a_{t:t+H}^{0}|z_t)
\propto
\exp
\left(
\frac{
\tilde{G}(z_t,a_{t:t+H}^{0})
}{\eta}
\right).
\label{eq:implied_base_dist}
\end{equation}

We now relate this implicit base distribution back to the policy derived in Step 1. Recall that the energy-regularized reward is defined as:
\begin{equation}
\tilde{\mathcal{R}} (z_h, a_h) = \mathcal{R}(z_h, a_h) - \frac{\eta}{\gamma^{h-t}} E(z_h, a_h).
\end{equation}
Based on the energy minimization objective~\eqref{eq:minimization_of_energy}, the learned energy approximates the behavior prior up to a state-dependent constant: $-E(z_h,a_h)\approx \log \beta(a_h|z_h)+C(z_h)$. 
In continuous state spaces with continuous dynamics, the state variations over a short rollout horizon $H$ are relatively small. Thus, we can reasonably approximate $C(z_h) \approx C(z_t)$, which acts as a fixed constant with respect to the future action sequence $a_{t:t+H}$. Summing over the trajectory and accounting for the discount factor, the energy-regularized return $\tilde{G}$ decomposes into the original return $G$ and the behavior prior term:
\begin{equation}
\frac{\tilde{G}(z_t,a_{t:t+H}^{0})}{\eta} \approx \frac{G(z_t,a_{t:t+H}^{0})}{\eta} + \sum_{h=t}^{t+H} \log \beta(a_h|z_h) + (H+1)C(z_t).
\end{equation}

Substituting this decomposition into~\eqref{eq:implied_base_dist}, the constant term $(H+1)C(z_t)$ is absorbed into the proportionality constant, and the base distribution can be explicitly written as:
\begin{equation}
\pi_{\phi}^0(a_{t:t+H}^{0}|z_t)
\propto
\left(\prod_{h=t}^{t+H} \beta(a_h|z_h)\right) \exp\left(\frac{G(z_t,a_{t:t+H}^{0})}{\eta}\right) \prod_{h=t}^{t+H-1} \mathbb{1}_{\{z_{h+1}=\mathcal F(z_h,a_h^0)\}}.
\end{equation}
This distribution is exactly identical to the optimal Gibbs policy $\pi^*_{\phi}(a_{t:t+H}| z_t)$ derived in Step 1.

Therefore, if the score matching condition is satisfied, the reverse diffusion process uses the exact score of the noisy marginal induced by the Gibbs policy. 
Consequently, the reverse diffusion sampler generates clean action sequences exactly from the optimal KL-regularized policy distribution.

Conversely, suppose the KL-constrained policy optimization problem attains its optimum. 
Then, by the variational argument above, the optimal policy has the Gibbs form in~\eqref{eq:gibbs_policy_proof}. 
The diffusion marginal induced by this clean-action distribution has a unique exact score, and by Lemma~\ref{lemma:score_function_proof}, this score is precisely the expression in~\eqref{eq:score_matching}. 
Thus, the score matching condition is satisfied.

Combining both directions, the KL-constrained policy optimization problem attains its Gibbs-form optimum if and only if the score matching condition in~\eqref{eq:score_matching} is satisfied. Finally, to provide flexibility in controlling the entropy and return trade-off during practical sampling, \textit{we relax the strict Lagrange multiplier $\eta$ to a general temperature parameter $\kappa > 0$}, yielding the target Gibbs distribution presented in Theorem~\ref{theorem:score_matching_policy}.
This completes the proof.
\end{proof}

\clearpage
\section{Practical Implementation and Hyperparameters} \label{append:practical_implement}

% \usepackage[ruled,linesnumbered,noend]{algorithm2e}
% \usepackage{amsmath}
% \usepackage{amssymb}

% \begin{document}

% Preamble

\SetKwInput{KwInput}{Input}
\SetKwInput{KwInit}{Initialize}
\SetKw{KwFor}{for}
\SetKw{KwEnd}{end for}
\DontPrintSemicolon
\SetNlSty{}{}{:}
\begin{algorithm}[ht]
\caption{Model-Based Diffusion Policy Optimization (MBDPO)}
\label{alg:MBDPO}

\KwInput{Encoder $\mathcal{E}$, latent dynamics $\mathcal{F}$, task embedding $\mathcal{E}_{env}$, reward function $\mathcal{R}$, value function $\hat{Q}$, implicit energy function $E$, score network $\hat{\phi}$, diffusion steps $N$, Monte Carlo samples $T$}
\KwInit{Replay buffer $\mathcal{D} \leftarrow \emptyset$}

\BlankLine
\texttt{// \textcolor{cyan}{Warmup: Train the initial world model}}\;
\For{$i = 1, 2, \dots, N_{\mathrm{warmup}}$}{
    Sample random action $a_h \sim \mathrm{Uniform}(\mathcal{A})$\;
    Execute action: $(s_{h}, \Box) \leftarrow \mathrm{env.step}(a_h)$\;
    Store transition $(s_h, a_h, r_h, s_{h+1})$ into $\mathcal{D}$\;
}
Update $\mathcal{E}, \mathcal{F}, \mathcal{R}, \mathcal{E}_{env}, \hat{Q}, E$ by minimizing the model loss in \eqref{eq:joint_minimizer} using $\mathcal{D}$\;

\BlankLine
\For{iteration $k = 1, 2, \dots, K$}{
    \texttt{// \textcolor{cyan}{Step 1: Sampling data from the environment via the score network}}\;
    Encode current observation: $z_h \leftarrow \mathcal{E}(s_h)$\;
    Sample action $a_h$ iteratively via the diffusion process in \eqref{eq:sampling}\; 
    Execute action: $(s_{h}, \Box) \leftarrow \mathrm{env.step}(a_h)$\;
    Store transition $(s_h, a_h, r_h, s_{h+1})$ into $\mathcal{D}$\;
    
    \vspace{0.5em}
    \texttt{// \textcolor{cyan}{Step 2: World model update}}\;
    Sample a batch of transitions $\mathcal{B} \sim \mathcal{D}$\;
    Update $\mathcal{E}, \mathcal{F}, \mathcal{R}, \mathcal{E}_{env}, \hat{Q}, E$ by minimizing the joint loss in \eqref{eq:joint_minimizer} on batch $\mathcal{B}$\;
    
    \vspace{0.5em}
    \texttt{// \textcolor{cyan}{Step 3: Diffusion policy optimization}}\;   
    Estimate the score function via \eqref{eq:monte_carlo} using imagined trajectories in the world model\;
    Update score network $\hat{\phi}$ by minimizing the supervised $L^2$ loss in \eqref{eq:score_fitting_loss}\;
}
\BlankLine
\end{algorithm}

Note that Algorithm \ref{alg:MBDPO} outlines the \textit{online} learning paradigm. In the \textit{offline} pretraining setting, the agent does not interact with the environment to collect new data. Instead, the replay buffer $\mathcal{D}$ is initialized with a static, pre-collected dataset, and Step 1 (sampling data via the score network and environment interaction) is entirely omitted. Consequently, the world model and diffusion policy are optimized, relying solely on the offline dataset.

Furthermore, for the \textbf{\textit{offline-to-online fine-tuning}} paradigm, the initial warmup phase is completely bypassed. Instead, the algorithm is initialized with the parameters of the world model and the score network that have been pre-trained on the offline dataset. The agent then proceeds directly to the main iterative loop, executing \textbf{Step 1}, \textbf{Step 2}, and \textbf{Step 3} to continuously interact with the environment and fine-tune the score function and world model.

% \subsection{Hyperparameter}

\begin{table}[t]
\centering
\caption{\textbf{MBDPO hyperparameters for the default 6M model.} We use the same hyperparameters for all online-from-scratch tasks.}
\label{tab:tdmpc2_hyperparameters}
\begin{tabular}{ll}
\toprule
\textbf{Hyperparameter} & \textbf{Value} \\
\midrule

\multicolumn{2}{l}{\textcolor{cyan}{\textbf{\underline{Search via diffusion sampling}}}} \\
Horizon ($H$) & 3 \\
Diffusion steps & 5-20 steps \\
Monte Carlo sample number & 512 \\
Temperature $\kappa$ & 0.5 \\
Energy regularized factor $\eta$ & 0.1 \\

\\[-0.3em]
\multicolumn{2}{l}{\textcolor{cyan}{\textbf{\underline{Replay buffer}}}} \\
Capacity & 1,000,000 \\
Sampling & Uniform \\

\\[-0.3em]
\multicolumn{2}{l}{\textcolor{cyan}{\textbf{\underline{Architecture (6M)}}}} \\
Encoder dim & 256 \\
MLP dim & 512 \\
Latent state dim & 512 \\
Task embedding dim & 96 \\
Task embedding norm & 1 \\
Activation & LayerNorm + Mish \\
$\hat{Q}$-function dropout rate & 1\% \\
Ensemble number of $\hat{Q}$-functions & 5 \\

\\[-0.3em]
\multicolumn{2}{l}{\textcolor{cyan}{\textbf{\underline{Optimization}}}} \\
Batch size & 256 (online) / 1024 (offline) \\
Policy prior loss norm. & Moving (5\%, 95\%) percentiles \\
Optimizer & Adam \\
Learning rate & $3 \times 10^{-4}$ \\
Encoder learning rate & $1 \times 10^{-4}$ \\
Gradient clip norm & 20 \\

\bottomrule
\end{tabular}
\end{table}

\begin{table}[t]
\centering
\caption{\textbf{Model configurations.} We list the specifications for each model configuration (size) of our multi-task offline experiments. \textit{Encoder dim} is the dimensionality of fully connected layers in the encoder $\mathcal{E}$, \textit{MLP dim} is the dimensionality of layers in all other components, \textit{Latent state dim} is the dimensionality of the latent representation $z$, \textit{\# encoder layers} is the number of layers in the encoder $\mathcal{E}$, \textit{\# Q-functions} is the number of learned $Q$-functions, and \textit{Task embedding dim} is the dimensionality of $e$ from Equation~(2). TD-targets are always computed by randomly subsampling two $\hat{Q}$-functions, regardless of the number of $\hat{Q}$-functions in the ensemble. We did not experiment with other model configurations. *The default (\texttt{base}) configuration used in our single-task RL experiments has 6M parameters.}
\label{tab:model_configurations}
\begin{tabular}{lccccc}
\toprule
 & \textbf{1.7M} & \textbf{6M*} & \textbf{21M} & \textbf{54M} & \textbf{340M} \\
\midrule
Encoder dim & 256 & 256 & 1024 & 1792 & 4096 \\
MLP dim & 384 & 512 & 1024 & 1792 & 4096 \\
Latent state dim & 128 & 512 & 768 & 768 & 1376 \\
\# encoder layers & 2 & 2 & 3 & 4 & 5 \\
\# $\hat{Q}$-functions & 2 & 5 & 5 & 5 & 8 \\
Task embedding dim & 96 & 96 & 96 & 96 & 96 \\
\bottomrule
\end{tabular}
\end{table}

\end{document}